%% file: main.tex
\documentclass[11pt]{article}
\usepackage{tocloft}
\input{notation.tex}

\input{preamble.tex}

\title{{\bf Global Sequential Testing for Multi-Stream Auditing}}

\author{Beepul Bharti\\\texttt{bbharti1@jhu.edu}}
\author{
    Beepul~Bharti$^{1}$\footnote{Corresponding author.}\\
    {\small bbharti1@jhu.edu}
    \and
    Ambar Pal$^{2}$\footnote{This work is not related to AP's position at Amazon.}\\
    {\small ambarpal@amazon.com}
    \and
    Jeremias~Sulam$^{1}$\\
    \small {jsulam1@jhu.edu}
}
\date{
    \centering
    $^1$Johns Hopkins University\\
    $^2$Amazon Responsible AI\\
    \vspace{1em}
    \today
}

\begin{document}
\maketitle
\vspace{-2em}
\begin{abstract}
    Across many risk-sensitive areas, it is critical to continuously audit machine learning systems as we receive more data to quickly determine if they are performing as designed. This auditing task can be modeled as a sequential hypothesis testing problem with $k$ data streams and a global null hypothesis that asserts the system operates as intended across all $k$ streams. Under the alternative, the standard global sequential test, which uses a Bonferroni correction, has an expected stopping time of $\mathcal{O}\left(\ln \nicefrac{k}{\alpha}\right)$ for large $k$ and significance level $\alpha$. In this work, we demonstrate that efficient sequential tests, relying on merging martingales via averaging and products rules, provide improved stopping times, and thus more powerful tests against the null. 
    Using these results, we show that a balanced test can match the Bonferroni rate of $\mathcal{O}\left(\ln \nicefrac{k}{\alpha}\right)$ in the sparse regime (just a few non-null streams) while achieving $\mathcal{O}\left(\nicefrac{1}{k}\ln \nicefrac{1}{\alpha}\right)$ under dense alternatives (many non-null steams). We validate our theory through experiments on both synthetic and real-world data. 
\end{abstract}

\tableofcontents
\newpage

\section{Introduction}

\label{sec:intro}
Machine learning (ML) systems are increasingly deployed in high-stakes decision-making contexts such as healthcare \citep{shailaja2018machine} and criminal justice \citep{rudin2020age}, where they influence critical decisions about individuals. Thus, the development of auditing procedures that continually evaluate a system's performance has become an essential area of study \citep{raji2019actionable, chugg2023auditing, bandy2021problematic, metaxa2021auditing, vecchione2021algorithmic}. Notably, the importance of auditing methods has been emphasized by regulatory and governance bodies, including the U.S. Office of Science and Technology Policy \citep{ostp2022blueprint}, the European Union \citep{eu2021aiact}, and the United Nations \citep{unesco2021recommendation}.

Statistical hypothesis testing provides a rigorous framework to perform online auditing of an ML system. For example, suppose we are tasked with auditing a newly deployed state-of-the-art medical foundation model at a hospital. The model can perform zero-shot diagnoses across various imaging modalities, like X-rays, computed tomography, magnetic resonance imaging, and more. Importantly, we receive a continuous stream of data representing the model's predictions. To audit the model’s performance on a specific modality, such as X-rays, we can define the following null hypothesis:
\begin{align*}
    H_{0} &:~ \text{The model is working as intended on X-rays.}
\end{align*}
Moreover, we can continuously test this null by leveraging \emph{sequential} hypothesis tests. Unlike traditional hypothesis tests \citep{neyman1933ix}, which require collecting a fixed batch of data up to a pre-specified time to obtain a valid $p$-value, sequential tests allow data collection to stop at arbitrary, data-dependent times once sufficient evidence against the null hypothesis has accumulated. 

Ensuring a system is performing well on one stream is important, but assessing its performance across multiple streams (e.g., demographic groups, geographic locations, etc) is crucial for ensuring fairness, safety, and reliability \citep{susarla2024zillow, obermeyer2019dissecting, dastin2022amazon}. For instance, in the foundation model example, it is critical to ensure the model is performing well on \emph{multiple} imaging modalities to assure safe patient outcomes. Thus, in this work we consider a multi-stream auditing problem: monitoring $k$ data streams, the goal is to raise an alarm as soon as there is sufficient evidence that the model is not performing as intended on \emph{any} stream. That is, we focus on sequentially testing a \emph{global} null hypothesis of the form:
\begin{align*}
    H^{\text{global}}_{0} &:~ \text{The model is working as intended across all image modalities.}
\end{align*}
Testing this global null is equivalent to testing the intersection of \( k \) hypotheses \( H_{1,0}, \ldots, H_{k,0} \), where each \( H_{i,0} \) asserts that the system is performing as intended on stream \( i \). Importantly, a range of alternatives may arise: for example, the model may perform poorly on only one modality (a sparse alternative) or on many modalities (a dense alternative). For both regimes -- and everything in between -- we must understand which tests are efficient and powerful to enable effective global testing.

To accomplish this, we study a set of sequential tests for this global null and provide new expected stopping times for them. The only test with a known expected stopping time under the global alternative is one that relies on Bonferroni correction. For significance level $\alpha \in (0,1)$, this test has an expected stopping time of $\calO(\ln \nicefrac{k}{\alpha})$ in both sparse and dense settings. While simple, this guarantee is unnecessarily large in cases when many streams provide evidence against the null -- as we will make clear shortly. More generally, this Bonferroni correction is an example of a test that employs a specific \emph{merging} strategy by combining evidence from different streams (by computing a max). This motivates the study of other common, previously studied merging strategies, such as the average and product \citep{vovk2021values, wang2025only, ramdas2024hypothesis}. While empirically useful, neither of these strategies have known stopping time guarantees that justify their use under different alternative hypotheses. In this work we provide stopping time results for each strategy, highlighting their strengths and weaknesses. Furthermore, by leveraging these findings, we construct a simple sequential test that balances their strengths and achieves an expected stopping time of $\calO(\ln \nicefrac{k}{\alpha})$ under a sparse alternative and a smaller bound of $\calO(\nicefrac{1}{k}\ln \nicefrac{1}{\alpha})$ under a dense alternative.

\section{Problem Formulation}

\label{sec:prob_form}
We formalize the auditing problem as a statistical hypothesis test \cite{daiindividual, chugg2023auditing}. Let $P_i$ be a distribution supported on $\mathcal{Z} \subset [-1,1]$. The random variable $Z \sim P_i$ captures the system's behavior, with larger $|Z|$ indicating greater deviations from the optimal behavior. 
Given data $Z_{i,1}, Z_{i,2}, \ldots \overset{\text{iid}}{\sim} P_i$, auditing a system's performance on $P_i$ can be framed as a hypothesis test with null and alternative hypotheses:
\begin{align}
H_{i,0}: \mathbb{E}_{P_i}[Z_{i,1}] = 0 \quad \text{and} \quad H_{i,1}: \mathbb{E}_{P_i}[Z_{i,1}] \neq 0.
\end{align}
The null $H_{i,0}$ asserts the system performs as intended on $P_i$, while the alternative $H_{i,1}$ suggests a deviation from expected behavior. This is a bounded mean test \citep{waudby2024estimating}, which tests if the mean of a bounded random variable equals zero. Such tests are widely used for ML system auditing \citep{daiindividual, chugg2023auditing}.

In this work, we move beyond single audits and study the problem of evaluating a systems performance across $k$ distributions, $P_1, \dots, P_k$. Formally, letting $[k] \coloneqq \{1, \dots, k\}$, we consider the multi-stream auditing task where we observe data from all distributions, $Z_{i,1}, Z_{i,2}, \ldots \overset{\text{iid}}{\sim} P_i$ for each $i \in [k]$, and want to test the \emph{global} null and alternative hypotheses:
\begin{equation}
    \begin{aligned}
        H^{\text{global}}_{0} : \forall i \in [k], \ \mathbb{E}_{P_i}[Z_{i,1}] = 0 \quad \text{and} \quad H^{\text{global}}_{1} : \exists i \in [k] ~~ \text{such that} ~~ \mathbb{E}_{P_i}[Z_{i,1}] \neq 0.
    \end{aligned}
\end{equation}
The global null holds if and only if every null  \(H_{i,0}\) is true. Thus, it provides a precise auditing process: rejecting $H^{\text{global}}_{0}$ implies the system is performing sub-optimally on at least one distribution $P_i$.

Unlike traditional hypothesis testing \citep{neyman1933ix}, we work in a setting where data is collected over time, i.e., we have \( k \) parallel data streams \( (Z_{i,t})_{t \geq 1} \equiv (Z_{i,1}, Z_{i,2}, \ldots) \), where \( Z_{i,1}, Z_{i,2}, \ldots \overset{\text{iid}}{\sim} P_i \) for all \( i \in [k] := \{1, \ldots, k\} \). Our goal is to \emph{sequentially} test \( H^{\text{global}}_{0} \) as data arrive, determining whether the system is performing as intended across \emph{all} streams and raising an alarm if there is evidence to the contrary. The standard approach to testing the global null relies on testing each individual null and then applying a Bonferroni correction. Tests relying on average and products for combining evidence from the different streams have been recently proposed and studied to varying degrees \citep{vovk2021values, ramdas2024hypothesis, wang2025only}. Importantly, however, it is unknown if these other strategies provide better stopping time guarantees, which is the core focus of our work.

{\noindent {\bf Additional Notation.}} For $k$ streams, $(Z_{1,t})_{t\geq 1}, \dots, (Z_{k,t})_{t\geq 1}$, denote $\Delta_i = \bE[Z_{i,1}]$\footnote{The $Z_{i,1},\dots ,Z_{i,t} \overset{iid}{\sim} P_i$, so $\mathbb{E}[Z_{i,1}] = \cdots = \mathbb{E}[Z_{i,t}] = \Delta_i$. We use $Z_{i,1}$ for clarity and notational convenience.}, $\Delta_{\textsf{m}} = \underset{i \in [k]}\max~\Delta_i$, and $\Delta_{\textsf{s}} =$ $\sum^k_{i=1} \Delta^2_i$. By definition, under $H^{\textsf{global}}_0$, these quantities equal zero. Under the alternative $H^{\textsf{global}}_1$, $\exists i \in [k]$, such that $\Delta_i \neq 0$. As a result, under $H^{\textsf{global}}_1$, $\Delta_{\textsf{m}} \neq 0$ and $\Delta_{\textsf{s}} \neq 0$.

\subsection{Related Work}

\label{sec:related_work}
Safe anytime valid inference, multi-stream testing, and auditing all draw on several distinct lines of research. We discuss the works in these areas most related to ours. 

\paragraph{Safe anytime-valid inference.}\citet{shafer2021testing, vovk2021values, shekhar2023nonparametric, waudby2024estimating, waudby2025universal, ramdas2024hypothesis} introduce and evaluate the testing by betting framework, providing numerous results regarding sequential tests and how to construct them with betting, martingales, and $e$-processes. Our work employs the machinery presented in these works. Note, there are several other works that study testing multiple hypotheses via sequential hypothesis tests \citep{xu2024online, javanmard2018online, zrnic2021asynchronous}. The primary focus of these works is developing sequential tests with controlled false discovery rate.

\paragraph{\bf Auditing ML systems via hypothesis testing.} \citet{cen2024transparency} operationalize modern algorithmic auditing using statistical hypothesis testing. 
However, their focus is on classical non-sequential settings where one has a fixed sample. On the other hand, there are various other works that focus on developing sequential tests to continuously audit deployed systems. \citet{chugg2023auditing} construct sequential tests to continually audit ML predictors for parity-fairness and \citet{daiindividual} construct a real-time auditing test to verify if demographic groups are harmed within a reporting based framework. More recent works have developed sequential tests to evaluate large language models (LLM). \citet{chenonline} propose tests to detect LLM generated texts and \citet{richterauditing} construct tests to detect behavioral shifts in LLMs. These works briefly discuss multi-stream extensions and global testing, though they all perform a Bonferroni correction.

\paragraph{Multi-stream sequential testing.}
A few works explore sequential testing in a multi-stream setting. \citet{kaufmann2021mixture} use a mixture-based approach to construct wealth processes for single streams and combine them using a product strategy. However, they assume parametric models for their data streams, modeling them as 1-dimensional exponential families, whereas we consider a general nonparametric setting. \citet{cho2024peeking} consider a different multi-stream setting: At each time, they observe a single outcome $Z_t$ along with another random variable $A_t$, taking a value in $[k]$ that determines the membership of $Z_t$. That is, every observed point belongs to one of $k$ possible groups. In contrast,  we observe $k$ outcomes at every time point, each from a different stream. Additionally, unlike our work, they do not provide expected stopping times guarantees. Similar to our work, \citet{sandoval2026multi} consider testing the global null in a setting with multiple arms (streams). The critical difference is at each time $t$, their tests can only select a single arm to obtain a data point. 

\paragraph{Merging evidence.} \citet{vovk2021values, ramdas2024hypothesis, wang2025only} detail various methods for merging evidence from multiple streams to test the global null. They introduce merging functions that, at a high level, allow one to combine sequential tests to yield other valid, sequential tests. Among these different merging strategies, we focus on the average and product rules. \citet{vovk2021values} discuss when these merging functions dominate others, albeit only in the sense of providing a larger degree of evidence. While this provides some intuition about when one merging-based test may reject quicker than another, explicit bounds on the expected stopping times of merging-based sequential tests remain unknown. Establishing maximal expected stopping times for sequential tests is central to drawing conclusions about their power, and can indicate when one test is preferable to another.

\vspace{-0.60em}

\section{A Primer in Sequential Testing}
\vspace{-0.75em}
\label{sec:primer}
A sequential test $\phi \equiv (\phi_t)_{t \geq 1}$ for a null $H_0$ makes a decision $\phi_t \in \{0,1\}$ at each (discrete) time $t$ based on the data observed up to time $t$. A decision $\phi_t = 1$ corresponds to rejecting $H_0$, while $\phi_t = 0$ indicates failing to reject $H_0$. A test $\phi$ is required to satisfy two key criteria. First, it must have level $\alpha$. A sequential test $\phi$ for $H_0$ is a  level-$\alpha$ test if $\sup_{P \in H_0} P(\exists t \geq 1: \phi_t = 1) \leq \alpha$. That is, the probability it incorrectly rejects $H_0$ is at most $\alpha$ \emph{simultaneously across all time steps}.


The second criteria is that $\phi$ be powerful. In sequential testing, one typically requires $\phi$ to be asymptotically powerful \citep{shekhar2023nonparametric, chugg2023auditing}, i.e., it satisfies $\inf_{P \in H_1} P(\exists t \geq 1: \phi_t = 1) = 1$. Another, complementary notion, requires that the stochastic process underlying $\phi$ grows fast under the alternative \citep{waudby2025universal}. Following \citep{shekhar2023nonparametric, waudby2025universal}, we deem a test $\phi$ powerful if, under $H_1$, it rejects quickly on average. Formally, letting $\tau = \min\{t \geq 1 : \phi_t = 1\}$ denote its stopping time, we require $\mathbb{E}[\tau]$ be small under the alternative.


For efficient multi-stream auditing, we seek level-$\alpha$ sequential tests for $H^{\text{global}}_{0}$ that stop quickly under $H^{\text{global}}_{1}$. 
In this work we focus on 
tests that rely on merging strategies and provide new expected stopping time guarantees, showcasing when they are, and are not, expected to be powerful. To begin, we introduce relevant definitions and properties of nonnegative supermartingales.

\paragraph{Martingales.} A stochastic process $M \equiv (M_t)_{t\geq 1}$, built on a stream of random variables $(Z_t)_{t\geq1}$, is a test-martingale for $H_0$ if it is non-negative, has initial value $M_0$ satisfying $E_P[M_0] \leq 1$ for all $P \in H_0$, and satisfies $\bE_P[M_{t+1} \mid Z_1,  \dots, Z_t] = M_t$,  for all $ t\geq 1$ and all $P \in H_0$. When the equality is replaced by an inequality ($\leq$), $M$ is a test-supermartingale. A central result regarding test-supermartingales is Ville's inequality \citep{ville1939etude}, 

\begin{restatable}[Ville's Inequality \citep{ville1939etude}]{theorem}{villes}
\label{theorem:villes}
Let $M$ be a test-supermartingale for null hypothesis $H_0$. Then $\forall \alpha > 0$ and $P \in H_0$, $P(\exists t\geq 1: M_t \geq 1/\alpha) \leq \alpha$.
\end{restatable}

In simple terms, one may construct a level-$\alpha$ test for $H_0$ by rejecting when a test super-martingale for $P\in H_0$ becomes large enough. Test martingales and Ville's inequality are the core tools underlying sequential tests. We now show how they are used to test a stream-specific null $H_{i,0}$, providing the foundations to study sequential tests for the global hypotheses counterpart.\footnote{The discussions regarding martingales can be generalized to stochastic processes adapted to the underlying canonical filtration. We omit this generalization for improved readability.}

\subsection{Testing a Stream-Specific Null $H_{i,0}$}

\label{sec:single_test}
To construct a sequential test for a single \(H_{i,0}\), we can use the \emph{testing-by-betting} framework \citep{shafer2021testing}. This involves starting with wealth \(W_{i,0} = 1\) and iteratively betting on the values of \(Z_{i,t}\) from the stream \((Z_{i,t})_{t \geq 1}\) to maximize our wealth. A positive bet indicates our belief that $Z_{i,t} > 0$ (resp. negative indicates $< 0)$, and the magnitude represents how much we are willing to bet. The key result is that we can structure our bets so our wealth increases if $\bE_{P_i}[Z] \neq 0$ (i.e., if the null is false) and stays constant otherwise, on average. Thus, with well-placed bets, our wealth can serve as a measure of \emph{evidence} against the null: the more wealth, the stronger our belief that the null is false.

Formally, at every time $t\geq 1$, using all the data up to $t-1$, $Z_{i,1}, \dots, Z_{i, t-1}$, we choose a \emph{signed} betting fraction $\lambda_{i,t}$ When $\lambda_{i,t} > 0$, it indicates our belief that $Z_{i,t} > 0$. Then, we receive $Z_{i,t}$ and, following previous works \citep{shafer2021testing, ramdas2023game, shekhar2023nonparametric, chugg2023auditing, waudby2024estimating, ryu2024confidence, orabona2023tight}, update our wealth according to 
\begin{align}
    W_{i,t} = W_{i,t-1} \cdot (1 + \lambda_{i,t}Z_{i,t}) = \prod_{j=1}^t (1 + \lambda_{i,j}Z_{i,j}).
\end{align}
The wealth process $W_i \equiv (W_{i,t})_{t \geq 1}$ is a test-martingale under $H_{i,0}$, so Ville's inequality ensures that rejecting when $W_{i,t} \geq 1/\alpha$ constitutes a level-$\alpha$ sequential test \citep{shekhar2023nonparametric}. Thus, to design a powerful level-$\alpha$ sequential test for $H_{i,0}$ it suffices to construct a wealth process that (1) is a test-supermartingale under $H_{i,0}$, and (2) grows rapidly under $H_{i,1}$.
 We select our sequence of betting fractions \((\lambda_{i,t})_{t \geq 1}\) using Online Newton Step (ONS) \citep{cutkosky2018black, hazan2007logarithmic} (\cref{alg:ons}), ensuring it is $[-\nicefrac{1}{2}, \nicefrac{1}{2}]$-valued. 
We choose ONS as it guarantees
the wealth grows exponentially under the alternative (see \cref{lemma:log_wealth_lb}).

With these choices, denote the wealth process as $W^{\textsf{ons}}_i \equiv (W^{\textsf{ons}}_{i,t})_{t\geq1}$ and test $\phi^{\textsf{ons}}_{i} \equiv  (\phi^{\textsf{ons}}_{i,t})_{t\geq1}$ with 
\begin{equation}
    \label{eq:ons_test}
    \phi^{\textsf{ons}}_{i,t} = \mathbf{1}[W^{\textsf{ons}}_{i,t} \geq 1/\alpha].
\end{equation} 
\citet {chugg2023auditing} show that $\phi^{\textsf{ons}}_{i}$ is a level-$\alpha$ sequential test for $H_{0,i}$ and that, under $H_{1,i}$, its stopping time $\tau_i = \min\{t: W^{\textsf{ons}}_{i,t} \geq 1/\alpha \}$ 
obeys $\bE[\tau_i] = \calO\left(\frac{1}{\Delta_i^2} \ln\frac{1}{\alpha \cdot \Delta_i^2}\right)$. This shows $\phi^{\textsf{ons}}_i$ adapts to the hardness of the alternative: a larger $|\Delta_i|$ implies less time required to reject $H_{i,0}$. Moreover, the smaller the significance level $\alpha$, the larger the expected stopping time (logarithmically). We provide a proof of this result for completeness in \cref{app:prop1}.

\subsection{The Standard Bonferroni Test for the Global Null}

\label{sec:bonferonni}
The machinery utilized to test a single $H_{0,i}$, described in the previous section, forms the basis for the standard strategy to test $H^{\text{global}}_{0}$. The test, defined as $\phi^{\textsf{bonf}} \equiv (\phi^{\textsf{bonf}}_t)_{t \geq 1}$, relies on $k$ parallel wealth processes $W^{\textsf{ons}}_{i}$, one for each $H_{i,0}$. The decision at every time $t$ is
\begin{align}
\label{eq:Bonf_test}
    \phi^{\textsf{bonf}}_t = \mathbf{1}\left[\max_{i \in [k]} W^{\textsf{ons}}_{i,t} \geq k/\alpha\right].
\end{align}
This test tracks the single highest wealth at each time $t$ and rejects $H^{\text{global}}_{0}$ using a Bonferroni-style correction that raises the rejection threshold to $\nicefrac{k}{\alpha}$. The following result shows that $\phi^{\textsf{bonf}}$ is a level-$\alpha$ sequential test and provides an upper bound on its expected stopping time under the alternative.

\begin{restatable}[\citet{chugg2023auditing}]{proposition}{chugg}
The test $\phi^{\textsf{bonf}}$ is a level-\(\alpha\) sequential test for $H^{\textsf{global}}_0$. Moreover, under 
$H^{\textsf{global}}_1$, its expected stopping obeys
\label{proposition:prop2} $
    \bE[\tau] = \calO\left(\frac{1}{\Delta_{\textsf{m}}^2}\ln\frac{k}{\alpha\cdot \Delta^2_{\textsf{m}}}\right)$.
\end{restatable}
A proof is included in \cref{app:proofs}. We pause to discuss this bound's dependence on $k$ and $\Delta_m$.

\paragraph{Dependence on $k$.} Keeping the maximum mean $\Delta_{\textsf{m}} \leq \epsilon$, the upper bound increases with $k$ as $\ln k$. While the dependence is only logarithmic, the bound \emph{always} increases, regardless of whether the new streams provide additional evidence against the null.

\paragraph{Dependence on $\Delta_{\textsf{m}}$.} $H^{\textsf{global}}_0$ can fail to hold in many ways, including when a majority, or even all, of the streams have nonzero means (a dense alternative). Regardless, keeping $k$ fixed, the upper bound depends on evidence from all streams only through the maximum mean.

Both dependencies expose a key limitation of $\phi^{\textsf{bonf}}$: when many streams have nonzero means, $\phi^{\textsf{bonf}}$ fails to fully use the evidence within all $k$ streams. To be precise, by construction, at every time $t$, this test merges evidence from every stream by tracking the \emph{only} the maximum wealth. This raises the natural question: can alternative merging strategies yield better expected stopping time guarantees?

\section{Powerful Global Tests via Merging}
\label{sec:testing_via_merge}
We now study sequential tests that rely on \emph{merging} evidence \citep{wang2025only,ramdas2024hypothesis}. As stated earlier, a wealth process $W^{\textsf{ons}}_i$ quantifies evidence against its null $H_{i,0}$, with more wealth indicating stronger evidence. Naturally, these processes also provide varying degrees of evidence against the global null. Merging looks to exploit this by aggregating these sources of evidence. The test $\phi^{\textsf{bonf}}$ can be thought of as a merging-based test: it combines the wealth processes $\{W^{\textsf{ons}}_{i}\}^k_{i=1}$ by taking their maximum at each time $t$ and dividing by $k$. 
However, other merging strategies---notably the average and the product---have desirable properties (described shortly) for aggregating evidence \citep[Propositions~3.1 \& 4.2]{vovk2021values}. Yet, precise guarantees on their expected stopping times remain unknown, which are critical to conclude on their power in sequential testing. We will now study these two merging functions and derive bounds on their expected stopping time.

\subsection{Product Wealth Process}

\label{sec:merging_product}
Given wealth processes $\{W^{\textsf{ons}}_{i}\}^k_{i=1}$, the product wealth process, $W^{\textsf{prod}} \equiv (W^{\textsf{prod}}_t)_{t\geq 1}$, takes the product of the $k$ individual wealth processes at each time $t$. In other words, 
\begin{align}
    W^{\textsf{prod}}_{t} = \prod^k_{i=1}W^{\textsf{ons}}_{i,t}.
\end{align} 
The process $W^{\textsf{prod}}$ is a test martingale for $H^{\textsf{global}}_0$ and so $\phi^{\textsf{prod}} \equiv (\phi^{\textsf{prod}}_{t})_{t \geq 1}$, where
\(
    \phi^{\textsf{prod}}_{t} = \mathbf{1}\left[W^{\textsf{prod}}_{t} \geq 1/\alpha\right],
\) is a level-$\alpha$ sequential test for $H^{\textsf{global}}_{0}$.

\citet[Proposition~4.2]{vovk2021values} show that this merging rule  {can} ``dominate'' other strategies in the sense that if, at any fixed time $t$, all streams provide evidence (i.e., \(W^{\textsf{ons}}_{i,t} > 1\) for all \(i\)), then the product outperforms a large class of strategies in terms of accumulating wealth\footnote{This dominance result requires the wealth processes be independent.}. It is not clear, however, whether this fixed-time behavior translates into \(\phi^{\textsf{prod}}\) having higher power, or a smaller stopping time in general. We address this by establishing the following expected stopping-time guarantee for \(\phi^{\textsf{prod}}\) under the alternative.





\begin{restatable}[Stopping time of $\phi^{\textsf{prod}}$]{theorem}{prodstop}
\label{theorem:prod_stopping_time}
The test $\phi^{\textsf{prod}}$ is a level-\(\alpha\) sequential test for $H^{\text{global}}_{0}$. Moreover, under $H^{\text{global}}_{1}$, its expected stopping time obeys 
\begin{align}
    \bE[\tau] = \mathcal{O}\left(T^{\textsf{prod}}\right) ~~ \text{where} ~~  T^{\textsf{prod}} = \frac{1}{\Delta_{\textsf{s}}}\ln\frac{1}{\alpha} + \frac{k}{\Delta_{\textsf{s}}}\ln\frac{k}{\Delta_{\textsf{s}} } + \frac{k}{\Delta_{\textsf{s}}^2}\ln\frac{k}{\Delta^2_{\textsf{s}}}.
    \label{eq:prod_stopping_time}
\end{align}
\end{restatable}

A proof is provided in \cref{app:prod_stopping_time} with all constant factors made explicit. A few remarks are in order. For simplicity, consider $\gamma \in [k]$ and suppose under the alternative the first $\gamma$ means are nonzero, i.e., $\Delta_i = \epsilon \neq 0$ when $i \leq \gamma$, and $\Delta_i = 0$ otherwise. In this setting, \cref{eq:prod_stopping_time} evaluates to 
\begin{align*}
    \bE[\tau] = \mathcal{O}\left(\frac{1}{\gamma \epsilon^2}\ln\frac{1}{\alpha} + \frac{k}{\gamma \epsilon^2} \ln \frac{k}{\gamma \epsilon^2} + \frac{k}{\gamma^2 \epsilon^4}\ln\frac{k}{\gamma^2 \epsilon^4}\right).
\end{align*}

\paragraph{Dense alternative $\gamma = k$.} In the large $k$ and small $\alpha$ regime, the dominating term is $\nicefrac{1}{k}\ln \nicefrac{1}{\alpha}$. This is significantly better than the $\ln \nicefrac{k}{\alpha}$ dependence of $\phi^{\textsf{bonf}}$ in \cref{proposition:prop2}. As each wealth process grows, their product $W^{\textsf{prod}}$ provides overwhelming evidence against $H^{\textsf{global}}_{0}$, enabling rapid rejection.

\paragraph{\bf Sparse alternative $\gamma = 1$.} When $k$ is large and $\alpha$ small, the dominating term is $k \ln k + \ln \nicefrac{1}{\alpha}$. This is significantly worse than the dependence exhibited by $\phi^{\textsf{bonf}}$. In this case, the growth of $W^{\textsf{prod}}$ from the single non-null stream is heavily inhibited by the remaining streams fluctuating around $1$.

\subsection{Average Wealth Process}

\label{sec:merging_average}
Given the processes $\{W^{\textsf{ons}}_i\}^k_{i=1}$, the average wealth process $W^{\textsf{ave}} \equiv (W^{\textsf{ave}}_t)_{t\geq 1}$ is defined with 
\begin{align}
    W^{\textsf{ave}}_{t} = \frac{1}{k}\sum^k_{i=1} W^{\textsf{ons}}_{i,t}.
\end{align}
At each time $t$, $W^{\textsf{ave}}_{t}$ is the mean of the $k$ individual wealth processes at that $t$, taking a $\nicefrac{1}{k}$ fraction from each. $W^{\textsf{ave}}$ is a test martingale for $H^{\textsf{global}}_0$ and so the test $\phi^{\textsf{ave}} \equiv (\phi^{\textsf{ave}}_{t})_{t \geq 1}$, with
\(
    \phi^{\textsf{ave}}_{t} = \mathbf{1}\left[W^{\textsf{ave}}_{t} \geq 1/\alpha\right],
\) 
is a level-$\alpha$ sequential test for $H^{\textsf{global}}_0$. 

In fact, this merging rule also ``dominates'' other strategies when the wealth processes are dependent \citep[Proposition 3.1]{vovk2021values}. Conceptually,  at any time $t$, 
if only a few $W^{\textsf{ons}}_{i,t} \geq 1$, the average, unlike the product, will still yield evidence because it merges additively. 
Yet, it is again unclear whether this dominance property leads to smaller stopping times. 
The following result addresses this gap by providing an expected stopping time guarantee for $\phi^{\textsf{ave}}$ under the alternative.




\begin{restatable}[Stopping time of $\phi^{\textsf{ave}}$]{theorem}{avestop}
\label{theorem:ave_stopping_time}
The test $\phi^{\textsf{ave}}$ is a level-\(\alpha\) sequential test for $H^{\text{global}}_{0}$. Moreover, under $H^{\text{global}}_{1}$, its expected stopping time obeys:
\begin{align}
    \bE[\tau] = \mathcal{O}(\min\{T, T^{\textsf{bonf}}\})
\end{align}
where 
\begin{align}
    T = \frac{k}{\Delta_{\textsf{s}}}\ln\frac{1}{\alpha} + \frac{k}{\Delta_{\textsf{s}}}\ln\frac{k}{\Delta_{\textsf{s}} } + \frac{k}{\Delta_{\textsf{s}}^2}\ln\frac{k}{\Delta_{\textsf{s}}^2} \quad \text{and} \quad T^\textsf{bonf} = \frac{1}{\Delta_{\textsf{m}}^2}\ln\frac{k}{\alpha\cdot \Delta^2_{\textsf{m}}}.
\end{align}
\end{restatable}

A proof is provided in \cref{app:ave_stopping_time}. Again, we pause to discuss how the bound behaves as the number of streams with nonzero means changes. As earlier, consider $\gamma \in [k]$, and suppose under the alternative that $\Delta_i = \epsilon$ when $i \leq \gamma$, and $\Delta_i = 0$ otherwise. Then, the bound evaluates to 
\begin{align*}
    \bE[\tau] = \calO\left(\frac{k}{\gamma\epsilon^2}\ln \frac{1}{\alpha} + \frac{k}{ \gamma\epsilon^2}\ln\frac{k}{\gamma\epsilon^2} + \frac{k}{\gamma^2\epsilon^4} \ln \frac{1}{\gamma^2\epsilon^4}\right).
\end{align*}

\paragraph{Dense alternative $\gamma = k$.} When $k$ is large and $\alpha$ is small, $T^{\textsf{bonf}}$ and $T$ are dominated by $\ln \nicefrac{k}{\alpha}$ and $\ln \nicefrac{1}{\alpha}$ respectively. Hence, the upper bound is dominated by $\ln \nicefrac{1}{\alpha}$. In this case, $\phi^{\textsf{ave}}$ does provide a benefit over $\phi^{\textsf{bonf}}$, just not to the same extent as $\phi^{\textsf{prod}}$.

\paragraph{Sparse alternative $\gamma = 1$.} In the large $k$ and small $\alpha$ regime, $T$ is dominated by $k \ln \nicefrac{1}{\alpha}$, whereas $T^{\textsf{bonf}}$ is dominated by $\ln \nicefrac{k}{\alpha}$. Hence, the bound matches that of $\phi^{\textsf{bonf}}$. This is because only $W^{\textsf{ons}}_1$ grows over time while the other wealth process $W^{\textsf{ons}}_2, \ldots, W^{\textsf{ons}}_k$ fluctuate around small values. In turn, the average wealth, $W^{\textsf{ave}}_t = \frac{1}{k} W^{\textsf{ons}}_{1,t}$, remains close to the maximum wealth, $\big(\frac{1}{k} \max_{i \in [k]} W^{\textsf{ons}}_{i,t}\big)_{t \geq 1}$, which is the process tracked by $\phi^\textsf{bonf}$. Thus, $\phi^{\textsf{ave}}$ behaves almost identically to $\phi^{\textsf{bonf}}$.

This analysis indicates that $\phi^{\textsf{ave}}$ is effective when evidence is concentrated in only a few streams, and also performs well, just not as good as $\phi^{\textsf{prod}}$, when many streams contain evidence against the null. \cref{tab:bound_comparison} shows the comparison between the tests in the large $k$, small $\alpha$ regime, showing the complementary strengths of the two tests. 

\subsection{A Balanced Sequential Test}
\label{sec:balanced_test}
So far, we have analyzed two merging-based tests beyond Bonferroni, $\phi^{\textsf{prod}}$ and $\phi^{\textsf{ave}}$, which work well in the sparse and dense regimes, respectively. Naturally, one does not know ahead of time which alternative one is in. If we use $\phi^{\textsf{ave}}$ when the alternative is dense, we will end up waiting a longer time to reject in comparison to had we used the product. Similarly, if we use $\phi^{\textsf{prod}}$ when the alternative is sparse, we will likely reject much later than we would have had we used the average. Luckily, we can get around this issue by building a wealth process that balances both $W^{\textsf{ave}}$ and $W^{\textsf{prod}}$, yielding a test that performs well in either regime.
\begin{table}[t]
    \caption{Comparison of the expected stopping time upper bounds, showing the dominating terms and their dependence on $k$ and $\alpha$ for large $k$ and small $\alpha$.}
    \label{tab:bound_comparison}
    \centering
    \begin{tabular}{ccc}
         Test & \multicolumn{2}{c}{Alternative} \\
         & Sparse $\gamma = 1$ & Dense $\gamma = k$\\
         \toprule
         Bonferroni $\phi^{\textsf{bonf}}$ (\cref{proposition:prop2}) & $ \ln \nicefrac{1}{\alpha} + \ln k$ & $ \ln \nicefrac{1}{\alpha} + \ln k$ \\
         Product $\phi^\textsf{prod}$ (\cref{theorem:prod_stopping_time}) & $\ln \nicefrac{1}{\alpha} + k \ln k$&$\nicefrac{1}{k} \ln \nicefrac{1}{\alpha}$\\
         Average $\phi^\textsf{ave}$ (\cref{theorem:ave_stopping_time}) & $\ln \nicefrac{1}{\alpha} + \ln k$&$\ln \nicefrac{1}{\alpha}$\\
         \midrule
         Balanced $\phi^\textsf{balance}$ (\cref{theorem:balance_stopping_time}) & $\ln \nicefrac{1}{\alpha} + \ln k$&$\nicefrac{1}{k} \ln \nicefrac{1}{\alpha}$\\
         \bottomrule
    \end{tabular}
\end{table}

Given $W^{\textsf{prod}}$ and $W^{\textsf{ave}}$, we can define the balanced process 
$W^{\textsf{balance}} \equiv (W^{\textsf{balance}}_{t})_{t\geq1}$ as
\begin{align}
    W^{\textsf{balance}}_{t} = \frac{1}{2}W^{\textsf{ave}}_{t} + \frac{1}{2}W^{\textsf{prod}}_{t}
\end{align}
and the associated test $\phi^{\textsf{balance}} \equiv (\phi^{\textsf{balance}}_t)_{t\geq1}$ where $\phi^{\textsf{balance}}_t = \mathbf{1}\left[W^{\textsf{balance}}_{t} \geq 1/\alpha\right]$. Since $W^{\textsf{ave}}$ and $W^{\textsf{prod}}$ are both test martingales for $H^{\textsf{global}}$, their average is as well \citep{vovk2021values}. Moreover, the following result provides an expected stopping time guarantee for $\phi^{\textsf{balance}}$ under the alternative, showing that it inherits the strengths of $\phi^{\textsf{prod}}$ and  $\phi^{\textsf{ave}}$.

\begin{restatable}[Stopping time of $\phi^{\textsf{balance}}$]{theorem}{balancestop}
\label{theorem:balance_stopping_time}
The test $\phi^{\textsf{balance}}$ is a level-\(\alpha\) sequential test for $H^{\textsf{global}}_0$. Moreover, under $H^{\textsf{global}}_1$, its expected stopping time obeys $\bE[\tau] \leq \min\{T^{\textsf{prod}}, T^{\textsf{bonf}}\}$.
\end{restatable}

A proof is provided in \cref{app:balance_stopping_time}. Crucially, when the alternative is dense, i.e. evidence is dispersed across a majority of streams, $\phi^{\textsf{balance}}$ behaves like $\phi^{\textsf{prod}}$. On the other hand, when the alternative is sparse, i.e., evidence is contained in just a few streams, $\phi^{\textsf{balance}}$ behaves like $\phi^{\textsf{ave}}$ (and  $\phi^\textsf{bonf}$). In this way, $\phi^{\textsf{balance}}$ always rejects nearly as quickly as the best possible test, regardless of the sparsity of the alternative. \cref{tab:bound_comparison} presents the stopping time guarantees for the two sparsity regimes across all the tests discussed, showing that $\phi^{\textsf{balance}}$ obtains the best of both worlds.

Before moving on, we note that this is not the only way of merging processes with improved expected stopping times. For instance, another alternative could rely on constructing a test that rejects whenever the product and Bonferroni tests reject by employing a union bound argument. In other words, rejecting 
whenever $\max\{W^\textsf{prod}_t, W^\textsf{bonf}_t\} \geq \nicefrac{2}{\alpha}$, where $W^\textsf{bonf}_t = \nicefrac{1}{k} \max_{i \in [k]} W^\textsf{ons}_{i,t}$. However, this test has the same expected stopping time as that for $\phi^{\textsf{balance}}$, and its wealth process is never higher than that of $W^{\textsf{balance}}$.

\section{Experiments}
\label{sec:experiments}
We now validate our theory on synthetic and real world data, focusing on how different levels of non-null streams affect the stopping time of the different tests.

\subsection{Synthetic}
We begin with a synthetic example with $k = 250$ streams, of which $k_1<k$ are non-nulls. For $\left\lfloor \frac{k_1}{k}\right\rfloor \in \{0.05, 0.30, 0.75\}$ fraction of streams, $Z_{i,t} \sim \text{Uniform}(a, b)$, with $a$ and $b$ chosen such that $|\mathbb{E}_{P_i}[Z]| = 0.1$ and $\text{Var}_{P_i}[Z] = 0.2$. For the remaining $k - k_1$ streams, $Z_{i,t} \sim \text{Uniform}\left(-\sqrt{\nicefrac{3}{5}}, \sqrt{\nicefrac{3}{5}}\right)$, yielding $\mathbb{E}_{P_i}[Z] = 0$ and $\text{Var}_{P_i}[Z] = 0.2$. 
We use $\phi^{\textsf{bonf}}, ~\phi^{\textsf{ave}},~\phi^{\textsf{prod}}$, and $\phi^{\textsf{balance}}$ to test the global null. Results for additional experiments with other values of $k$ and proportions $\left\lfloor \frac{k_1}{k} \right\rfloor$, are presented in \cref{app:add_exp_results}. This experiment is run 1,000 times.

\begin{figure*}[h!]
    \centering
    \hspace{-0.5em}
    \begin{subfigure}[t]{0.32\linewidth}
        \centering
        \includegraphics[width=\linewidth]{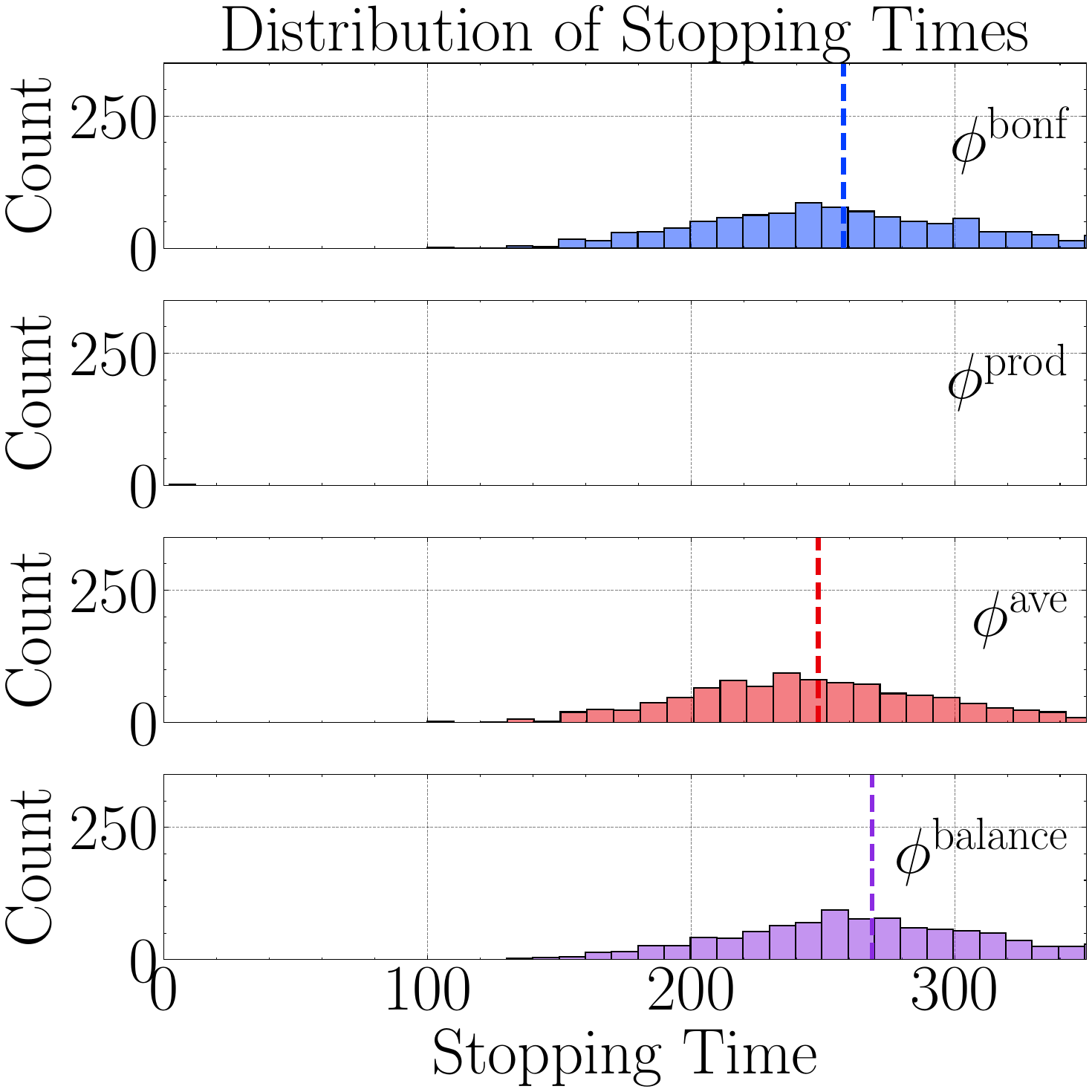}
    \end{subfigure}
    \begin{subfigure}[t]{0.32\linewidth}
        \centering
        \includegraphics[width=\linewidth]{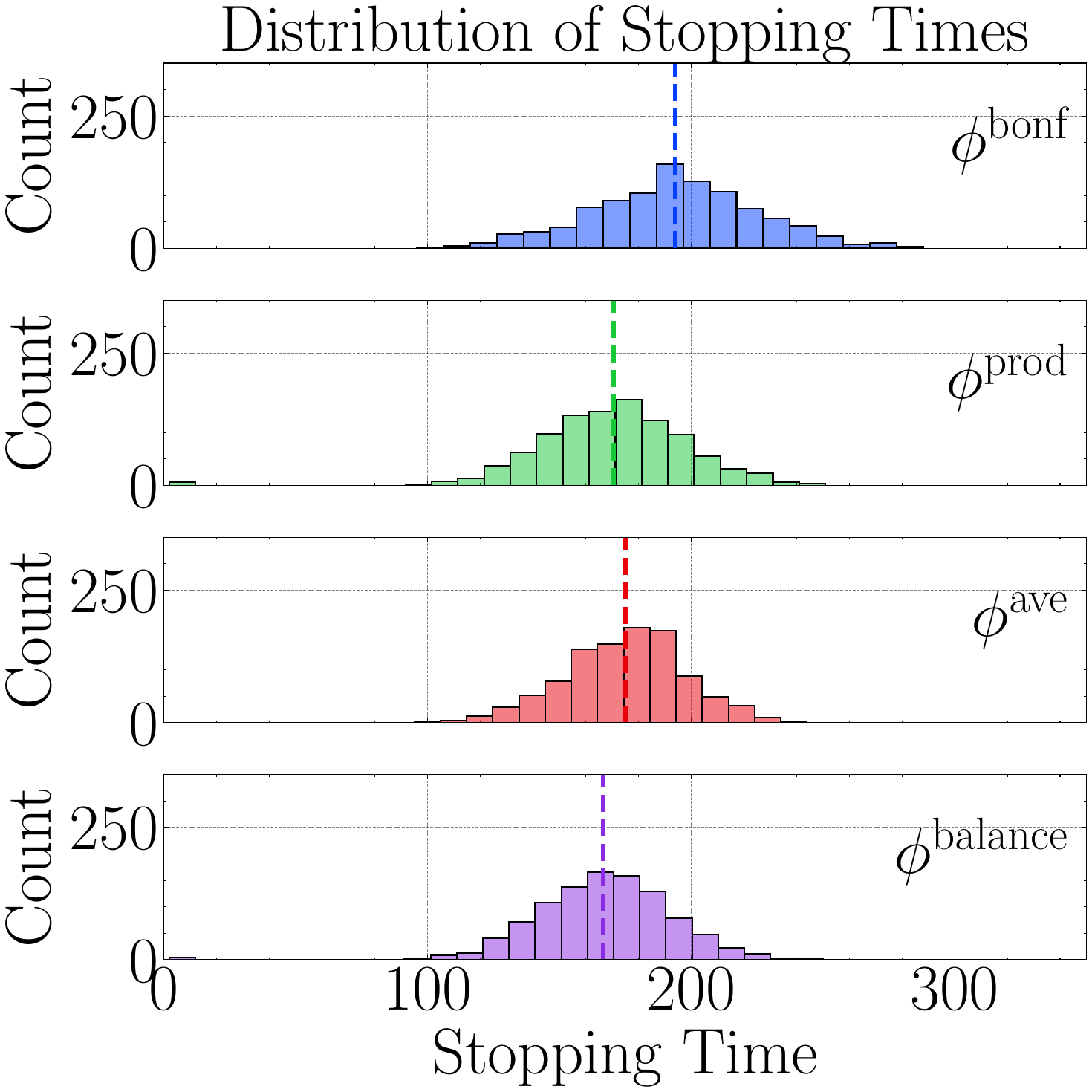}
    \end{subfigure}
    \begin{subfigure}[t]{0.32\linewidth}
        \centering
        \includegraphics[width=\linewidth]{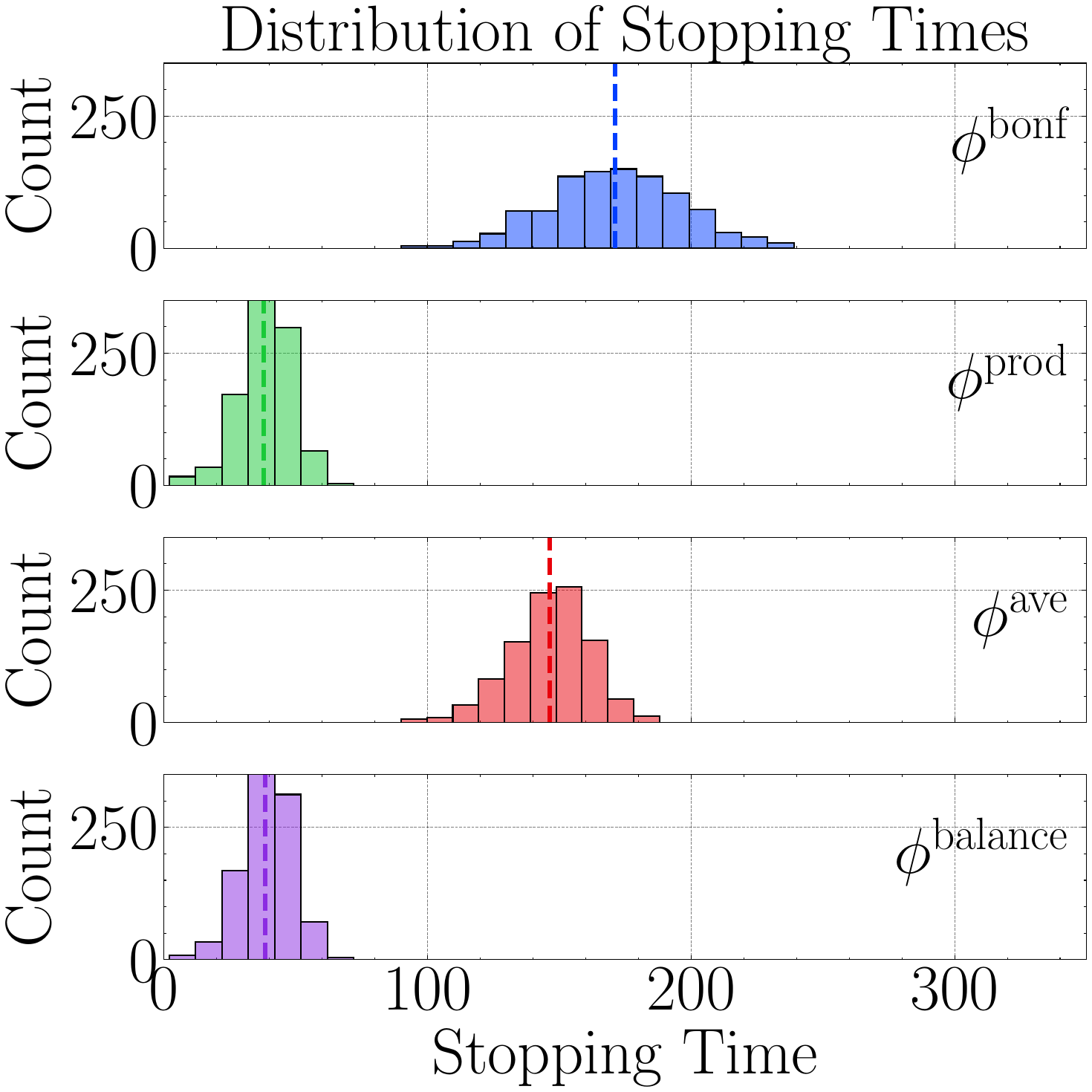}
    \end{subfigure}
    \vspace{1em}
    
    \hspace{-1em}
    \begin{subfigure}[t]{0.32\linewidth}
        \includegraphics[width=\linewidth]{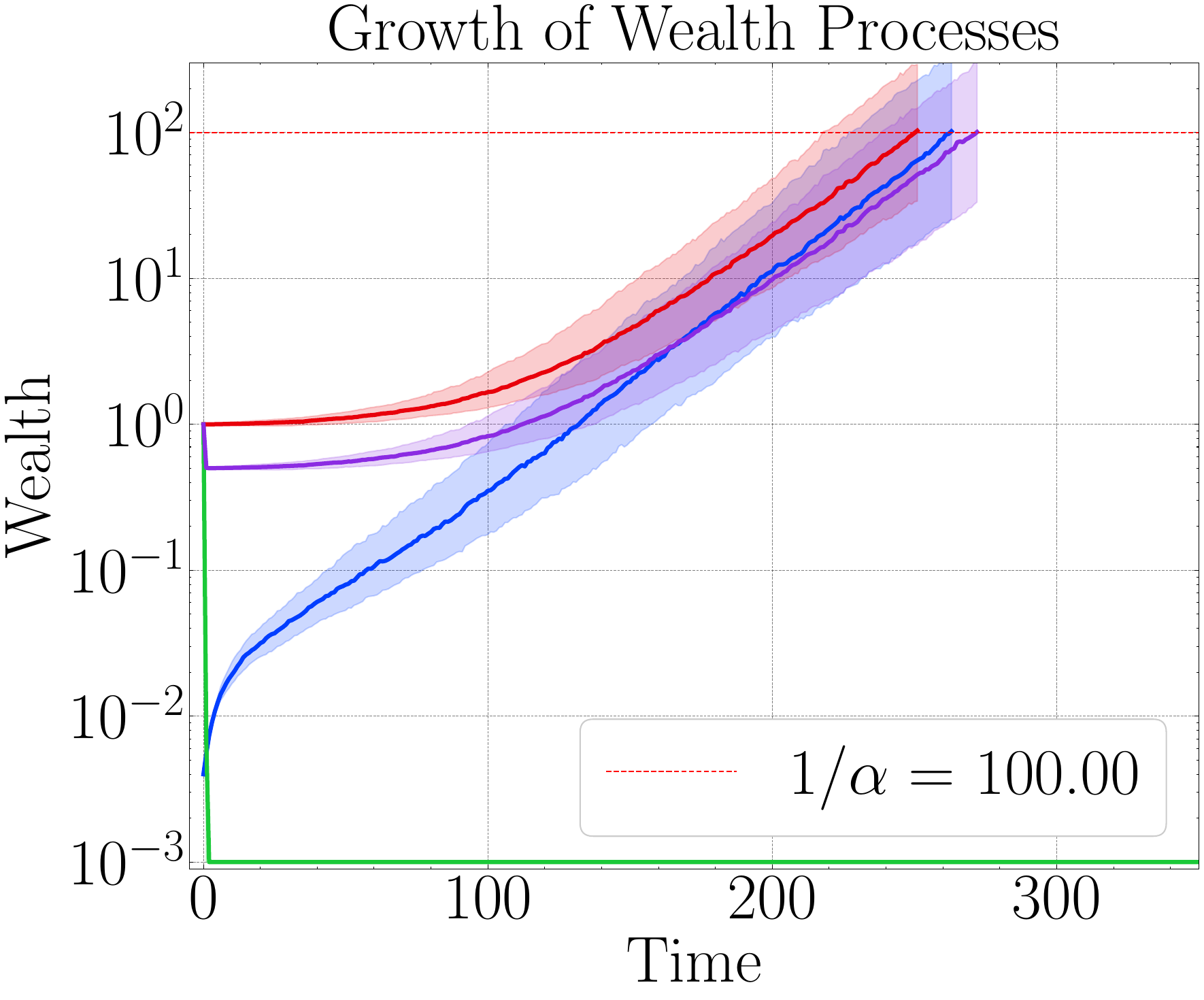}
    \end{subfigure}
    \begin{subfigure}[t]{0.32\linewidth}
        \includegraphics[width=\linewidth]{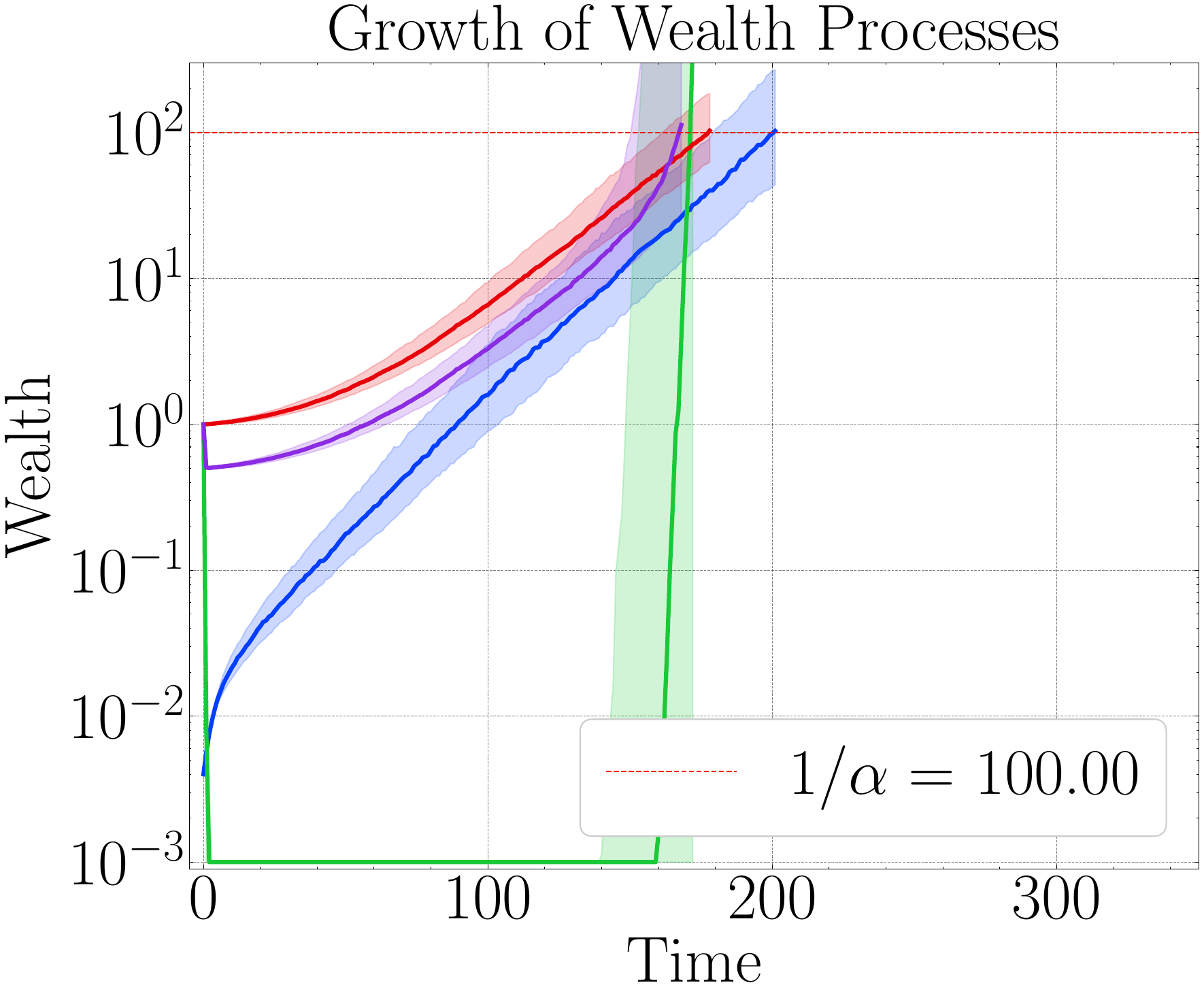}
    \end{subfigure}
    \begin{subfigure}[t]{0.32\linewidth}
        \includegraphics[width=\linewidth]{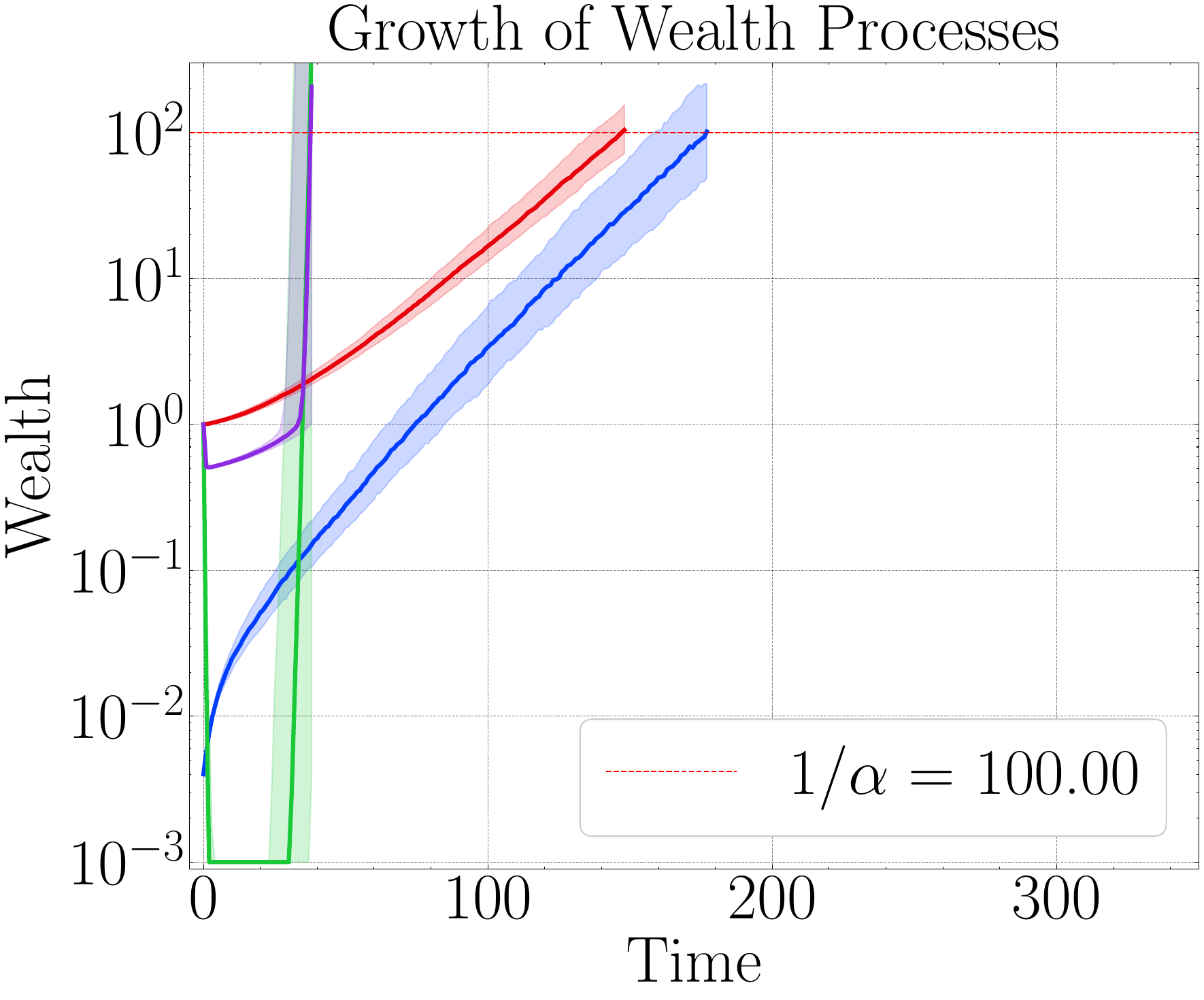}
    \end{subfigure}
    
    \vspace{-0.1em}
    \centering
    \hspace{1em}
    \includegraphics[width=0.7\textwidth]{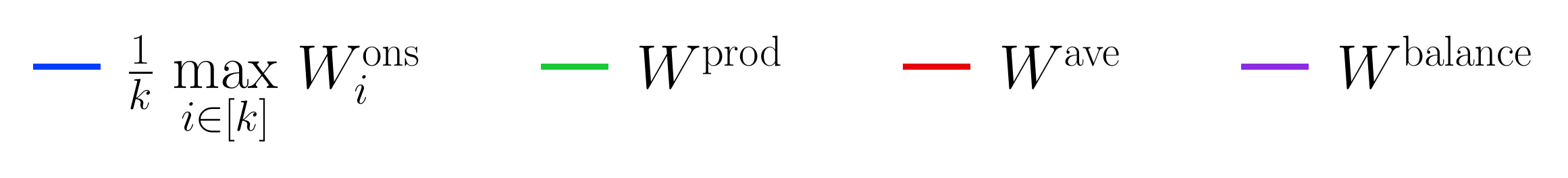}
    
    \vspace{-2em}
    \begin{subfigure}[t]{0.32\linewidth}
        \caption{$\left\lfloor \frac{k_1}{k}\right\rfloor = 0.05$}
        \label{fig:wealth_1}
    \end{subfigure}
    \begin{subfigure}[t]{0.32\linewidth}
        \caption{$\left\lfloor \frac{k_1}{k}\right\rfloor = 0.30$}
        \label{fig:wealth_2}
    \end{subfigure}
    \begin{subfigure}[t]{0.32\linewidth}
        \caption{$\left\lfloor \frac{k_1}{k}\right\rfloor = 0.75$}
        \label{fig:wealth_3}
    \end{subfigure}
    \vspace{0.5em}
    
    \caption{{\bf Top:} Distribution of stopping times, over 1,000 simulations, for various sequential tests across settings with varying proportions of streams with nonzero means. A test rejects when its corresponding wealth process exceeds $\nicefrac{1}{\alpha}$ for $\alpha = 0.01$. The dashed vertical line is the empirical mean of the stopping times. {\bf Bottom:} Trajectories of various wealth processes across settings with different amounts of nonzero means. Each line represents the median trajectory of a wealth process over 1,000 simulations, with shaded areas indicating the 25\% and 75\% quantiles. The y-axis is presented on a logarithmic scale. Wealth processes are clipped to $10^{-3}$ for visualization purposes.}
    \label{fig:synthetic_wealths}
\end{figure*}

\paragraph{Sparse alternative $\left(\left\lfloor \frac{k_1}{k}\right\rfloor = 0.01\right)$}. The stopping time distributions in \cref{fig:wealth_1} illustrate that when 1\% of streams have nonzero means, $\phi^{\textsf{bonf}}$ and  $\phi^{\textsf{ave}}$ have the smallest stopping times. In contrast, $\phi^{\textsf{prod}}$ fails to reject even after 1,000 time steps (only the first 350 time steps are shown). Notably, the stopping time distribution of $\phi^{\textsf{balance}}$, shown in the top panel, is nearly identical to those of $\phi^{\textsf{bonf}}$ and  $\phi^{\textsf{ave}}$. The processes tracked by these tests are nearly identical and grow at a similarly fast rate. In contrast, $W^{\textsf{prod}}$ does not grow, causing $\phi^{\textsf{prod}}$ to not reject.

\paragraph{Moderate alternative $\left(\left\lfloor \frac{k_1}{k}\right\rfloor = 0.30\right)$}. When the proportion of streams with nonzero means increases to 30\%, the behavior of the tests begin to change. The stopping time distribution in \cref{fig:wealth_2} shows that all of the tests have nearly identical distributions. This finding is supported by \cref{fig:wealth_2}, where we see all the wealth processes cross $\nicefrac{1}{\alpha}$ around $t \approx 140 - 200$. The interesting finding is how the wealth processes grow. The wealth processes tracked by $\phi^{\textsf{bonf}}$, $\phi^{\textsf{ave}}$, and $\phi^{\textsf{balance}}$ remain nearly identical and grow at the same rate. In contrast, $W^{\textsf{prod}}$ initially decreases, approaching zero, before finally rapidly increasing around $t \approx 160$.

\paragraph{Dense alternative $\left(\left\lfloor \frac{k_1}{k}\right\rfloor = 0.75\right)$}. When 75\% of streams have nonzero means, the behavior of the tests and wealth processes change significantly. The stopping time distributions in \cref{fig:wealth_3} show that $\phi^{\textsf{prod}}$ has the smallest times, with $\phi^{\textsf{balance}}$ nearly matching its behavior. On the other hand, $\phi^{\textsf{bonf}}$ and $\phi^{\textsf{ave}}$ have similar, but larger stopping times than $\phi^{\textsf{prod}}$ and $\phi^{\textsf{balance}}$. The wealth processes $W^{\textsf{prod}}$ and $W^{\textsf{balance}}$ are nearly identical, decreasing towards zero briefly before rapidly increasing and reaching $\nicefrac{1}{\alpha}$ in $t \leq 50$. The wealth processes tracked by $\phi^{\textsf{bonf}}$ and $\phi^{\textsf{ave}}$ are all similar, growing slower and more gradually than $W^{\textsf{prod}}$ and $W^{\textsf{balance}}$ and rejecting at a later time.

\subsection{Zero-shot medical image classification}
\begin{figure*}
    \centering
    \begin{subfigure}{0.47\linewidth}
        \centering
        \includegraphics[width=\linewidth]{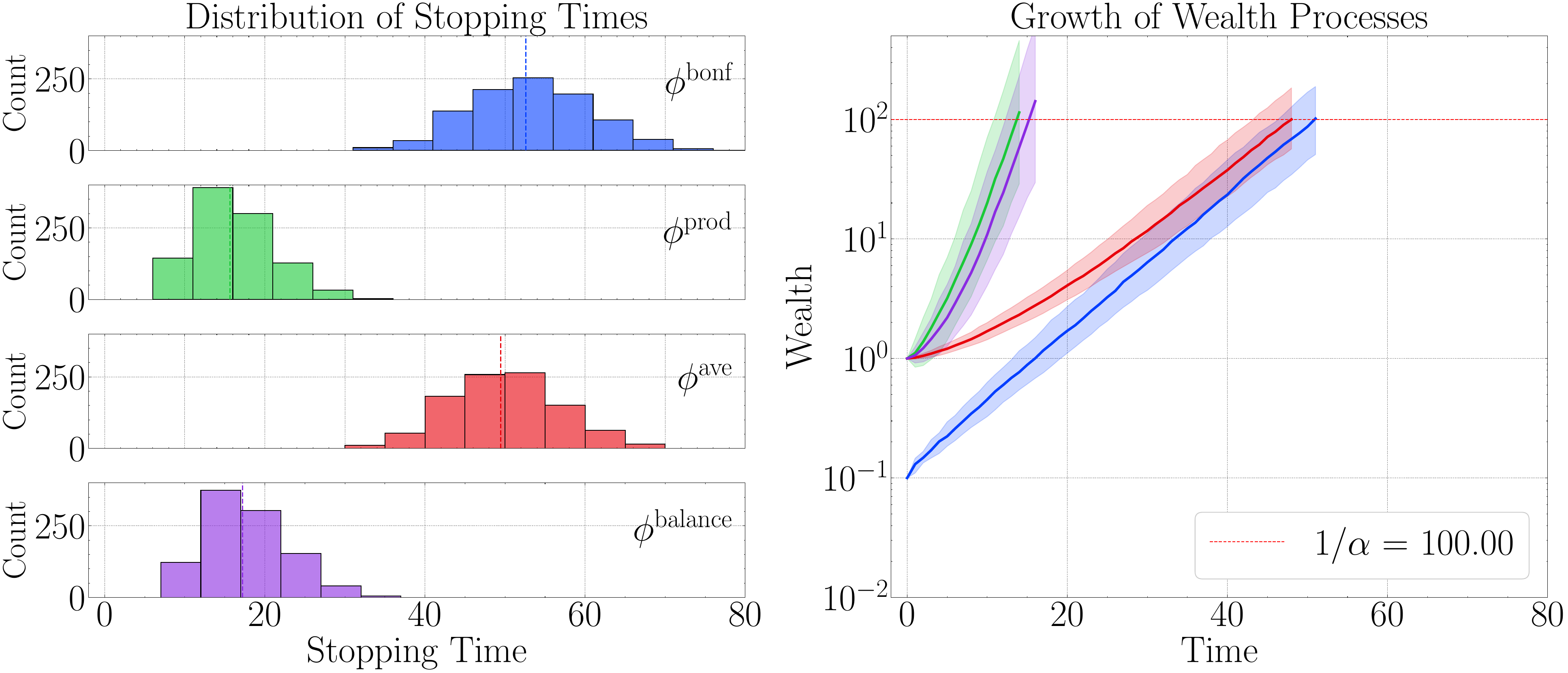}
        \label{fig:conceptCLIP_pre}
    \end{subfigure}
    \hspace{0.03\linewidth} 
    \begin{subfigure}{0.47\linewidth}
        \centering
        \includegraphics[width=\linewidth]{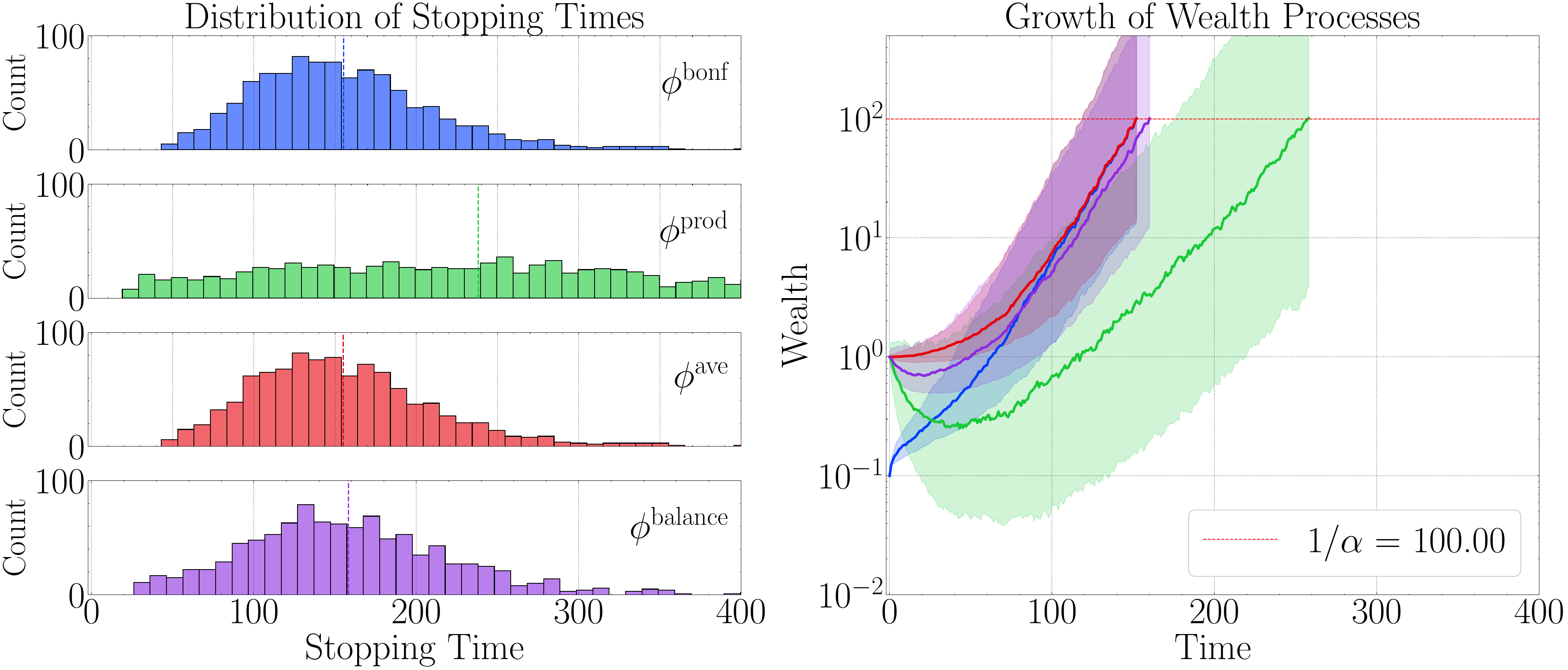}
        \label{fig:conceptclip_post}
    \end{subfigure}

    \vspace{-1.5em}
    \centering
    \hspace{1em}
    \includegraphics[width=0.65\textwidth]{figures/legend_horizontal.pdf}

    \vspace{-2em}
    \begin{subfigure}[t]{0.48\linewidth}
        \caption{Zero-shot multiaccuracy testing results.}
    \end{subfigure}
    \begin{subfigure}[t]{0.48\linewidth}
        \caption{Post-processed multiaccuracy testing results.}
    \end{subfigure}

    \caption{{\bf Left plot of each figure:} Distribution of stopping times, over 1,000 runs, for various sequential tests. A test rejects when its corresponding wealth process exceeds $\nicefrac{1}{\alpha}$ for $\alpha = 0.01$. The dashed vertical line is the empirical mean of the stopping times. {\bf Right plot of each figure:} Various wealth process trajectories. Each line represents the median trajectory of a wealth process over 1,000 runs, with shaded areas indicating the 25\% and 75\% quantiles.}
    \label{fig:conceptclip_results}
\end{figure*}

We evaluate ConceptCLIP \citep{nie2025conceptclip}, the first medical vision–language pretraining model designed to perform accurate prediction tasks on medical images across diverse modalities.  We assess whether the zero-shot ConceptCLIP model $f(X)$ is \emph{multiaccurate} \citep{kim2019multiaccuracy, hebert2018multicalibration} across $k=10$ groups, i.e. we evaluate if the predictions of $f$ are unbiased on every group. The groups we consider are defined by intersections of imaging modalities and anatomical regions, such as \texttt{lung computed tomography (CT)}, \texttt{chest X-ray}, \texttt{breast mammogram}, \texttt{retinal optical ct (OCT)}, and more. Additional details regarding the datasets are in \cref{app:add_exp_results}.

Specifically, we consider $k$ streams of medical image–label pairs $(X_{i,t}, Y_{i,t})_{t \geq 1}$, where $(X_{i,j}, Y_{i,j}) \sim P_i$ for each group $i \in [k]$, and test the global null hypothesis:
\begin{align}
    H^{\text{global}}_0: \forall i \in [k], \ \bE_{P_i}[Z] = 0.
\end{align}
where $Z = f(X) - Y$. Empirical estimates of $\bE_{P_i}[Z]$ for all groups are displayed in \cref{app:conceptclip_errors_table}, provided in \cref{app:add_exp_results}. In this experiment, we analyze which tests reject the null the fastest.

\paragraph{\bf Zero-shot multiaccuracy}. From  \cref{app:conceptclip_errors_table}, we see that ConceptCLIP makes biased predictions on a majority of groups, as $|\bE_{P_i}[Z]| \geq 0.05$ for eight of ten groups. The stopping times for the various tests and wealth processes displayed in \cref{fig:conceptCLIP_pre} support this finding. The tests $\phi^{\textsf{prod}}$ and $\phi^{\textsf{balance}}$ reject the quickest, on average, and their wealth processes grow rapidly. In contrast $\phi^{\textsf{bonf}}$ and $\phi^{\textsf{ave}}$ reject at later times. This is reflected in their wealth processes, which grow slower than $W^{\textsf{prod}}$ and $W^{\textsf{balance}}$.

\paragraph{Multiaccuracy adjustment}. For illustrative purposes, we post-process the predictions of ConceptCLIP to be multiaccurate \citep{hebert2018multicalibration} on all images except those of \texttt{colon endoscopy}, to simulate a sparse alternative where only stream has a nonzero mean. That is, we enforce $\bE_{P_i}[Z] \approx 0$ for all groups except for \texttt{colon endoscopy} images, which has $\bE_{P_i}[Z] \approx -0.09$. \cref{fig:conceptclip_post} displays the stopping times for the various tests and wealth processes and, as expected, we see that $\phi^{\textsf{bonf}}$ and $\phi^{\textsf{ave}}$ reject the quickest, as their wealth processes grow the fastest. Importantly, $\phi^{\textsf{balance}}$ nearly matches the performance of these tests. On the other hand, the remaining tests also reject, just at later times as their wealth processes grow more gradually.

\section{Conclusion}
\label{sec:conclusion}
We studied a testing problem where given several parallel streams of data, reflecting the performance of an ML system, the goal is to raise an alarm whenever any of the streams shows unusual behavior. We modeled this as a sequential hypothesis testing problem, gathering evidence against the global null hypothesis that all data streams have mean zero. Utilizing tools in mean testing for a single stream, we highlight that the standard global sequential test that employs a Bonferroni correction is not always efficient, mainly because of its inability to adapt to the number of streams with non-zero means. Thus, we studied test that relied on different strategies of merging evidence from multiple streams. We provide new stopping time results for tests that employ the average and product merging rules. Subsequently, we demonstrate that a balanced test, which balances these two merging strategies, has improved expected rejection times, rejecting faster than all the other tests while having controlled type-I error. Importantly, the balanced test draws the benefits of both product and average merging processes without having to know in advance the kind of alternative (sparse or dense). Finally, we empirically validated our results on both synthetic and real data, showing that our theoretical conclusions lead to faster and improved practical methods for flagging erroneous model behavior in multi-stream settings.


\bibliographystyle{plainnat}
\bibliography{arxiv/bibliography}
\newpage
\appendix 

\section{Proofs \label{app:proofs}}
For the following proofs, we recall our setting established in \cref{sec:prob_form}, introduce some additional notation, and define some additional properties of stochastic processes.

{\noindent{{\bf Setting, Notation, \& Definitions.}}} We have $k$ parallel data streams $(Z_{1,t})_{t \geq 1}, \dots, (Z_{k,t})_{t \geq 1}$. For each $i \in [k]$, the $i^{\text{th}}$ stream consists of i.i.d points and we denote its mean as $\Delta_i = \bE[Z_{i,1}]$. Moreover, we define $\Delta_{\textsf{m}} = \max_{i \in k} \Delta_i$, and $\Delta_{\textsf{s}} = \sum^k_{i=1} \Delta^2_i$. By definition, under the null $H^{\textsf{global}}_0$, $\Delta_i$, $\Delta_{\textsf{m}}$, and $\Delta_{\textsf{s}}$ equal zero. Under the alternative $H^{\textsf{global}}_1$, $\exists i \in [k]$, such that $\Delta_i \neq 0$. As a result, under $H^{\textsf{global}}_1$, $\Delta_{\textsf{m}} \neq 0$ and $\Delta_{\textsf{s}} \neq 0$. A stochastic process $M \equiv (M_t)_{t\geq 1}$ is adapted to the stream $(Z_{t})_{t \geq 1}$ if $\mathbb{E}_P[M_t | Z_1, \ldots, Z_{t}] = M_t$ $\forall t \geq 1$, and predictable if $\mathbb{E}_P[M_t | Z_1, \ldots, Z_{t-1}] = M_t$ $\forall t \geq 1$.

\subsection{Proof of level-$\alpha$ control and stopping time of $\phi^{\textsf{ons}}_i$ \label{app:prop1}}
\begin{proof} We first show that $\phi^{\textsf{ons}}_{i} \equiv (\phi^{\textsf{ons}}_{i,t})_{t\geq1}$ is a level-$\alpha$ sequential test for $H_{i,0}$. Then we derive an upper bound on its expected stopping time. The proof technique comes from \citet{chugg2023auditing}.

{\noindent {\bf Level-$\alpha$ test}}. Let $\alpha \in (0,1)$. Consider any $i \in [k]$. Recall $\phi^{\textsf{ons}}_{i,t} = \mathbf{1}[W^{\textsf{ons}}_{i,t} \geq 1/\alpha]$, where
\begin{align}
    W^{\textsf{ons}}_{i,t} = \prod_{j=1}^{t} \left(1 + \lambda_{i,j} \cdot Z_{i,j}\right)
\end{align}
and the sequence $(\lambda_{i,t})_{t\geq1}$ is determined by ONS (\cref{alg:ons}). First, by construction the wealth process $W^{\textsf{ons}}_{i}$ has initial value $W^{\textsf{ons}}_{i,0} = 1$. Second, $(\lambda_{i,t})_{t\geq1}$ is a $[-\nicefrac{1}{2}, \nicefrac{1}{2}]$-valued predictable sequence, thus $W^{\textsf{ons}}_{i}$ is adapted and $W^{\textsf{ons}}_{i,t} \geq 0$ for all $t \geq 1$. Furthermore, for any $P \in H_{i,0}$,
\begin{align*}
    \bE_{P}[W^{\textsf{ons}}_{i,t} \mid Z_{i,1}, \dots, &Z_{i,t-1}] = \bE_{P}\left[\prod_{j=1}^{t} \left(1 + \lambda_{i,j} \cdot Z_{i,j}\right) \mid Z_{i,1}, \dots, Z_{i,t-1} \right]\\
    &= \prod^{t-1}_{j=1}\left(1 + \lambda_{i,j} \cdot Z_{i,j}\right)\cdot \bE_{P}\left[1 + \lambda_{i,t} \cdot Z_{i,t}\mid Z_{i,1}, \dots, Z_{i,t-1}\right]\\
    &= W^{\textsf{ons}}_{i,t-1}\cdot(1 + \lambda_{i,t}\cdot\bE_{P}\left[Z_{i,t} \mid Z_{i,1}, \dots, Z_{i,t-1}\right]) = W^{\textsf{ons}}_{i,t-1}\cdot(1 + 0) = W^{\textsf{ons}}_{i,t-1}.
\end{align*}
Therefore $(W^{\textsf{ons}}_{i,t})_{t\geq 1}$ is a test martingale for $H_{i,0}$. Thus, by \nameref{theorem:villes},
\begin{align}
    \sup_{P \in H_{i,0}} P(\exists t \geq 1: \phi^{\textsf{ons}}_{i,t} = 1) = \sup_{P \in H_{i,0}} P(\exists t \geq 1: W^{\textsf{ons}}_{i,t} \geq 1/\alpha) \leq \alpha.
\end{align}

{\noindent {\bf Expected stopping time.}}
Let $\alpha \in (0,1)$. By definition, $\tau_i := \min\{t\geq 1: \phi^{\textsf{ons}}_{i,t} = 1\} = \min\{t \geq 1: W^{\textsf{ons}}_{i,t} \geq 1/\alpha\}$. Since $\tau_i$ is a non-negative integer valued random variable, for any $P \in H_{i,1}$
\begin{align}
    \bE_P[\tau_i] &= \sum^{\infty}_{t=0}P(\tau_i > t) \\
    &= 1 + \sum^{\infty}_{t=1}P(\tau_i > t) \\
    &\leq 1 + \sum^{\infty}_{t=1}P(W^{\textsf{ons}}_{i,t} < 1/\alpha)\\
    &= 1 + \sum^{\infty}_{t=1}P(\ln(W^{\textsf{ons}}_{i,t}) < \ln(1/\alpha)) = 1+ \sum^{\infty}_{t=1}P(E_{i,t})
\end{align}
where $E_{i,t} = \{\ln(W^{\textsf{ons}}_{i,t}) < \ln(1/\alpha)\}$. Now, defining
\begin{align}
    A_{i,t} = \sum^t_{j=1} Z_{i,j} \quad V_{i,t} = \sum^t_{j=1} (Z_{i,j})^2,
\end{align}
by \cref{lemma:log_wealth_lb} we have
\begin{align}
    \ln(W^{\textsf{ons}}_{i,t}) &\geq \frac{(A_{i,t})^2}{4(V_{i,t} + |A_{i,t}|)} - 2\ln\left(4t\right), \quad \forall t \geq 1.
\end{align}
Furthermore, $\forall t \geq 1$, we know $|A_{i,t}| \leq \sum_{j\leq t} |Z_{i,j}| \leq t$ and $V_{i,t} \leq t$. Thus,
\begin{align}
    E_{i,t} &\subseteq \left\{\frac{(A_{i,t})^2}{4(V_{i,t} + |A_{i,t}|)} - 2\ln\left(4t\right) < \ln(1/\alpha) \right\}\\
    &= \left\{ (A_{i,t})^2 < 4(V_{i,t} + |A_{i,t}|) (\ln(1/\alpha) + 2\ln(4t)) \right\}\\
    &\subseteq \left\{(A_{i,t})^2 < 8t\ln(1/\alpha) + 16t\ln(4t) \right\}\\
    &\subseteq \left\{ (A_{i,t})^2 < 16t\ln(4t/\alpha) \right\}\\
    &\subseteq \left\{ |A_{i,t}| < 4\sqrt{t\ln(4t/\alpha)} \right\}.
\end{align}
Since $A_{i,t} = \sum_{j \leq t} Z_{i,t}$ where $Z_{i,j} \in [-1,1]$ by \cref{lemma:concentration}, with probability at least $1 - \nicefrac{1}{t^2}$
\begin{align}
    |A_{i,t}|  &\geq \bE[A_{i,t}] - \sqrt{4t\ln(2t)} = t\cdot\left|\Delta_i\right| - \sqrt{4t\ln(2t)}.
\end{align}
By \cref{lemma:loglemma}, 
\begin{align}
    \frac{1}{2}t\cdot\left|\Delta_i\right| &\geq \sqrt{4t\ln(2t)} ~~ \text{for all}~~ t \geq \frac{32}{\left|\Delta_i\right|^2}\ln\left(\frac{32}{\left|\Delta_i\right|^2}\right)\\
    \frac{1}{2}t\cdot\left|\Delta_i\right| &\geq 4\sqrt{t\ln(4t/\alpha)} ~~ \text{for all}~~ t \geq \frac{128}{\left|\Delta_i\right|^2}\ln\left(\frac{256}{\alpha \cdot \left|\Delta_i\right|^2}\right).
\end{align}
As a result, for
\begin{align}
    t \geq T := \frac{128}{\left|\Delta_i\right|^2}\ln\left(\frac{256}{\alpha \cdot \left|\Delta_i\right|^2}\right) + \frac{32}{\left|\Delta_i\right|^2}\ln\left(\frac{32}{\left|\Delta_i\right|^2}\right) = \mathcal{O}\left(\frac{1}{\Delta_i^2}\ln \frac{1}{\alpha \cdot \Delta_i^2}\right).
\end{align}
we have
\begin{align}
    t\cdot\left|\Delta_i\right| \geq 4\sqrt{t\ln(4t/\alpha)} + \sqrt{4t\ln(2t)}.
\end{align}
Therefore, by the law of total probability, for $t \geq T$
\begin{align}
    P(E_{i,t}) &\leq P\left(|A_{i,t}| < 4\sqrt{t\ln(4t/\alpha)} \right) \leq 1/t^2
\end{align}
and so we can conclude
\begin{align}
    \bE[\tau_i] \leq 1 + \sum^\infty_{t=1}P(E_{i,t}) &= 1+ \sum^T_{t=1}P(E_{i,t}) + \sum^\infty_{t=T}P(E_{i,t}) \leq 1 + T + \sum^\infty_{t=T}\frac{1}{t^2} \leq 1 + T + \frac{\pi^2}{6}.
\end{align}
\end{proof}

\subsection{Proof of \cref{proposition:prop2} \label{app:prop2}}
\chugg*
\begin{proof} We first show that $\phi^{\textsf{bonf}} \equiv (\phi^{\textsf{bonf}}_t)_{t \geq 1}$ is a level-$\alpha$ sequential test for $H^{\textsf{global}}_{0}$. Then we derive the upper bound on the expected stopping time.

{\noindent {\bf Level-$\alpha$ test}}. Let $\alpha \in (0,1)$. Recall $ \phi^{\textsf{bonf}}_t = \mathbf{1}\left[\max_{i \in [k]} W^{\textsf{ons}}_{i,t} \geq \frac{k}{\alpha}\right]$, where $\forall i \in [k]$
\begin{align}
    W^{\textsf{ons}}_{i,t} = \prod_{j=1}^{t} \left(1 + \lambda_{i,j} \cdot Z_{i,j}\right)
\end{align}
and the wealth processes $W^{\textsf{ons}}_{i}$ for all $i \in [k]$ are test-martingales for their respective nulls $H_{0,i}$. By \nameref{theorem:villes} and a union bound
\begin{align}
     \sup_{P \in H_{i,0}} P(\exists t\geq 1: \phi^{\textsf{bonf}}_t = 1) &= \sup_{P \in H_{i,0}} P\left(\exists t\geq 1: \max_{i \in [k]} W^{\textsf{ons}}_{i,t} \geq \frac{k}{\alpha} = 1\right)\\
     &\leq P\left(\exists t\geq 1: \cup^k_{i=1} W^{\textsf{ons}}_{i,t} \geq k/\alpha\right)\\
     &\leq \sum^k_{i=1} P(\exists t\geq 1: W^{\textsf{ons}}_{i,t} \geq k/\alpha)\\ &\leq \frac{\alpha}{k}k = \alpha.
\end{align}

{\noindent {\bf Expected stopping time}}. 
Let $\alpha \in (0,1)$. By definition, $\tau := \min\{t \geq 1: \phi^{\textsf{bonf}}_{t} = 1\} = \min\{t\geq 1: \max_{i \in [k]} W^{\textsf{ons}}_{i,t} \geq k/\alpha\}$. Since $\tau$ is a non-negative integer valued random variable, for any $P \in H^{\textsf{global}}_{1}$ and any $i \in [k]$
\begin{align}
    \bE_P[\tau] &= \sum^{\infty}_{t=0}P(\tau > t)\\
    &= 1 + \sum^{\infty}_{t=1}P(\tau > t)\\
    &\leq 1 + \sum^{\infty}_{t=1}P\left( \max_{i \in [k]}W^{\textsf{ons}}_{i,t} < k/\alpha\right)\\
    &\leq 1 + \sum^{\infty}_{t=1}P\left( W^{\textsf{ons}}_{i,t} < k/\alpha\right).
\end{align}
Applying the same argument used in the expected stopping time proof of $\phi^{\textsf{ons}}_{i}$ presented in \cref{app:prop1}, we get
\begin{align}
    \bE[\tau] \leq 1 + T + \frac{\pi^2}{6} \quad \text{where} \quad T = \frac{128}{\left(\left|\Delta_i\right|\right)^2}\ln\left(\frac{256}{(\nicefrac{\alpha}{k}) \cdot \left(\left|\Delta_i\right|\right)^2}\right) + \frac{32}{\left(\left|\Delta_i\right|\right)^2}\ln\left(\frac{32}{\left(\left|\Delta_i\right|\right)^2}\right).
\end{align}
Since this holds for any $i \in [k]$
\begin{align}
    \bE[\tau] \leq \min_{i \in [k]} \left\{\frac{128}{\left|\Delta_i\right|^2}\ln\left(\frac{256k}{\alpha \cdot \left|\Delta_i\right|^2}\right) + \frac{32}{\left|\Delta_i\right|^2}\ln\left(\frac{32}{\left|\Delta_i\right|^2}\right)\right\} + \frac{\pi^2}{6}.
\end{align}
Furthermore,
\begin{align}
    &\min_{i \in [k]} \left\{\frac{128}{\left|\Delta_i\right|^2}\ln\left(\frac{256k}{\alpha \cdot \left|\Delta_i\right|^2}\right) + \frac{32}{\left|\Delta_i\right|^2}\ln\left(\frac{32}{\left|\Delta_i\right|^2}\right)\right\} \leq \frac{128}{\Delta_{\textsf{m}}^2}\ln\left(\frac{256k}{\alpha \cdot \Delta_{\textsf{m}}^2}\right) + \frac{32}{\Delta_{\textsf{m}}^2}\ln\left(\frac{32}{\Delta_{\textsf{m}}^2}\right)
\end{align}
concluding the proof.
\end{proof}

\subsection{Proof of \cref{theorem:prod_stopping_time} \label{app:prod_stopping_time}}
\prodstop*
\begin{proof}
We first show that $\phi^{\textsf{prod}} = (\phi^{\textsf{ons}}_{i,t})_{t\geq1}$ is a level-$\alpha$ sequential test for $H^{\textsf{global}}_0$. Then we derive the upper bound on its expected stopping time.

\noindent{{\bf Level-$\alpha$ test}}. Let $\alpha \in (0,1)$. Recall $\phi^{\textsf{prod}}_{t} = \mathbf{1}[W^{\textsf{prod}}_{t} \geq 1/\alpha]$, where
\begin{align}
    W^{\textsf{prod}}_{t} = \prod^k_{i=1} W^{\textsf{ons}}_{i,t}.
\end{align}
For all $i \in [k]$, the wealth processes $W^{\textsf{ons}}_{i}$ are test martingales for their respective nulls $H_{i,0}$. By definition, they are also test-martingales for $ H^{\textsf{global}}_0$. Furthermore, $W^{\textsf{prod}}_0 = 1$, $W^{\textsf{prod}}_0 \geq 0$ for all $t \geq 1$, and for any $P \in H^{\textsf{global}}_0$
\begin{align*}
    \bE_{P}[W^{\textsf{prod}}_{t} &\mid Z_{1,1}, \dots, Z_{1,t-1}, \dots, Z_{k,1}, \dots, Z_{k,t-1}]\\
    &= \bE\left[\prod^k_{i=1} W^{\textsf{ons}}_{i,t} \mid Z_{1,1}, \dots, Z_{1,t-1}, \dots, Z_{k,1}, \dots, Z_{k,t-1} \right]\\
    &= \prod^k_{i=1}\bE\left[W^{\textsf{ons}}_{i,t} \mid Z_{i,1}, \dots, Z_{i,t-1}\right]\\
    &= 
    \prod^k_{i=1}W^{\textsf{ons}}_{i,t-1}.
\end{align*}
Therefore $W^{\textsf{prod}}$ is a test martingale for $H^{\textsf{global}}_{0}$ and so, by \nameref{theorem:villes},
\begin{align}
    \sup_{P \in H^{\textsf{global}}_0} P(\exists t \geq 0: \phi^{\textsf{prod}}_{t} = 1) = \sup_{P \in H^{\textsf{global}}_0} P(\exists t \geq 0: W^{\textsf{prod}}_{t} \geq 1/\alpha) \leq \alpha.
\end{align}
{\bf Expected stopping time}. Let $\alpha \in (0,1)$. By definition, $\tau := \min\{t: \phi^{\textsf{prod}}_{t} = 1\} = \min\{t: W^{\textsf{prod}}_{t} \geq 1/\alpha\}$. Since $\tau$ is a non-negative integer valued random variable, for any $P \in H^{\textsf{global}}_{1}$
\begin{align}
    \bE[\tau] = \sum^{\infty}_{t=1}P(\tau > t) = \sum^{\infty}_{t=1}P(W^{\textsf{prod}}_{t} < 1/\alpha)
    &= \sum^{\infty}_{t=1}P\left(\prod^k_{i=1}W^{\textsf{ons}}_{i,t} < 1/\alpha\right)\\
    &= \sum^{\infty}_{t=1}P\left(\ln\left(\prod^k_{i=1}W^{\textsf{ons}}_{i,t}\right) < \ln(1/\alpha)\right)\\
    &= \sum^{\infty}_{t=1}P\left(\sum^k_{i=1}\ln\left(W^{\textsf{ons}}_{i,t}\right) < \ln(1/\alpha)\right)\\
    &= \sum^{\infty}_{t=1}P(E_t),
\end{align}
where $E_t = \left\{\sum^k_{i=1}\ln\left(W^{\textsf{ons}}_{i,t}\right) < \ln(1/\alpha)\right\}$. Now, defining
\begin{align}
    A_{i,t} = \sum^t_{j=1} Z_{i,j},  \quad V_{i,t} = \sum^t_{j=1} (Z_{i,j})^2,
\end{align}
by \cref{lemma:log_wealth_lb}, 
$\forall i \in [k]$, we have the following guarantee for $W^{\textsf{ons}}_{i,t}$
\begin{align}
    \ln(W^{\textsf{ons}}_{i,t}) &\geq \frac{(A_{i,t})^2}{4(V_{i,t} + |A_{i,t}|)} - 2\ln\left(4t\right), \quad \forall t \geq 1.
\end{align}
Therefore
\begin{align}
    \sum^k_{i=1}\ln(W^{\textsf{ons}}_{i,t}) \geq \sum^k_{i=1}\frac{(A_{i,t})^2}{4(V_{i,t} + |A_{i,t}|)} - 2k\ln(4t), \quad \forall t \geq 1.
\end{align}
Furthermore, $\forall i \in [k]$, $\forall t \geq 1$, $|A_{i,t}| \leq \sum_{j\leq t} |Z_{i,j}| \leq t$ and $V_{i,t} \leq t$. Thus,
\begin{align}
    E_t &\subseteq \left\{ \sum^k_{i=1}\frac{(A_{i,t})^2}{4(V_{i,t} + |A_{i,t}|)} - 2k\ln(4t)< \ln(1/\alpha) \right\}\\
    &\subseteq \left\{\sum^k_{i=1}(A_{i,t})^2 < 8t\ln(1/\alpha) + 16tk\ln(4t) \right\}.
\end{align}
Now consider the function
\begin{align}
    \psi (Z_{1,1}, \dots, Z_{1,t}, \dots, Z_{k,1}, \dots, Z_{k,t}) = \sum^k_{i=1}(A_{i,t})^2.
\end{align}  
Consider the vector $(Z_{1,1}, \ldots, Z_{k,t})$ and another vector $(\bar{Z}_{1,1}, \ldots, \bar{Z}_{k,t})$ which differs from the first vector in exactly one coordinate. Since $\psi$ is symmetric with respect to $(Z_{1,1}, \ldots, Z_{k,t})$, we can assume this coordinate is $Z_{1,1}$ without loss of generality. Then
\begin{align}
    \big| \psi(Z_{1,1}, \ldots, Z_{k,t}) - \psi(\bar Z_{1,1}, \ldots, \bar Z_{k,t})\big| &= \left|(Z_{1,1})^2 - (\bar Z_{1,1})^2 + 2 \left(\sum_{t'=2}^t Z_{1,t'}\right) (Z_{1,1} - \bar{Z}_{1,1})\right| \\
    &\leq 2 + 4 (t - 1) \leq 4t.
\end{align}
So, by \nameref{lemma:mcdiarmid}, we have
\begin{align}
    P \big( \psi - \bE[\psi] \leq -\beta \big) \leq \exp\left(\frac{-2\beta^2}{16kt^3}\right).
\end{align}
Setting the right hand side equal to $1/t^2$ and solving for $\beta$ yields $\beta = t^{1.5}\sqrt{16k\ln(t)}$. Thus, with probability at least $1 - \nicefrac{1}{t^2}$:
\begin{align}
    \sum^k_{i=1} (A_{i,t})^2 \geq \bE[\psi] - t^{1.5} \sqrt{16k \ln(t)} 
    &= \sum_{i=1}^k \bE[(A_{i,t})^2] - t^{1.5} \sqrt{16k \ln(t)} \\
    &\geq  \sum_{i=1}^k  \bE[A_{i,t}]^2 - t^{1.5} \sqrt{16k \ln(t)}\\
    &= t^2\sum_{i=1}^k  \Delta_i^2 - t^{1.5} \sqrt{16k \ln(t)}\\
    &=t^2\Delta_{\textsf{s}} - t^{1.5} \sqrt{16k \ln(t)},
\end{align}
where the last equality holds because, by definition, $\Delta_{\textsf{s}} = \sum^k_{i=1} \Delta^2_i.$ Now, by \cref{lemma:loglemma},
\begin{align}
    \frac{1}{3}t^2\Delta_{\textsf{s}} &\geq 16tk\ln(4t) ~~ \text{for all}~~ t \geq \frac{96k}{\Delta_{\textsf{s}}}\ln\left(\frac{192k}{\Delta_{\textsf{s}}}\right)\\
    \frac{1}{3}t^2 \Delta_{\textsf{s}} &\geq t^{1.5} \sqrt{16k\ln(t)} ~~ \text{for all}~~ t \geq \frac{288k}{\Delta_{\textsf{s}}^2}\ln\left(\frac{144k}{\Delta_{\textsf{s}}^2}\right)
\end{align}
and furthermore $\frac{1}{3}t^2\Delta_{\textsf{s}} \geq 8t\ln(1/\alpha) $ for all $t \geq \frac{24}{\Delta_{\textsf{s}}}\ln(1/\alpha)$. As a result, for 
\begin{align}
    t \geq T_1 := &\frac{96k}{\Delta_{\textsf{s}}}\ln\left(\frac{192k}{\Delta_{\textsf{s}}}\right) + \frac{288k}{\Delta_{\textsf{s}}^2}\ln\left(\frac{144k}{\Delta_{\textsf{s}}^2}\right) + \frac{24}{\Delta_{\textsf{s}}}\ln\left(\frac{1}{\alpha}\right)
\end{align}
we have 
\begin{align}
    t^2\Delta_{\textsf{s}} \geq t^{1.5} \sqrt{16k \ln(t)} + 8tk\ln(1/\alpha) + 16tk\ln(4t).
\end{align}
Therefore, by the law of total probability, for $t \geq T_1$
\begin{align}
    P(E_{t}) &\leq P\left(\sum^k_{i=1}(A_{i,t})^2 < 8tk\ln(1/\alpha) + 16tk\ln(4t)\right) \leq 1/t^2
\end{align}
and so we can conclude
\begin{align}
    \bE[\tau] \leq \sum^\infty_{t=1}P(E_{t}) = \sum^{T_1}_{t=1}P(E_{t}) + \sum^\infty_{t=T_1}P(E_{t})
    &\leq T_1 + \sum^\infty_{t=T_1}\frac{1}{t^2} \leq T_1 + \frac{\pi^2}{6}.
\end{align}
\end{proof}

\subsection{Proof of \cref{theorem:ave_stopping_time} \label{app:ave_stopping_time}}
\avestop*
\begin{proof}
We first show that $\phi^{\textsf{ave}} = (\phi^{\textsf{ave}}_{t})_{t\geq1}$ is a level-$\alpha$ sequential test for $H^{\textsf{global}}_0$. Then we derive the upper bound on its expected stopping time.

{\noindent {\bf Level-$\alpha$ test}}. Let $\alpha \in (0,1)$. Recall $\phi^{\textsf{ave}}_{t} = \mathbf{1}[W^{\textsf{ave}}_{t} \geq 1/\alpha]$, where
\begin{align}
    W^{\textsf{ave}}_{t} = \frac{1}{k}\sum^k_{i=1} W^{\textsf{ons}}_{i,t}.
\end{align}
For all $i \in [k]$, the wealth processes $W^{\textsf{ons}}_{i}$ are test-martingales for their respective nulls $H_{i,0}$. By definition, they are also test-martingales for $ H^{\textsf{global}}_0$. Furthermore $W^{\textsf{ave}}_0 = 1$, $W^{\textsf{ave}}_0 \geq 0$ for all $t \geq 1$, and for any $P \in H^{\textsf{global}}_0$
\begin{align*}
    \bE_{P}[W^{\textsf{ave}}_{t} &\mid Z_{1,1}, \dots, Z_{1,t-1}, \dots, Z_{k,1}, \dots, Z_{k,t-1}]\\
    &= \bE\left[\frac{1}{k}\sum^k_{i=1} W^{\textsf{ons}}_{i,t} \mid Z_{1,1}, \dots, Z_{1,t-1}, \dots, Z_{k,1}, \dots, Z_{k,t-1} \right]\\
    &= \frac{1}{k}\sum^k_{i=1}\bE\left[W^{\textsf{ons}}_{i,t} \mid  Z_{i,1}, \dots, Z_{i,t-1}\right]\\
    &= \frac{1}{k}\sum^k_{i=1}W^{\textsf{ons}}_{i,t-1}.
\end{align*}
Thus $W^{\textsf{ave}}_{t}$ is a test-martingale for $H^{\textsf{global}}_0$. So, by \nameref{theorem:villes},
\begin{align}
    \sup_{P \in H^{\textsf{global}}_0} P(\exists t \geq 0: \phi^{\textsf{ave}}_{t} = 1) = \sup_{P \in H^{\textsf{global}}_0} P(\exists t \geq 0: W^{\textsf{ave}}_{t} \geq 1/\alpha) \leq \alpha.
\end{align}

{\noindent {\bf Expected stopping time}}. Let $\alpha \in (0,1)$. By definition, $\tau := \min\{t: \phi^{\textsf{ave}}_{t} = 1\} = \min\{t: W^{\textsf{ave}}_{t} \geq 1/\alpha\}$. Since $\tau$ is a non-negative integer valued random variable, for any $P \in H^{\textsf{global}}_{1}$
\begin{align}
    \bE[\tau] = \sum^{\infty}_{t=1}P(\tau > t) \leq \sum^{\infty}_{t=1}P(W^{\textsf{ave}}_{t} < 1/\alpha)
    &= \sum^{\infty}_{t=1}P\left(\frac{1}{k}\sum^k_{i=1}W^{\textsf{ons}}_{i,t} < 1/\alpha\right)\\
    &= \sum^{\infty}_{t=1}P\left(\ln\left(\frac{1}{k}\sum^k_{i=1}W^{\textsf{ons}}_{i,t}\right) < \ln(1/\alpha)\right)\\
    &\leq \sum^{\infty}_{t=1}P\left(\frac{1}{k}\sum^k_{i=1}\ln\left(W^{\textsf{ons}}_{i,t}\right) < \ln(1/\alpha)\right)\\
    &= \sum^{\infty}_{t=1}P\left(\sum^k_{i=1}\ln\left(W^{\textsf{ons}}_{i,t}\right) < k\ln(1/\alpha)\right) \\
    &= \sum^{\infty}_{t=1}P(E_t)
\end{align}
where $E_t = \left\{\sum^k_{i=1}\ln\left(W^{\textsf{ons}}_{i,t}\right) < k\ln(1/\alpha)\right\}$ and $\frac{1}{k}\sum^k_{i=1}\ln\left(W^{\textsf{ons}}_{i,t}\right) \leq \frac{1}{k}\sum^k_{i=1}W^{\textsf{ons}}_{i,t}$ due to Jensen's inequality. Following the stopping time proof of $\phi^{\textsf{prod}}$ in \cref{app:prod_stopping_time}, we have
\begin{align}
    \sum^k_{i=1}\ln(W^{\textsf{ons}}_{i,t}) \geq \sum^k_{i=1}\frac{(A_{i,t})^2}{4(V_{i,t} + |A_{i,t}|)} - 2k\ln(4t), \quad \forall t \geq 1.
\end{align}
where $A_{i,t} = \sum^t_{j=1} Z_{i,j}$ and $V_{i,t} = \sum^t_{j=1} (Z_{i,j})^2$. Since $|A_{i,t}| \leq \sum_{j\leq t} |Z_{i,j}| \leq t$ and $V_{i,t} \leq t$ for all $i \in [k]$ and all $ t \geq 1$, we have
\begin{align}
    E_t &\subseteq \left\{ \sum^k_{i=1}\frac{(A_{i,t})^2}{4(V_{i,t} + |A_{i,t}|)} - 2k\ln(4t)< k\ln(1/\alpha) \right\}\\
    &\subseteq \left\{\sum^k_{i=1}(A_{i,t})^2 < 8t(k\ln(1/\alpha) + 2k\ln(4t)) \right\} \\
    &=\left\{\sum^k_{i=1}(A_{i,t})^2 < 8tk\ln(1/\alpha) + 16tk\ln(4t) \right\}.
\end{align}
From the stopping time proof of $\phi^{\textsf{prod}}$ in \cref{app:prod_stopping_time}, we know with probability at least $1 - \nicefrac{1}{t^2}$:
\begin{align}
    \sum^k_{i=1} (A_{i,t})^2 &\geq t^2\Delta_{\textsf{s}} - t^{1.5} \sqrt{16k \ln(t)}.
\end{align}
By \cref{lemma:loglemma},
\begin{align}
    \frac{1}{3}t^2\Delta_{\textsf{s}} &\geq 16tk\ln(4t) ~~ \text{for all}~~ t \geq \frac{96k}{\Delta_{\textsf{s}}}\ln\left(\frac{192k}{\Delta_{\textsf{s}}}\right)\\
    \frac{1}{3}t^2 \Delta_{\textsf{s}} &\geq t^{1.5} \sqrt{16k\ln(t)} ~~ \text{for all}~~ t \geq \frac{288k}{\Delta_{\textsf{s}}^2}\ln\left(\frac{144k}{\Delta_{\textsf{s}}^2}\right)
\end{align}
and furthermore $\frac{1}{3}t^2\Delta_{\textsf{s}} \geq 8tk\ln(1/\alpha) $ for all $t \geq \frac{24k}{\Delta_{\textsf{s}}}\ln(1/\alpha)$. As a result, for 
\begin{align}
    t \geq T_1 := &\frac{96k}{\Delta_{\textsf{s}}}\ln\left(\frac{192k}{\Delta_{\textsf{s}}}\right) + \frac{288k}{\Delta_{\textsf{s}}^2}\ln\left(\frac{144k}{\Delta_{\textsf{s}}^2}\right) + \frac{24k}{\Delta_{\textsf{s}}}\ln(1/\alpha)
\end{align}
we have 
\begin{align}
    t^2\sum_{i=1}^k  \Delta_i^2 \geq t^{1.5} \sqrt{16k \ln(t)} + 8tk\ln(1/\alpha) + 16tk\ln(4t).
\end{align}
Therefore, by the law of total probability, for $t \geq T_1$
\begin{align}
    P(E_{t}) &\leq P\left(\sum^k_{i=1}(A_{i,t})^2 < 8tk\ln(1/\alpha) + 16tk\ln(4t)\right) \leq 1/t^2
\end{align}
and so we can conclude
\begin{align}
    \bE[\tau] \leq \sum^\infty_{t=1}P(E_{t}) = \sum^{T_1}_{t=1}P(E_{t}) + \sum^\infty_{t=T_1}P(E_{t})
    &\leq T_1 + \sum^\infty_{t=T_1}\frac{1}{t^2} \leq T_1 + \frac{\pi^2}{6}.
\end{align}
One can also see that for any $i \in [k]$
\begin{align}
    \bE[\tau] = \sum^{\infty}_{t=1}P(\tau > t) \leq \sum^{\infty}_{t=1}P(W^{\textsf{ave}}_{t} < 1/\alpha)
    &= \sum^{\infty}_{t=1}P\left(\frac{1}{k}\sum^k_{i=1}W^{\textsf{ons}}_{i,t} < 1/\alpha\right)\\
    &\leq \sum^{\infty}_{t=1}P\left(\frac{1}{k}W^{\textsf{ons}}_{i,t} < 1/\alpha\right)\\
    &= \sum^{\infty}_{t=1}P\left(W^{\textsf{ons}}_{i,t} < k/\alpha\right)\\
    &\leq T_2 + \frac{\pi^2}{6}
\end{align}
where 
\begin{align}
    T_2 := \frac{128}{\Delta_{\textsf{m}}^2}\ln\left(\frac{256k}{\alpha \cdot \Delta_{\textsf{m}}^2}\right) + \frac{32}{\Delta_{\textsf{m}}^2}\ln\left(\frac{32}{\Delta_{\textsf{m}}^2}\right)
\end{align}
due to the proof of
\cref{proposition:prop2} in \cref{app:prop2}. Thus, for any $P \in H^{\textsf{global}}_{1}$, the expected stopping time of $\phi^{\textsf{ave}}$ obeys 
\begin{align}
    \bE[\tau] \leq \min\{T_1, T_2\}.
\end{align}
\end{proof}

\subsection{Proof of \cref{theorem:balance_stopping_time} \label{app:balance_stopping_time}}
\balancestop*
\begin{proof}
We first show that $\phi^{\textsf{balance}} = (\phi^{\textsf{balance}}_{t})_{t\geq1}$ is a level-$\alpha$ sequential test for $H^{\textsf{global}}_0$. Then we derive the upper bound on its expected stopping time.

{\bf Level-$\alpha$ test}. Let $\alpha \in (0,1)$. Recall $\phi^{\textsf{balance}}_{t} = \mathbf{1}[W^{\textsf{balance}}_{t} \geq 1/\alpha]$, where
\begin{align}
    W^{\textsf{balance}}_{t} = \frac{1}{2}W^{\textsf{ave}}_{t} + \frac{1}{2} W^{\textsf{prod}}_{t}.
\end{align}
The wealth processes $W^{\textsf{ave}}$ and $W^{\textsf{prod}}$ are both test-martingales for $ H^{\textsf{global}}_{0}$. Thus, $W^{\textsf{balance}}_{0} = 1$, $W^{\textsf{balance}}_{t} \geq 0$ for all $t \geq 1$, and for any $P \in H^{\textsf{global}}_{0}$
\begin{align}
    \bE[&W^{\textsf{balance}}_{t} \mid Z_{1,1}, \dots, Z_{1,t}, \dots, Z_{k,1}, \dots, Z_{k,t}] \\
    &= \bE\left[\frac{1}{2}W^{\textsf{ave}}_{t} + \frac{1}{2} W^{\textsf{prod}}_{t}\mid Z_{1,1}, \dots, Z_{1,t}, \dots, Z_{k,1}, \dots, Z_{k,t}\right]\\
    &=\frac{1}{2}\bE\left[W^{\textsf{ave}}_{t}\mid Z_{1,1}, \dots, Z_{1,t}, \dots, Z_{k,1}, \dots, Z_{k,t} \right]+ \frac{1}{2} \bE\left[W^{\textsf{prod}}_{t}\mid Z_{1,1}, \dots, Z_{1,t}, \dots, Z_{k,1}, \dots, Z_{k,t}\right] \\
    &= \frac{1}{2}W^{\textsf{ave}}_{t-1} + \frac{1}{2} W^{\textsf{prod}}_{t-1}.
\end{align}
As a result, $W^{\textsf{balance}}_{t}$ is a test-martingale for $ H^{\textsf{global}}_{0}$. Thus, by \nameref{theorem:villes},
\begin{align}
    \sup_{P \in H^{\textsf{global}}_0} P(\exists t \geq 0: \phi^{\textsf{balance}}_{t} = 1) = \sup_{P \in H^{\textsf{global}}_0} P(\exists t \geq 0: W^{\textsf{balance}}_{t} \geq 1/\alpha) \leq \alpha.
\end{align}

{\bf Expected stopping time}. Let $\alpha \in (0,1)$. By definition, $\tau := \min\{t: \phi^{\textsf{balance}}_{t} = 1\} = \min\{t: W^{\textsf{balance}}_{t} \geq 1/\alpha\}$. since $\tau$ is a non-negative integer valued random variable, for any $P \in H^{\textsf{global}}_{1}$
\begin{align}
    \bE[\tau] = \sum^{\infty}_{t=1}P(\tau > t) \leq \sum^{\infty}_{t=1}P(W^{\textsf{balance}}_{t} < 1/\alpha)
    &= \sum^{\infty}_{t=1}P\left(\frac{1}{2}W^{\textsf{ave}}_{t} + \frac{1}{2} W^{\textsf{prod}}_{t} < 1/\alpha\right)\\
    &< \sum^{\infty}_{t=1}P\left(\frac{1}{2}W^{\textsf{prod}}_{t}  < 1/\alpha\right)\\
    &= \sum^{\infty}_{t=1}P\left(W^{\textsf{prod}}_{t} < 2/\alpha\right).
\end{align}
By the proof of \cref{theorem:prod_stopping_time} provided in \cref{app:prod_stopping_time}, we know 
\begin{align}
    \sum^{\infty}_{t=1}P\left(W^{\textsf{prod}}_{t} < 2/\alpha\right) \leq  T_1 + \frac{\pi^2}{6}
\end{align}
where 
\begin{align}
    T_1 := &\frac{96k}{  \Delta_{\textsf{s}}}\ln\left(\frac{192k}{\Delta_{\textsf{s}}}\right) + \frac{288k}{\Delta_{\textsf{s}}^2}\ln\left(\frac{144k}{\Delta_{\textsf{s}}^2}\right) + \frac{24}{\Delta_{\textsf{s}}}\ln\left(\frac{2}{\alpha}\right).
\end{align}

We can also bound the expected stopping time in the following manner. 
\begin{align}
    \bE[\tau] &< \sum^{\infty}_{t=1}P\left(\frac{1}{2}W^{\textsf{ave}}_{t}  < 1/\alpha\right) = \sum^{\infty}_{t=1}P\left(W^{\textsf{ave}}_{t} < 2/\alpha\right).
\end{align}
By the expected stopping time proof of $\phi^{\textsf{ave}}$ provided in \cref{app:ave_stopping_time}, we know 
\begin{align}
    \sum^{\infty}_{t=1}P\left(W^{\textsf{ave}}_{t} < 2/\alpha\right) \leq  T_2 + \frac{\pi^2}{6}
\end{align}
where
\begin{align}
    T_2 := \frac{128}{\Delta_{\textsf{m}}^2}\ln\left(\frac{256\cdot(2k)}{\alpha \cdot \Delta_{\textsf{m}}^2}\right) + \frac{32}{\Delta_{\textsf{m}}^2}\ln\left(\frac{32}{\Delta_{\textsf{m}}^2}\right).
\end{align}
Thus, for any $P \in H^{\textsf{global}}_{1}$, the expected stopping time of $\phi^{\textsf{balance}}$ obeys 
\begin{align}
    \bE[\tau] \leq \min\{T_1, T_2\}.
\end{align}
\end{proof}

\section{Useful Lemmas and Inequalities \label{app:useful_lemmas}}
In this section, we present  results that are used throughout the proofs in \cref{app:proofs}. For cited results, their proofs can be found in the referenced works; otherwise, proofs are provided here.

\begin{lemma}[Hoeffding's Inequality \citep{hoeffding1963probability}]
\label{lemma:hoeffding}
Let $X_1, \dots, X_t$ be independent random variables such that $X_j \in [a_j, b_j]$ almost surely. Then, for all $\beta > 0$,
\begin{align}
P\left(\left|\sum_{j=1}^t X_j - \bE\left[\sum_{j=1}^t X_j\right]\right| \geq \beta \right) \leq 2\exp\left( - \frac{2\beta^2}{\sum^t_{j=1}(b_j - a_j)^2} \right).
\end{align}
\end{lemma}

\begin{lemma}[Concentration of bounded random variables \citep{chugg2023auditing}]
\label{lemma:concentration}
Let $X_1, \dots, X_t$ be independent random variables such that $X_j \in [-1, 1]$ almost surely and $\bE[X_j] = \mu$. with probability at least $1 - \nicefrac{1}{t^2}$
\begin{align}
    \left|\sum^t_{j=1}X_j\right| \geq t\cdot |\mu| - \sqrt{4t\ln(2t)}.
\end{align}
\end{lemma}
\begin{proof}
By \nameref{lemma:hoeffding}
\begin{align}
    P\left(\left|\sum^t_{j=1}X_j - \bE\left[\sum^t_{j=1}X_j\right]\right| \geq \beta \right) \leq 2\exp\left(-\beta^2/2t \right).
\end{align}
Setting the right hand side equal to $1/t^2$ and solving for $\beta$ yields $\beta = \sqrt{2t\ln(2t^2)}$. Thus, with probability at least $1 - \nicefrac{1}{t^2}$,
\begin{align}
    \left|\left|\sum^t_{j=1}X_j\right| - \left|\bE\left[\sum^t_{j=1}X_j\right]\right|\right| \leq \left|\sum^t_{j=1}X_j - \bE\left[\sum^t_{j=1}X_j\right]\right| \leq \sqrt{2t\ln(2t^2)} 
    &\leq \sqrt{4t\ln(2t)} 
\end{align}
which implies with probability at least $1 - \nicefrac{1}{t^2}$
\begin{align}
    \left|\sum^t_{j=1}X_j\right| &\geq \left|\bE\left[\sum^t_{j=1}X_j\right]\right| - \sqrt{4t\ln(2t)} = t\cdot|\mu| - \sqrt{4t\ln(2t)}.
\end{align}
\end{proof}

\begin{lemma}[McDiarmid's Inequality \citep{mcdiarmid1989method}]
\label{lemma:mcdiarmid}
Let $X_1, \dots, X_n$ be independent random variables taking values in some set $\mathcal{X}$. Let $f: \mathcal{X}^n \to \mathbb{R}$ be a function that satisfies the following bounded differences property: for every $i \in [n]$ and for all possible values $x_1, \dots, x_n, x_i' \in \mathcal{X}$,
\[
|f(x_1, \dots, x_n) - f(x_1, \dots, x_{i-1}, x_i', x_{i+1}, \dots, x_n)| \leq c_i,
\]
where $c_i$ are constants. Then, for any $\epsilon > 0$, the following inequalities holds:
\begin{align} 
P\left( f(X_1, \dots, X_n) - \mathbb{E}[f(X_1, \dots, X_n)] \geq \epsilon \right) &\leq \exp\left(-\frac{2\epsilon^2}{\sum_{i=1}^n c_i^2}\right)\\
P\left( f(X_1, \dots, X_n) - \mathbb{E}[f(X_1, \dots, X_n)] \leq -\epsilon \right) &\leq \exp\left(-\frac{2\epsilon^2}{\sum_{i=1}^n c_i^2}\right).
\end{align}
\end{lemma}

\begin{theorem}[Theorem 20 \cite{cutkosky2018black}] \label{theorem:opt_wealth_lb}
Let $\|\cdot\|$ be \emph{any} norm on $\mathbb{R}^{d}$ and
$\|\cdot\|_{*}$ be its dual norm. Define $\calK = \{v \in \bR^d: \|v\| \leq \nicefrac{1}{2}\}$. Fix any vector $u\in\mathbb{R}^d$ satisfying $\|u\|=1$ and let $(g_t)_{t\ge 1} \subset \bR^d$ be any sequence of vectors satisfying $\|g_t\|_{*}\leq 1$ for all $t\geq 1$. Then,
\begin{align}
    \max_{\lambda \in \calK} ~\sum^t_{j=1} \ln(1 + \langle \lambda , g_j \rangle) \geq \frac{1}{4}\frac{\langle \sum^t_{j=1} g_j, u \rangle^2}{\sum^t_{j=1} \langle g_j, u \rangle^2 + \left|\left\langle\sum^t_{j=1} g_j, u\right\rangle\right|}.
\end{align}
\end{theorem}

\begin{theorem}[Lemma 17 \citep{cutkosky2018black}]
\label{theorem:ons_regret}
Let $\|\cdot\|$ be a norm on $\mathbb{R}^{d}$ and
$\|\cdot\|_{*}$ be its dual norm. Define $\calK = \{v \in \bR^d: \|v\| \leq \nicefrac{1}{2}\}$ and let $(g_t)_{t\ge 1}$ be any sequence of vectors satisfying $\|g_t\|_{*}\leq 1$ for all $t \geq 1$. Then, for $\beta = \frac{2- \ln 3}{2}$, the sequence $(\lambda_t)_{t\geq1} \subset \calK$ generated by ONS (\cref{alg:ons}) with input stream $(g_t)_{t\ge 1}$ satisfies 
\begin{align}
    \sum^t_{j=1} -\ln(1 + \langle \lambda_j, g_j\rangle) - \min_{\lambda \in \calK} ~\sum^t_{j=1} -\ln(1 + \langle \lambda, g_j \rangle) &\leq d\left(\frac{\beta}{8} + \frac{2}{\beta}\ln\left(4\sum^t_{j=1} ||g_j||_{*}^2 + 1\right)\right).
\end{align}
By substituting $\beta = \frac{2- \ln 3}{2}$ and using the bounds $\beta \leq \nicefrac{8}{17}$ and $\nicefrac{2}{\beta} \leq 4.5$, one retrieves the following precise bound stated in Lemma 17 of \citep{cutkosky2018black}.
\begin{align}
    \sum^t_{j=1} -\ln(1 + \langle \lambda_j, g_j\rangle) - \min_{\lambda \in \calK} ~\sum^t_{j=1} -\ln(1 + \langle \lambda, g_j\rangle) &\leq d\left(\frac{1}{17} + 4.5\ln\left(4\sum^t_{j=1} ||g_j||_{*}^2 + 1\right)\right).
\end{align}
{\bf NOTE}: To prove the result above, \citet{cutkosky2018black} rely on another result from their work, Theorem 11 (page 18). The proof of Theorem 11 contains a minor typographical error: in the second equation on page 19\footnote{We refer to the version published in the 31st Annual Conference on Learning Theory.}, the final term in the inequality has a factor of $\nicefrac{2}{\beta}$, which should be $\nicefrac{1}{2\beta}$. For $\beta = \frac{2-\ln(3)}{2}$, the quantity $\frac{2}{\beta} \approx 4.44$, which leads to the $4.5$ factor in the bound. The correct value should be any value greater $\nicefrac{1}{2\beta} \approx 1.11$. The correct upper bound, with $\beta = \frac{2- \ln 3}{2}$ is
\begin{align}
    \sum^t_{j=1} -\ln(1 + \langle \lambda_j, g_j\rangle) - \min_{\lambda \in \calK} ~\sum^t_{j=1} -\ln(1 + &\langle \lambda, g_j\rangle) \leq F\left(d, \sum^t_{j=1} ||g_j||_{*}^2\right)
\end{align}
where 
\begin{align}
     F\left(d, \sum^t_{j=1} ||g_j||_{*}^2\right) = \leq d\left(\frac{2- \ln 3}{16} + \frac{1}{2-\ln 3}\ln\left(4\sum^t_{j=1} ||g_j||_{*}^2 + 1\right)\right)
\end{align}
which is less than $2d\ln(4t)$ for all $t \geq 1$.
\end{theorem}

\begin{lemma}[Log wealth lower bound]
\label{lemma:log_wealth_lb}
Consider $k$ parallel data streams $(Z_{i,t})_{t \geq 1}, \dots, (Z_{k,t})_{t \geq 1}$. For any data stream $(Z_{i,t})_{t \geq 1}$ and its wealth process $W^{\textsf{ons}}_{i,t} = \prod^t_{j=1}(1 + \lambda_{i,j}Z_{i,j})$ where the $(\lambda_{i,t})_{t\geq 1}$ are determined by running ONS (\cref{alg:ons}) on the stream $(Z_{i,t})_{t \geq 1}$, by \cref{theorem:opt_wealth_lb,theorem:ons_regret}, we have
\begin{align}
    \ln(W^{\textsf{ons}}_{i,t}) \geq \frac{1}{4}\frac{ (\sum^t_{j=1} Z_{i,j})^2}{\sum^t_{j=1} (Z_{i,j})^2 + \left|\sum^t_{j=1} Z_{i,j} \right|} - 2\ln(4t).
\end{align}
Let $\|\cdot\|$ be a norm on $\mathbb{R}^{k}$ and
$\|\cdot\|_{*}$ be its dual norm and define $\calK = \{v \in \bR^k: \|v\| \leq \nicefrac{1}{2}\}$. Consider the stream of multivariate outcomes, $(\vec{Z}_t)_{t\geq 1}$ where $\vec{Z}_t = (Z_{1,t}, \dots, Z_{k,t})$ and its corresponding wealth process $W^{\textsf{mv-ons}}_{t} = \prod^t_{j=1}(1 + \langle \vec{\lambda}_{j},~ \vec{Z}_{j}\rangle)$ where the $(\vec{\lambda}_{j})_{t\geq 1}$ are determined by running ONS (\cref{alg:ons}) on the stream $(\vec{Z}_t)_{t\geq 1}$. Then, by \cref{theorem:opt_wealth_lb,theorem:ons_regret}, for any vector $u \in \bR^k$ satisfying $\|u\| = 1$ we have
\begin{align}
    \ln(W^{\textsf{mv-ons}}_{t}) \geq \frac{1}{4}\frac{\langle \sum^t_{j=1} u, \vec{Z}_j \rangle^2}{\sum^t_{j=1} \langle u, \vec{Z}_j \rangle^2 + \left|\left\langle\sum^t_{j=1} u, \vec{Z}_j \right\rangle\right|} - 2k\ln(4t).
\end{align}
\end{lemma}

\begin{proof} We prove the result for the stream of multivariate outcomes $(\vec{Z}_t)_{t\geq 1}$. By \cref{theorem:ons_regret}, the sequence $(\vec{\lambda}_t)_{t\geq1} \subset \calK$ generated by the ONS algorithm (\cref{alg:ons}) with input stream $(\vec{Z}_t)_{t\ge 1}$ satisfies 
\begin{align}
\sum^t_{j=1} -\ln(1 + \langle \vec{\lambda}_j, \vec{Z}_j\rangle) - \min_{\lambda \in \calK} ~\sum^t_{j=1} -\ln(1 + \langle \vec{\lambda}, \vec{Z}_j\rangle) &\leq 2k\ln(4t).
\end{align}
Since $
    \min_{\vec{\lambda} \in \calK} ~\sum^t_{j=1} -\ln(1 + \langle \vec{\lambda}, \vec{Z}_j\rangle) = -\max_{\vec{\lambda} \in \calK} \sum^t_{j=1} \ln(1 + \langle \vec{\lambda}, \vec{Z}_j\rangle)$
this implies 
\begin{align}
\sum^t_{j=1} -&\ln(1 + \langle \vec{\lambda}_j, \vec{Z}_j\rangle) - \left(-\max_{\vec{\lambda} \in \calK} \sum^t_{j=1} \ln(1 + \langle \vec{\lambda}, \vec{Z}_j\rangle)\right)\\
&= \max_{\vec{\lambda} \in \calK} \sum^t_{j=1} \ln(1 + \langle \vec{\lambda}, \vec{Z}_j\rangle) - \sum^t_{j=1} \ln(1 + \langle \vec{\lambda}_j, \vec{Z}_j\rangle) \leq 2k\ln(4t),
\end{align}
which further implies
\begin{align}
    \sum^t_{j=1} \ln(1 + \langle \vec{\lambda}_j, \vec{Z}_j\rangle) \geq \max_{\vec{\lambda} \in \calK} \sum^t_{j=1} \ln(1 + \langle \vec{\lambda}, \vec{Z}_j\rangle) - 2k\ln(4t).
\end{align}
By \cref{theorem:opt_wealth_lb}, for any vector $u \in \bR^k$ satisfying $\|u\| = 1$ we have 
\begin{align}
    \sum^t_{j=1} \ln(1 + \langle \vec{\lambda}_j, \vec{Z}_j\rangle) \geq \frac{1}{4}\frac{\langle \sum^t_{j=1} u, \vec{Z}_j \rangle^2}{\sum^t_{j=1} \langle u, \vec{Z}_j \rangle^2 + \left|\left\langle\sum^t_{j=1} u, \vec{Z}_j \right\rangle\right|} - 2k\ln(4t).
\end{align}
Note, $\sum^t_{j=1} \ln(1 + \langle \vec{\lambda}, \vec{Z}_j\rangle)$ is precisely $\ln(W^{\textsf{mv-ons}}_t)$, thus we have proven the result.

For a stream of univariate outcomes $(Z_{i,t})_{t\geq1}$ and its corresponding wealth process $W^{\textsf{ons}}_{i,t}$, can apply an identical argument to get the following guarantee:
\begin{align}
    \ln(W^{\textsf{ons}}_{i,t}) \geq \frac{1}{4}\frac{\langle \sum^t_{j=1} u, \vec{Z}_j \rangle^2}{\sum^t_{j=1} \langle u, \vec{Z}_j \rangle^2 + \left|\left\langle\sum^t_{j=1} u, \vec{Z}_j \right\rangle\right|} - 2k\ln(4t).
\end{align}
Now note, here $Z_{i,t} \in [-1,1] \subset \bR$, thus $k = 1$. Furthermore, there is only two vectors $u \in \bR$ satisfying $\|u\| = |u| = 1$, which are $u = 1$ and $u = -1$. Plugging either of these $u$ into the equation above proves the univariate version of the result. 
\end{proof}

\begin{lemma}
\label{lemma:loglemma}
Fix \(A>0\) and \(0<B<\frac{A}{e}\). Then,
\begin{align}
     t \geq \frac{2}{B}\,\ln\left(\frac{A}{B}\right) \implies \frac{\ln(At)}{t} \leq B
\end{align}
\end{lemma}

\begin{proof}
Set
\begin{align}
    L := \ln\left(\frac{A}{B}\right) \quad \text{and} \quad y := Bt.
\end{align}
Since \(At = \frac{A}{B}\,y = e^{L}\,y\), we have
\begin{align}
    \ln(At) = L + \ln(y).
\end{align}
Therefore
\begin{align}
    \frac{\ln(At)}{t} = \frac{L + \ln(y)}{y/B} = B \frac{L + \ln(y)}{y}.
\end{align}
To prove the desired inequality it suffices to show then when $y \geq 2L$, we have
\begin{equation}
   L + \ln(y) \leq y.
\end{equation}
Since $y \geq 2L$ we have $L \leq y/2$. Furthermore, we have $y > 2$ because $L > 1$, and for $y > 2$, we have $\ln(y) < y/2$. Thus 
\begin{align}
    L + \ln y \leq \frac{y}{2} + \frac{y}{2} \leq y.
    \end{align}
As a result, for $A > 0$ and $0< B < \nicefrac{A}{e}$, we have $\frac{\ln(At)}{t} \leq B$ whenever  \(t \ge \tfrac{2}{B}\ln(\tfrac{A}{B})\).
\end{proof}

\section{Algorithms \label{app:algorithms}}
\begin{algorithm}[h!]
\caption{\textsc{Online Newton Step} (ONS)}
\label{alg:ons}
\textbf{Input:} Stream $(g_t)_{t\geq 1} \subseteq [-1,1]^d$
\begin{algorithmic}[1]            
    \State \textbf{Initialize} $\lambda_1 = \vec{0} \in \bR^d$, $H_{0} = I_d$                  
    \For{$j = 1, \dots$}
        \State $H_j \gets H_{j-1} + \frac{1}{(1 + \langle \lambda_j,  g_j\rangle)^2}g_j g^T_j$  
        \State $\lambda_{j+1} \gets \text{proj}^{H_j}_{\|v\|_1 \leq \nicefrac{1}{2}}\left(\lambda_j - \frac{2}{2-\ln(3)}H^{-1}_{j}g_j\right)$, where
        \begin{equation*}
            \text{proj}^{H_j}_{\|v\|_1 \leq \nicefrac{1}{2}}(y) = \underset{\|v\|_1 \leq \nicefrac{1}{2}}{\text{argmin}}~\langle H_j(v-y), v-y \rangle
        \end{equation*}
    \EndFor
  \end{algorithmic}
\end{algorithm}
\paragraph{Remark}: When the stream $(g_t)_{t\geq 1}$ consists of $g_{t} \in [-1,1]$, the update rule for $\lambda_{t+1}$ simplifies to \citep{shekhar2023nonparametric, chugg2023auditing, waudby2024estimating}:
\begin{align}
    &\lambda_{j+1} = \Pi_{[\frac{-1}{2}, \frac{1}{2}]}\left(\lambda_{j} - \frac{2}{2-\ln 3}\frac{\nu_{j}}{1 + \sum^{j}_{i=1}(\nu_i)^2}\right) ~~
     \text{where} ~~ \quad \nu_i = \frac{-g_i}{1+\lambda_i\cdot g_i},
\end{align}
where $\Pi_{[\frac{-1}{2}, \frac{1}{2}]}$ is the projection onto the $[-1/2, 1/2]$ interval.

\newpage 
\section{Additional Experimental Details \& Results \label{app:add_exp_results}}
\subsection{Synthetic}
\subsubsection{Results for $k = 25$}
\begin{figure*}[h!]
    \centering
    \resizebox{1\textwidth}{!}{
    \hspace{-0.5em}
    \begin{subfigure}[t]{0.19\linewidth}
        \centering
        \includegraphics[width=\linewidth]{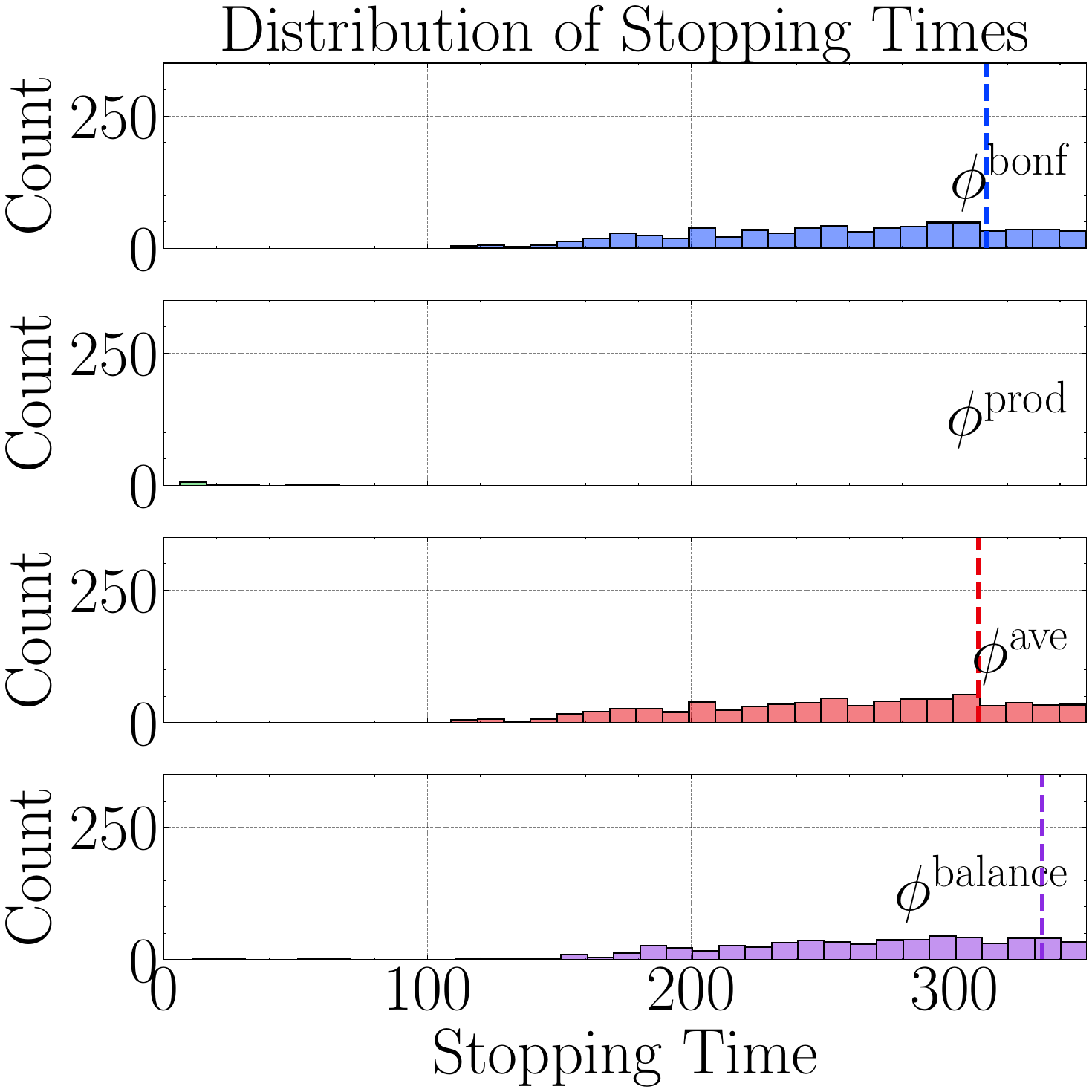}
    \end{subfigure}
    \begin{subfigure}[t]{0.19\linewidth}
        \centering
        \includegraphics[width=\linewidth]{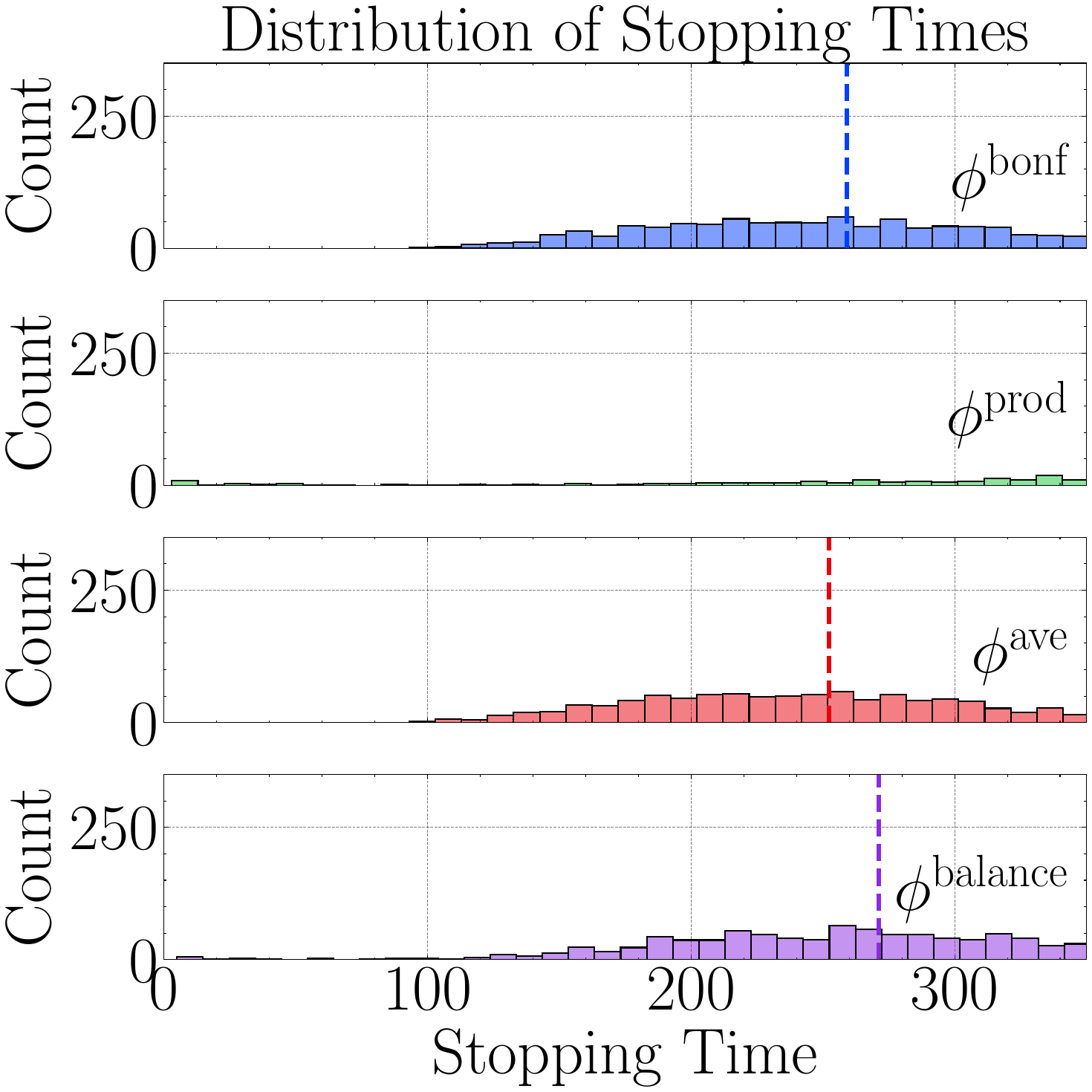}
    \end{subfigure}
    \begin{subfigure}[t]{0.19\linewidth}
        \centering
        \includegraphics[width=\linewidth]{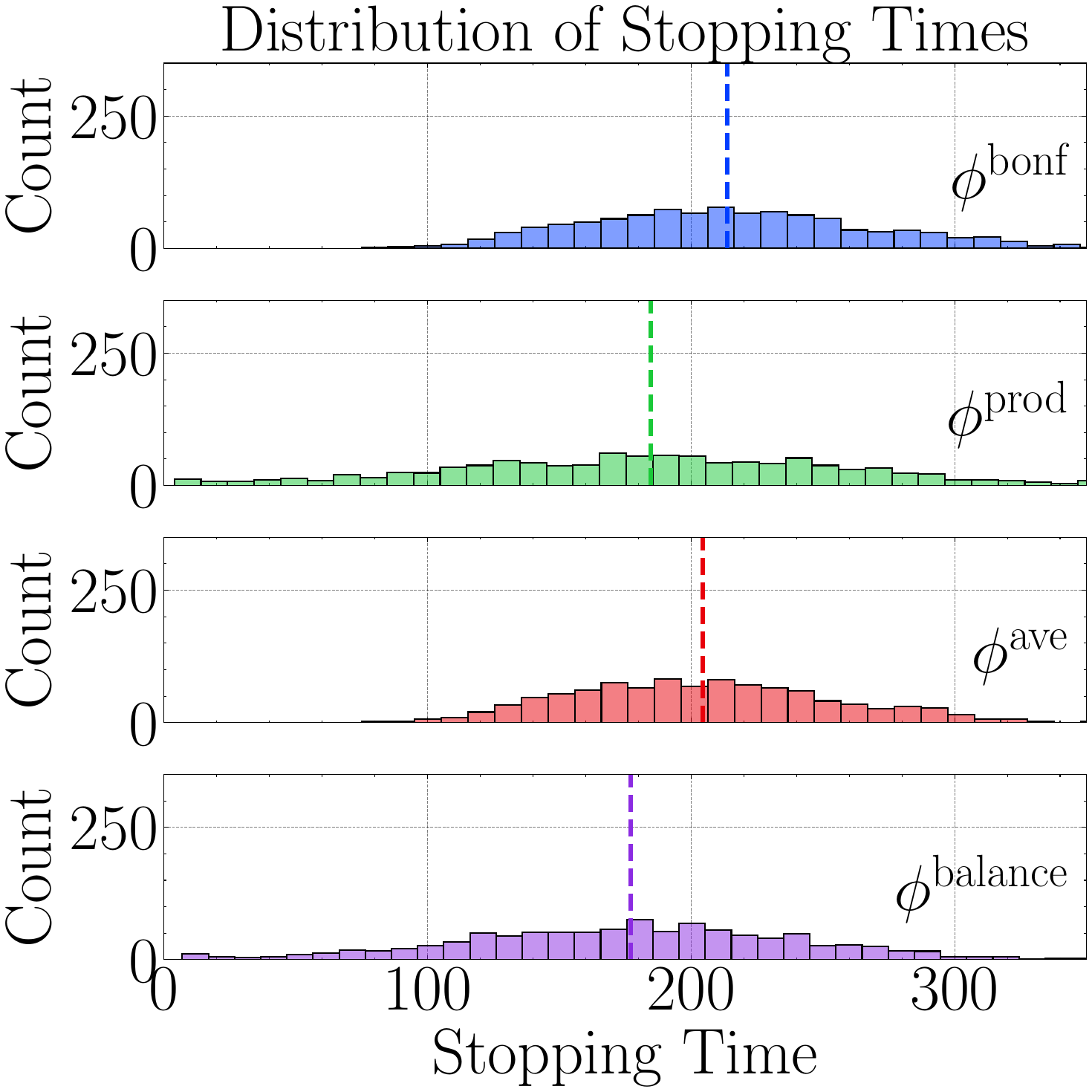}
    \end{subfigure}
    \begin{subfigure}[t]{0.19\linewidth}
        \centering
        \includegraphics[width=\linewidth]{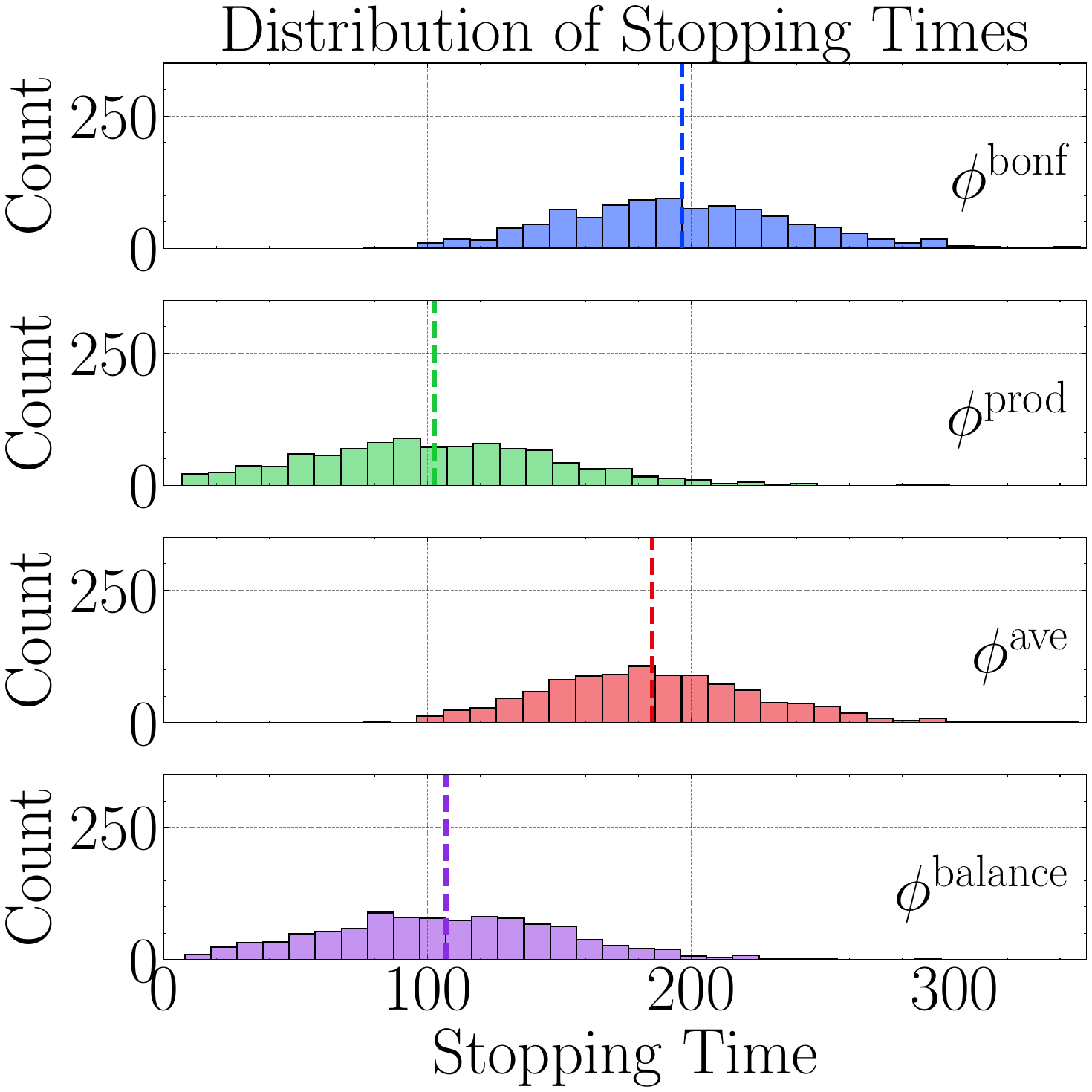}
    \end{subfigure}
    \begin{subfigure}[t]{0.19\linewidth}
        \centering
        \includegraphics[width=\linewidth]{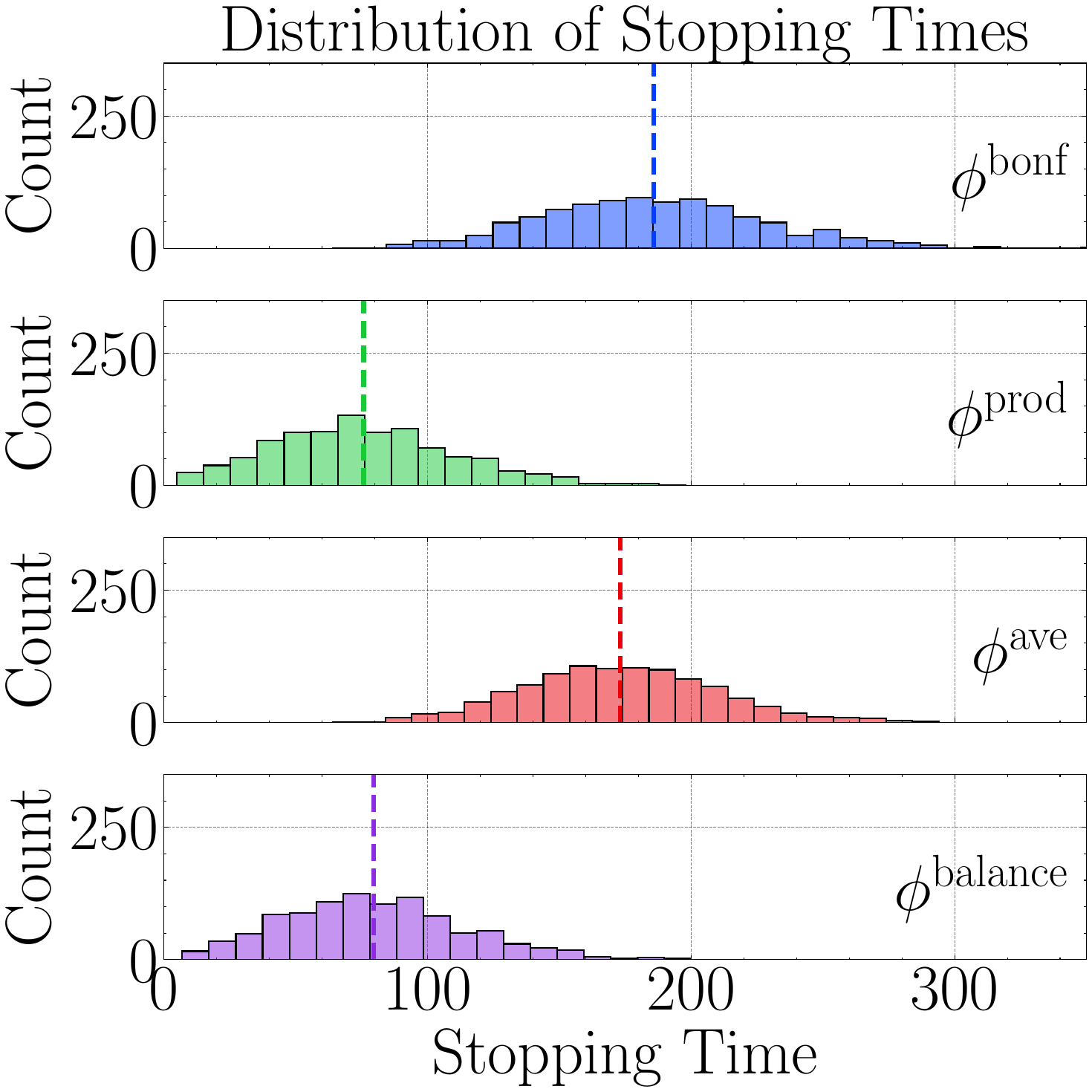}
    \end{subfigure}
    \begin{subfigure}[t]{0.19\linewidth}
        \centering
        \includegraphics[width=\linewidth]{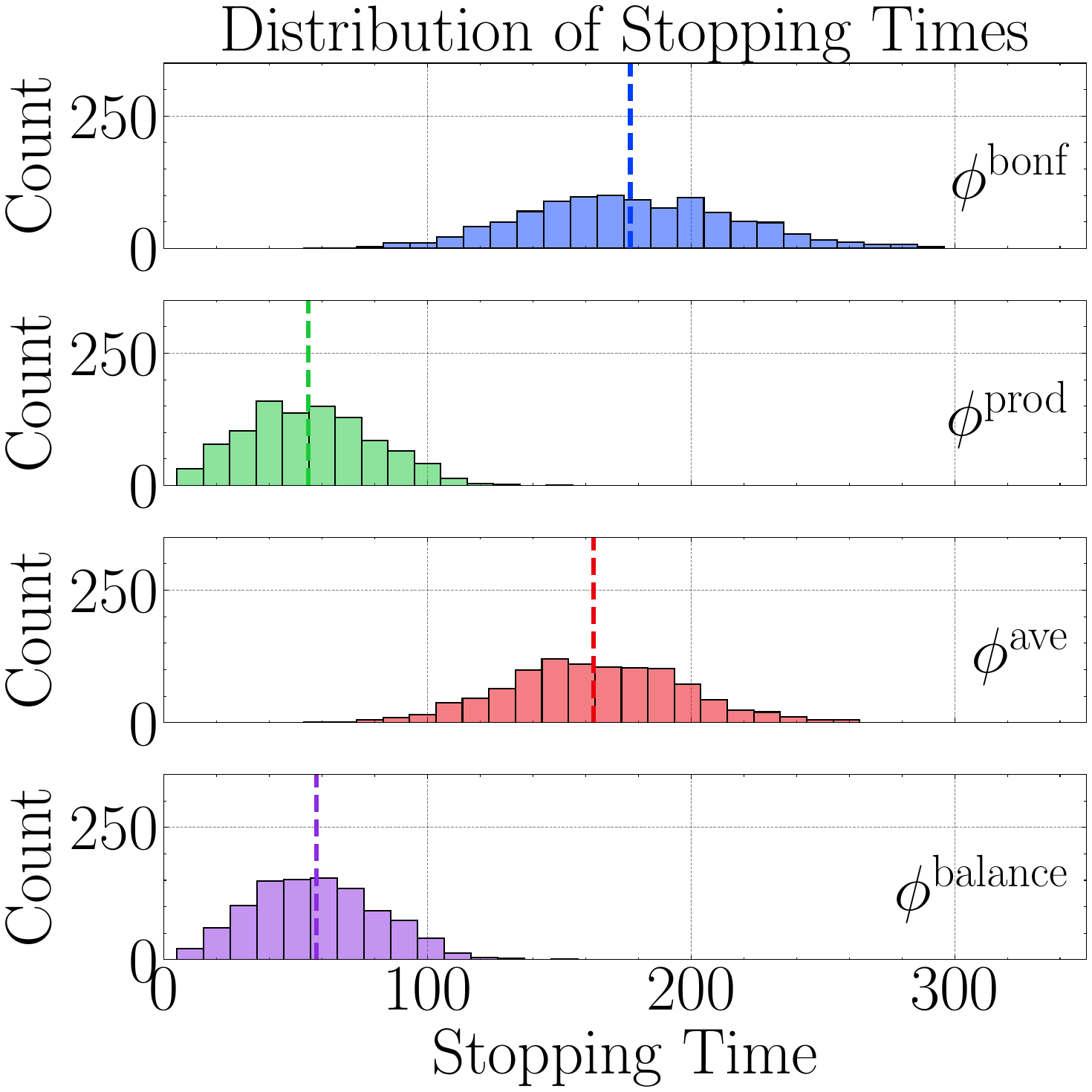}
    \end{subfigure}
    }

    \centering
    \resizebox{1\textwidth}{!}{%
    \begin{subfigure}[t]{0.19\linewidth}
        \includegraphics[width=\linewidth]{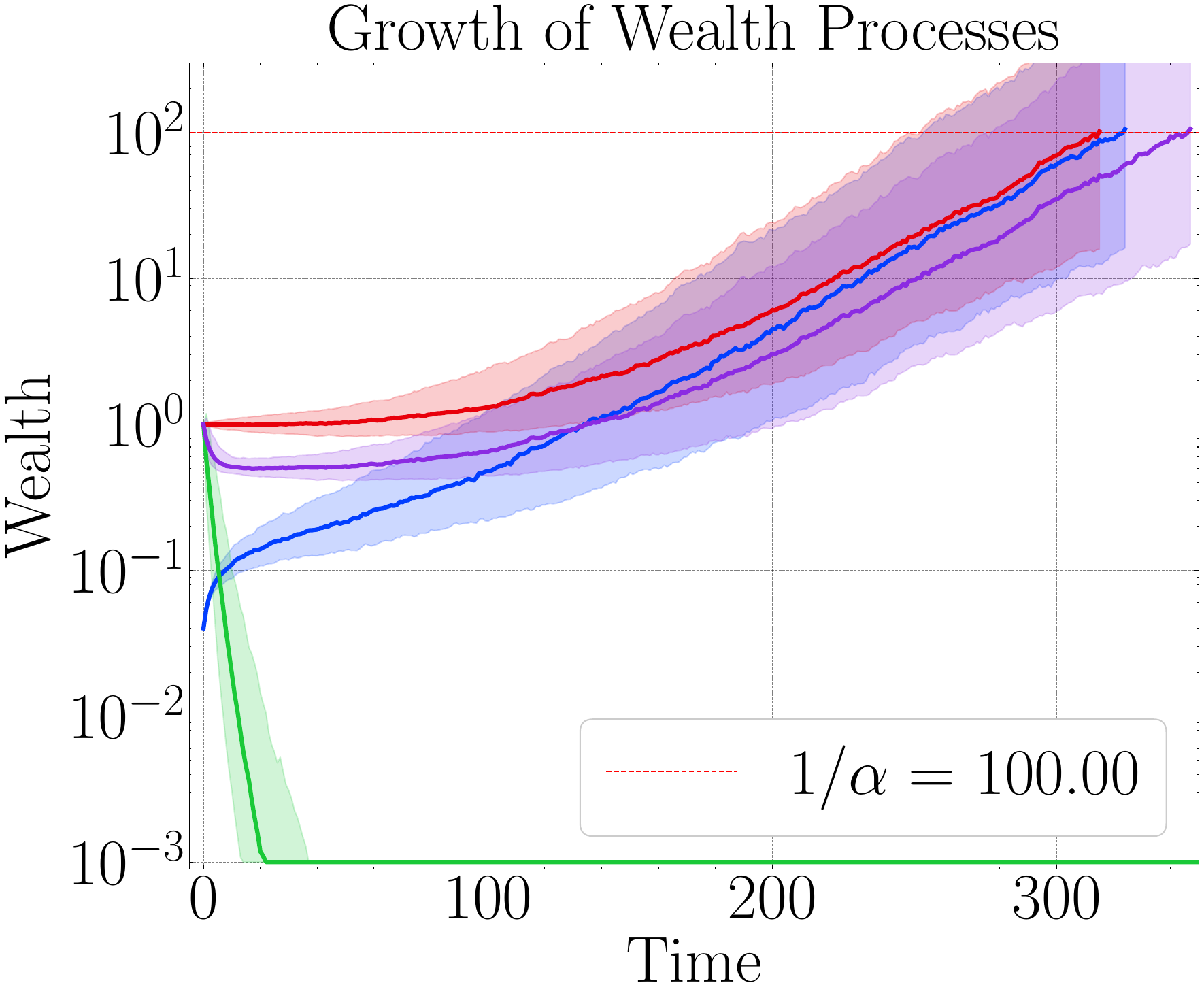}
        \caption{$\lfloor \frac{k_1}{k}\rfloor = 0.05$}
    \end{subfigure}
    \begin{subfigure}[t]{0.19\linewidth}
        \centering
        \includegraphics[width=\linewidth]{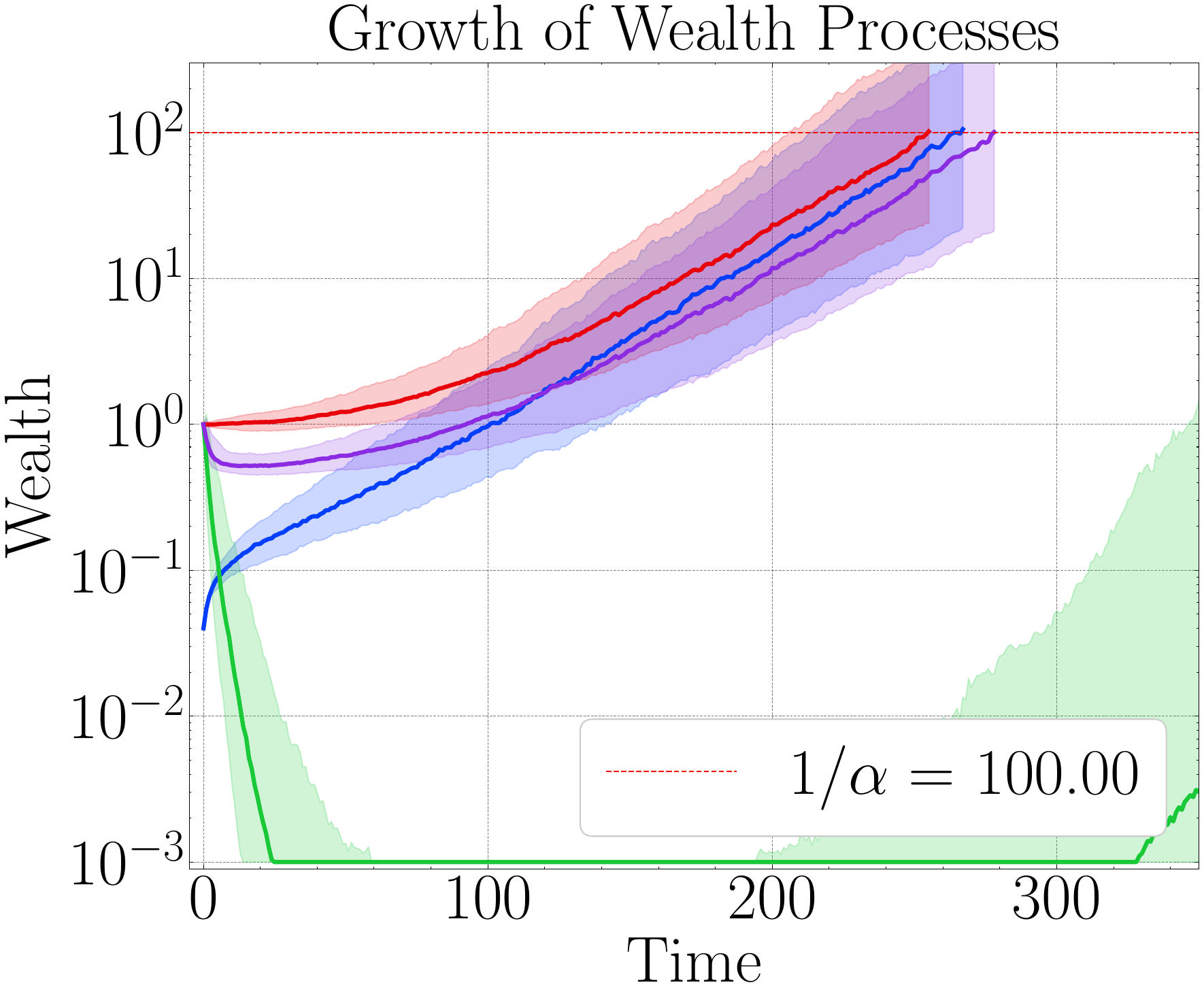}
        \caption{$\lfloor \frac{k_1}{k}\rfloor = 0.15$}
    \end{subfigure}
    \begin{subfigure}[t]{0.19\linewidth}
        \includegraphics[width=\linewidth]{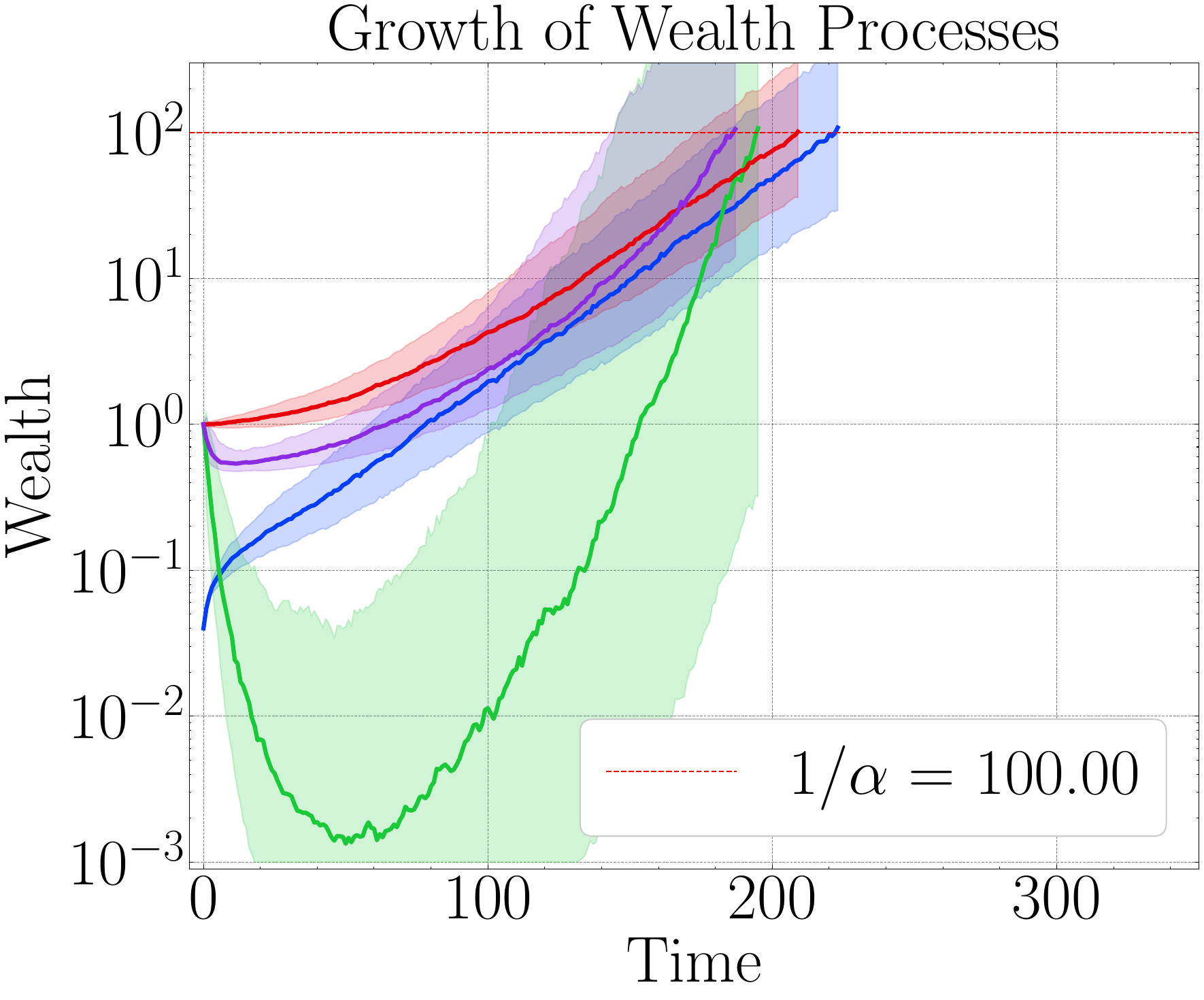}
        \caption{$\lfloor \frac{k_1}{k}\rfloor = 0.30$}
    \end{subfigure}
    \begin{subfigure}[t]{0.19\linewidth}
        \includegraphics[width=\linewidth]{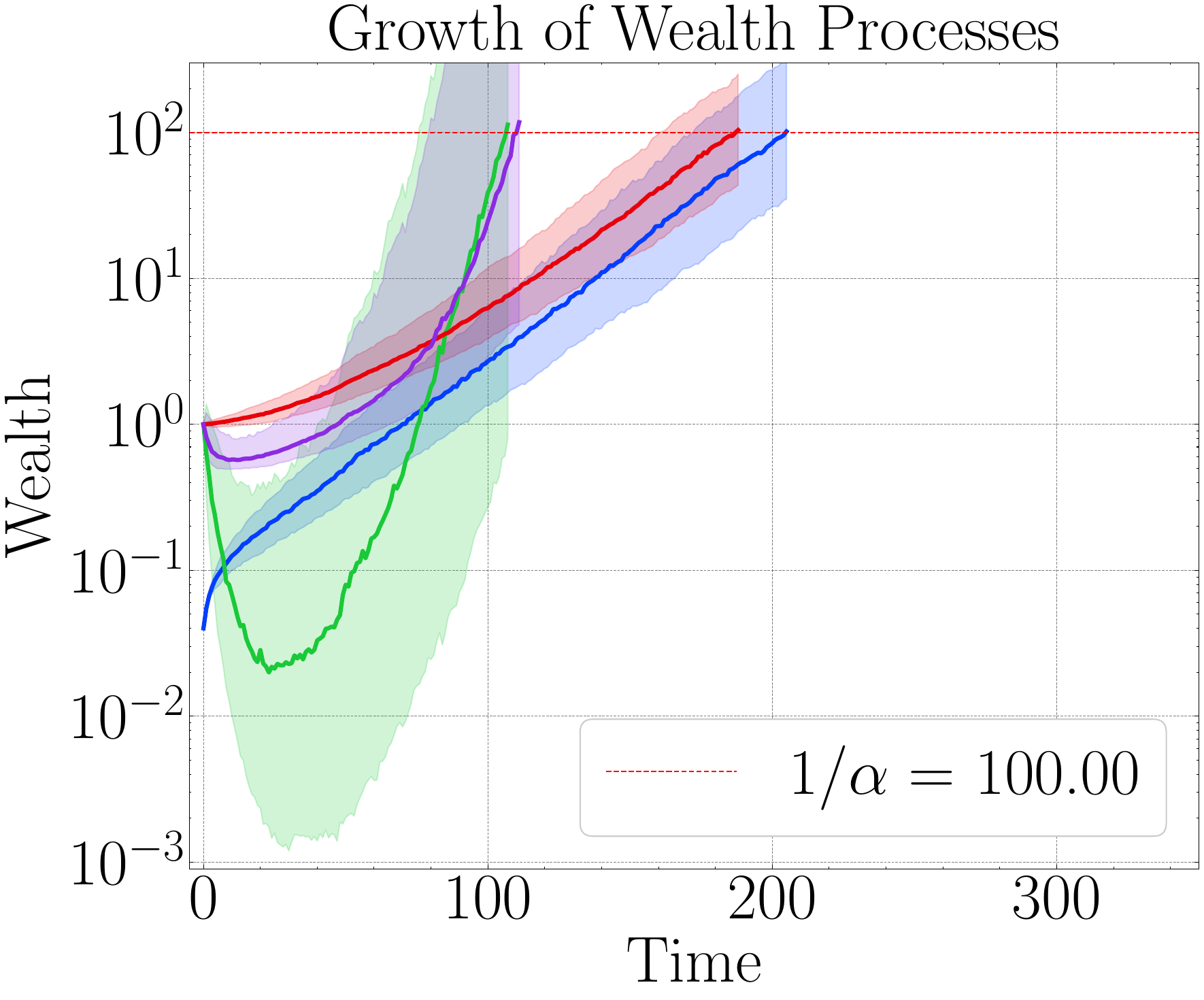}
        \caption{$\lfloor \frac{k_1}{k}\rfloor = 0.45$}
    \end{subfigure}
    \begin{subfigure}[t]{0.19\linewidth}
        \includegraphics[width=\linewidth]{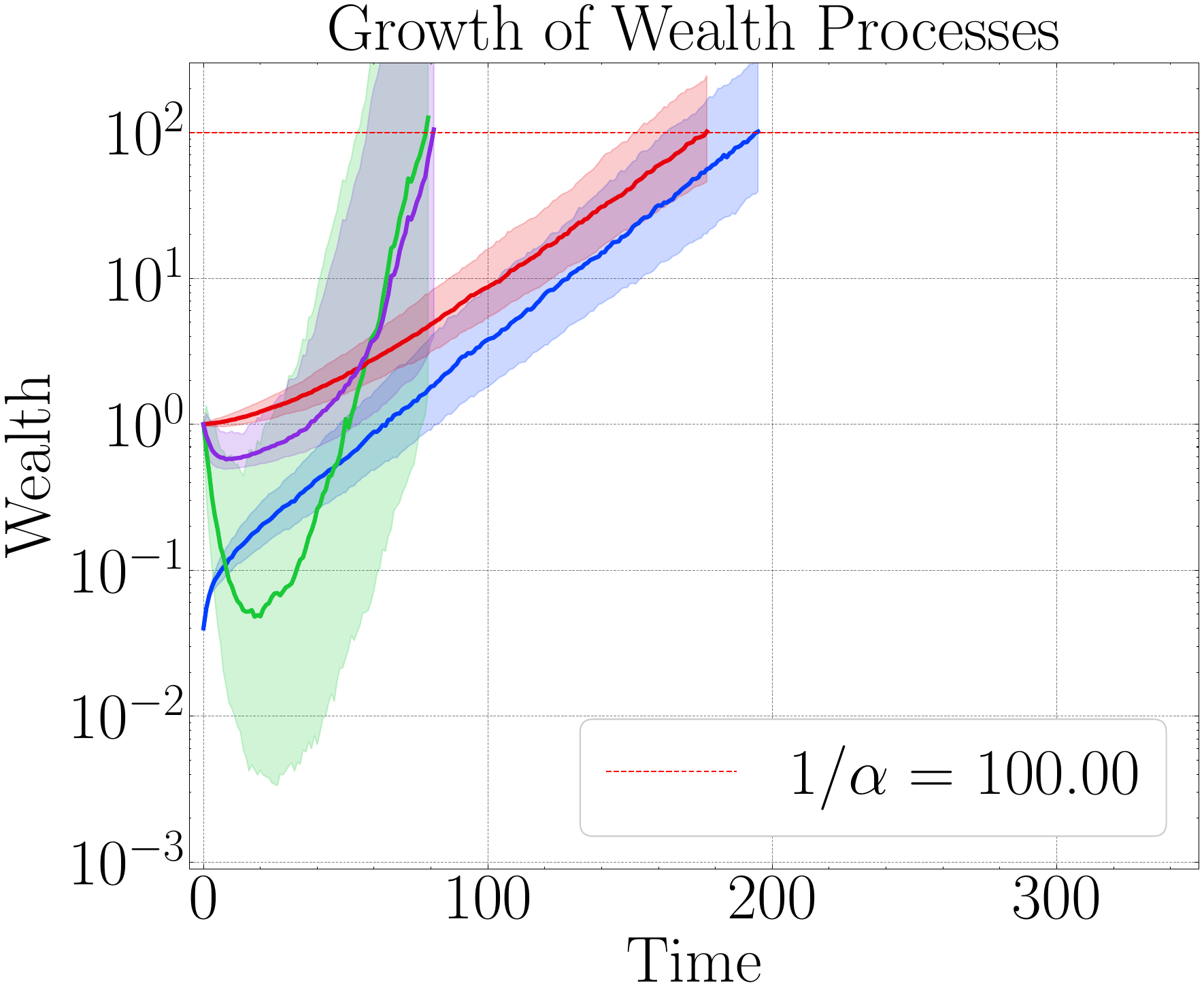}
        \caption{$\lfloor \frac{k_1}{k}\rfloor = 0.60$}
    \end{subfigure}
    \begin{subfigure}[t]{0.19\linewidth}
        \includegraphics[width=\linewidth]{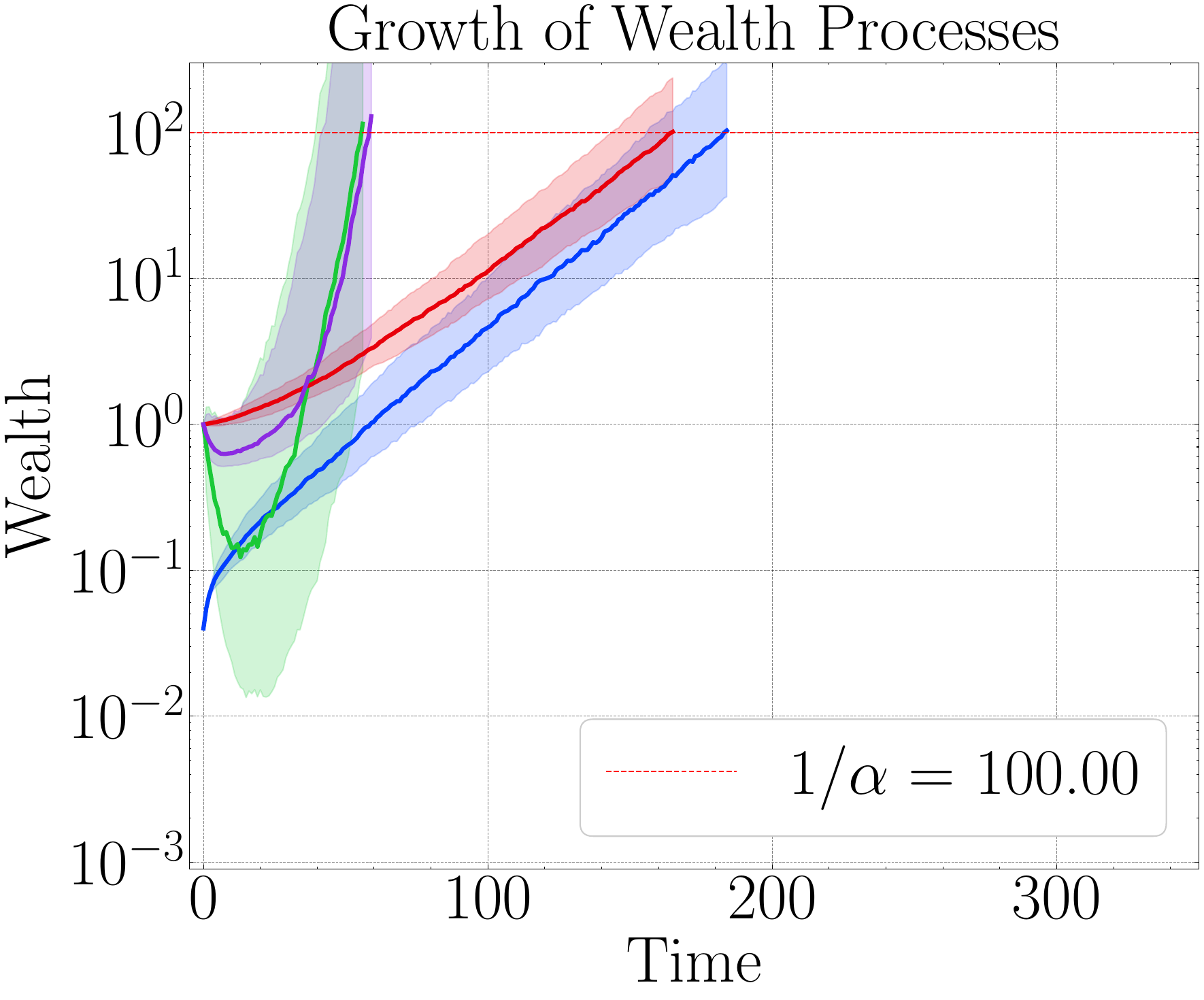}
        \caption{$\lfloor \frac{k_1}{k}\rfloor = 0.75$}
    \end{subfigure}
    }
    \caption{{\bf Top:} Distribution of stopping times, over 1,000 simulations, for various sequential tests across settings with varying proportions of streams with nonzero means. A test rejects when its corresponding wealth process exceeds $\nicefrac{1}{\alpha}$ for $\alpha = 0.01$. The dashed vertical line is the empirical mean of the stopping times. {\bf Bottom:} Trajectories of various wealth processes across settings with different amounts of nonzero means. Each line represents the median trajectory of a wealth process over 1,000 simulations, with shaded areas indicating the 25\% and 75\% quantiles. The y-axis is presented on a logarithmic scale. Wealth processes are clipped to $10^{-3}$ for visualization purposes.}
\end{figure*}

\newpage

\subsubsection{Results for $k$ = 100}
\begin{figure*}[h!]
    \centering
    \resizebox{1\textwidth}{!}{
    \hspace{-0.5em}
    \begin{subfigure}[t]{0.19\linewidth}
        \centering
        \includegraphics[width=\linewidth]{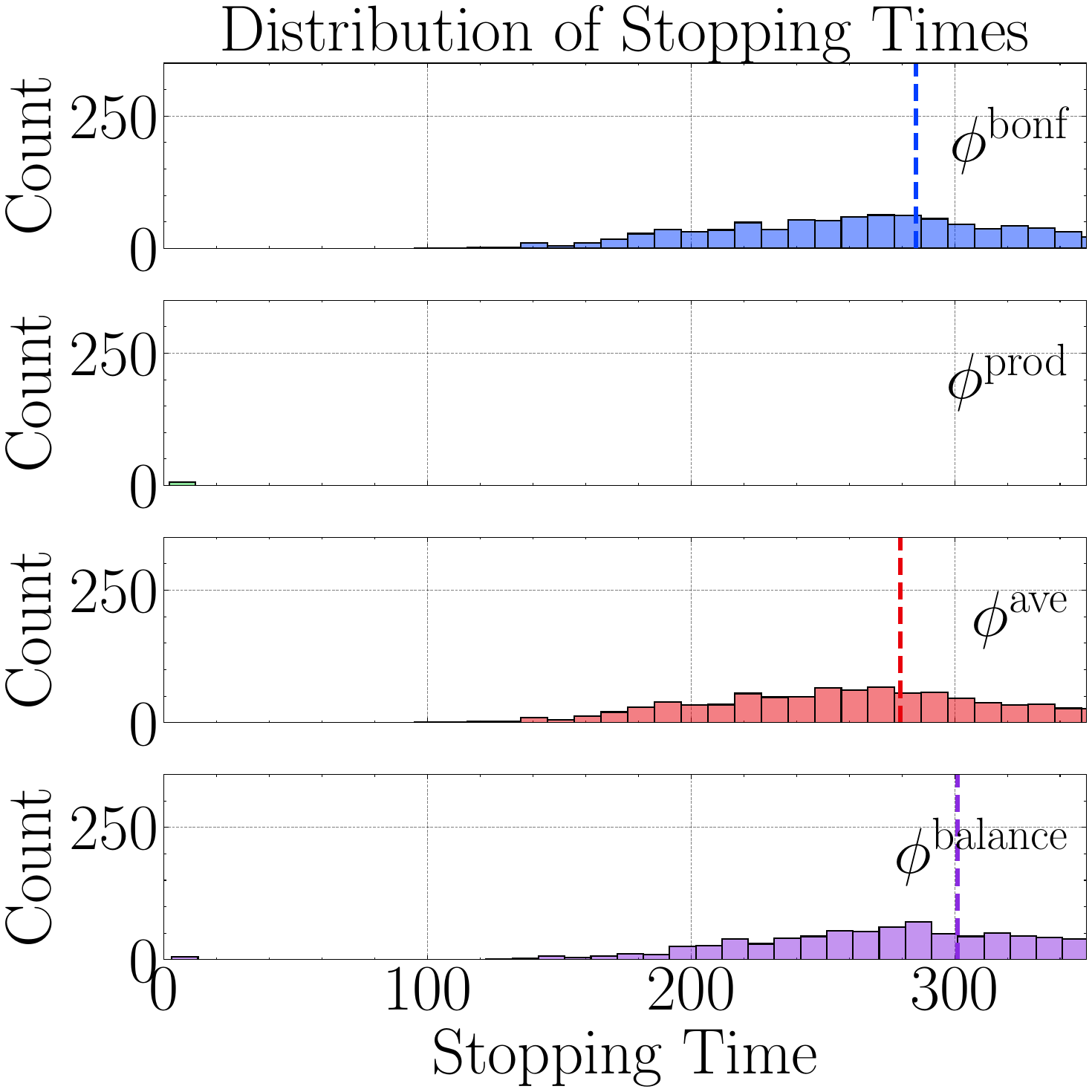}
    \end{subfigure}
    \begin{subfigure}[t]{0.19\linewidth}
        \centering
        \includegraphics[width=\linewidth]{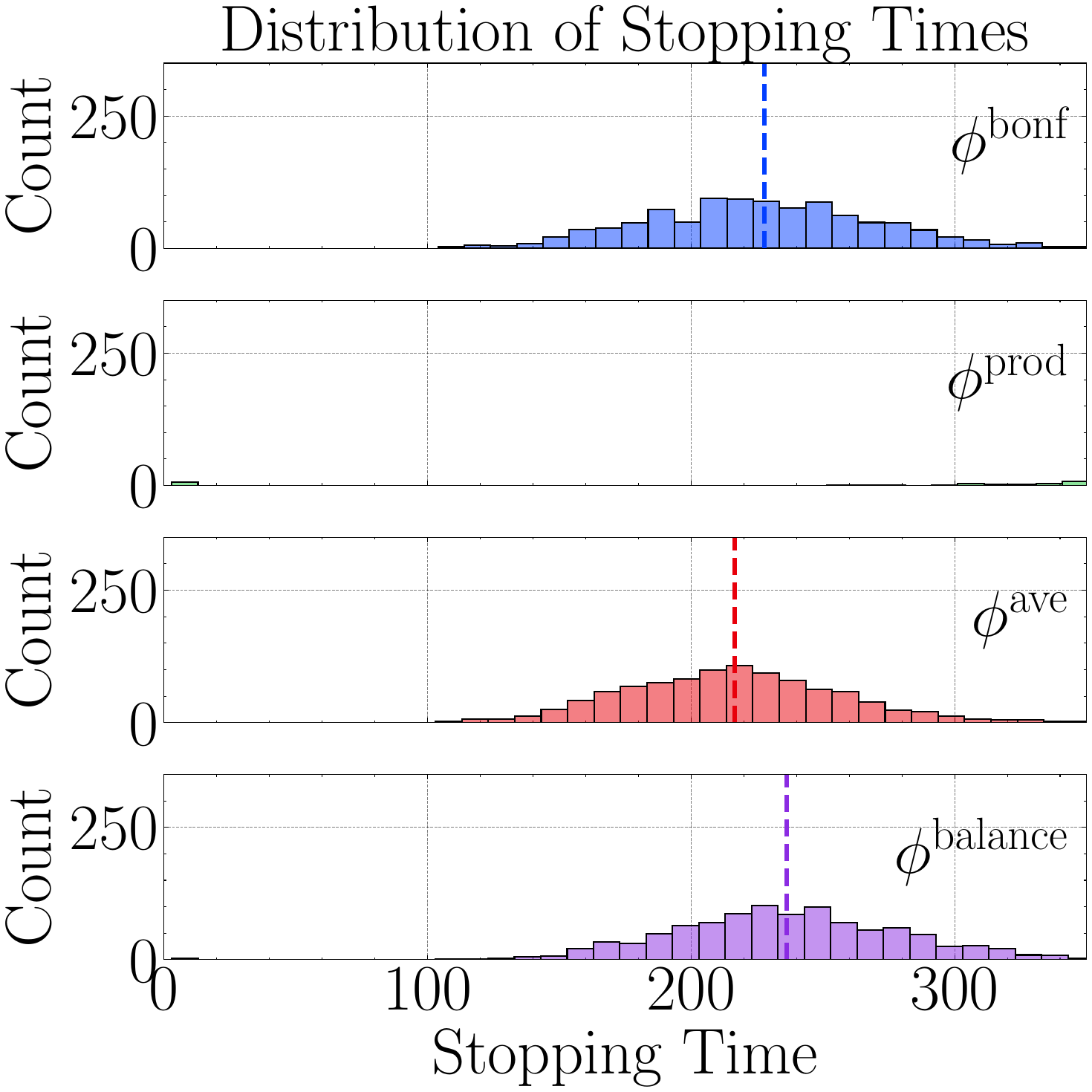}
    \end{subfigure}
    \begin{subfigure}[t]{0.19\linewidth}
        \centering
        \includegraphics[width=\linewidth]{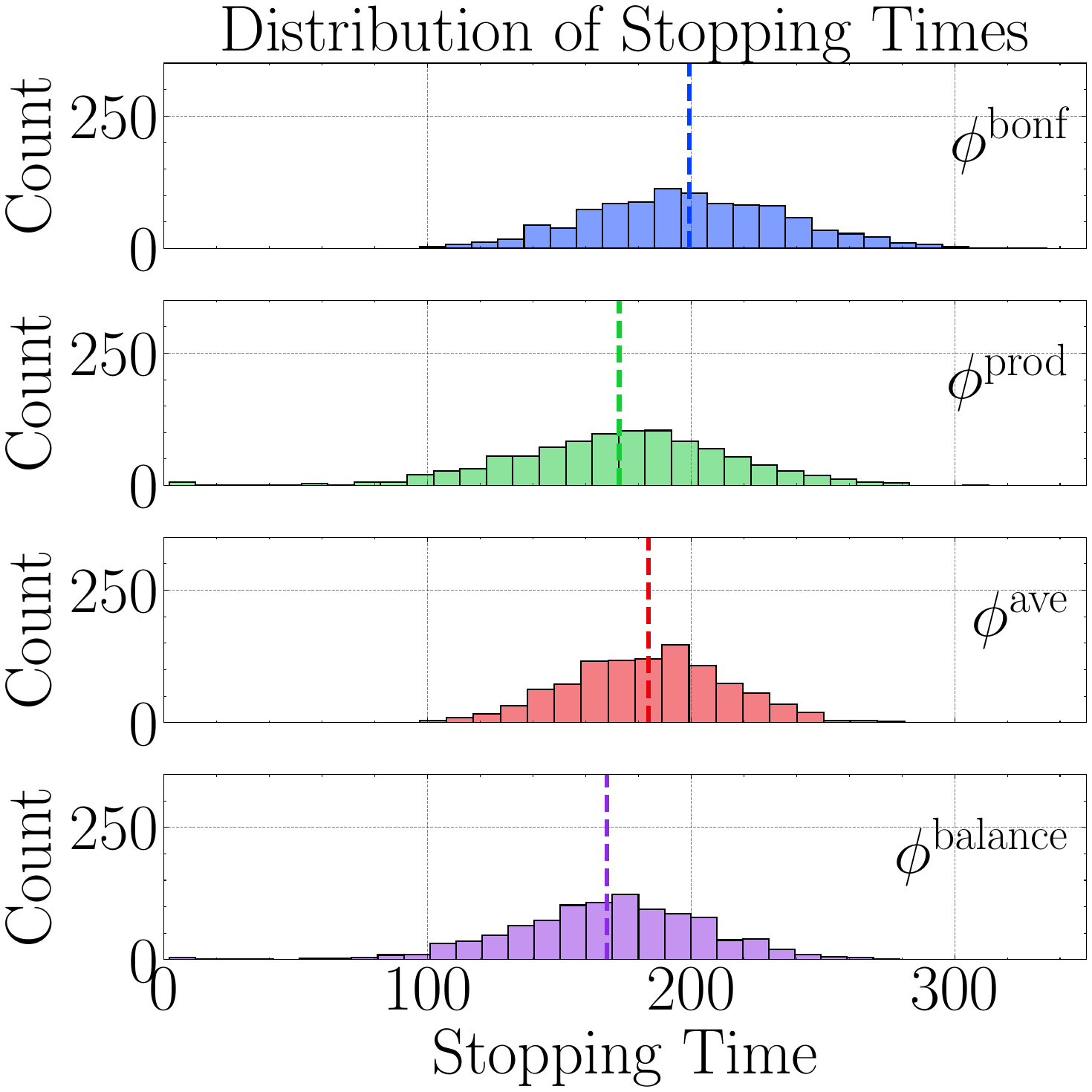}
    \end{subfigure}
    \begin{subfigure}[t]{0.19\linewidth}
        \centering
        \includegraphics[width=\linewidth]{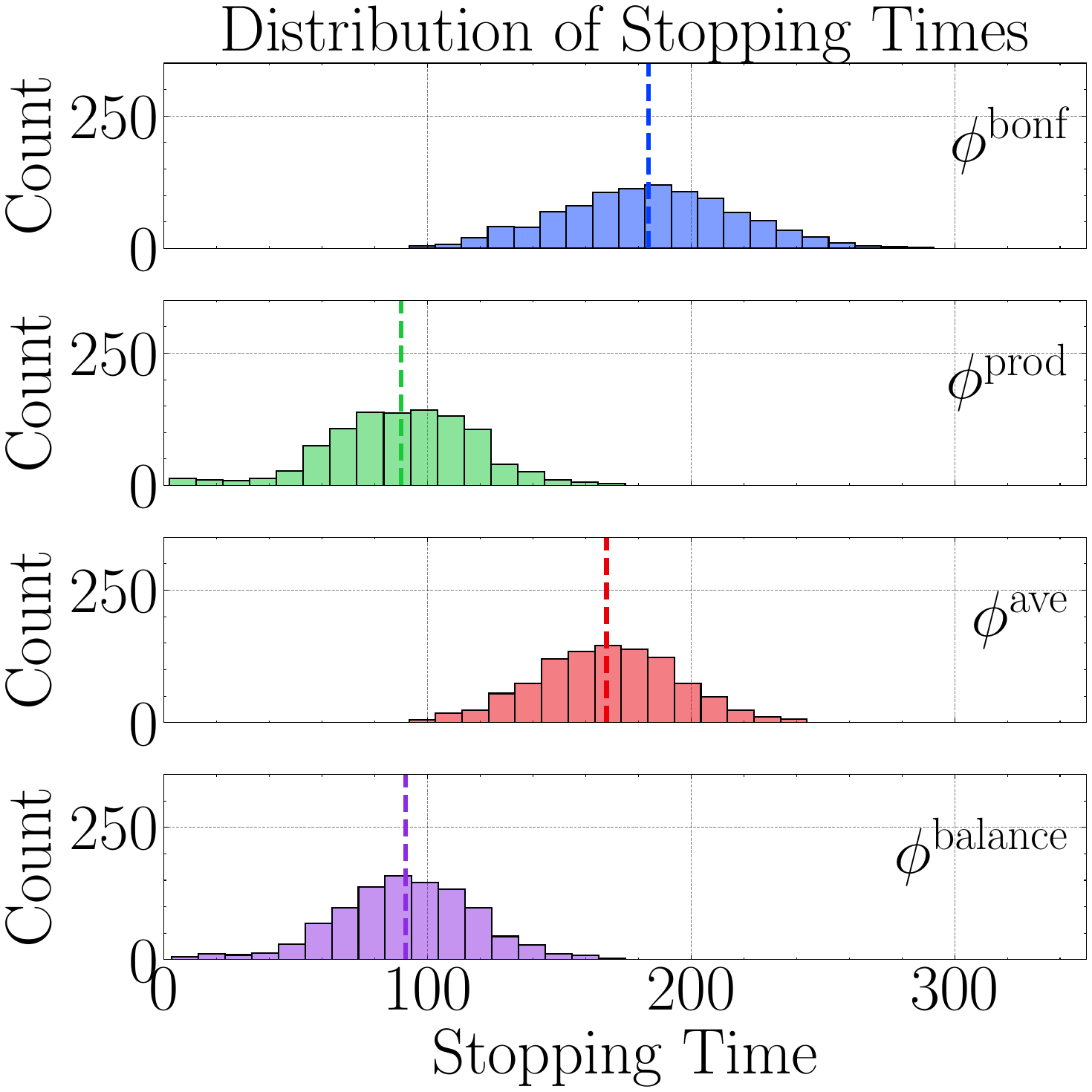}
    \end{subfigure}
    \begin{subfigure}[t]{0.19\linewidth}
        \centering
        \includegraphics[width=\linewidth]{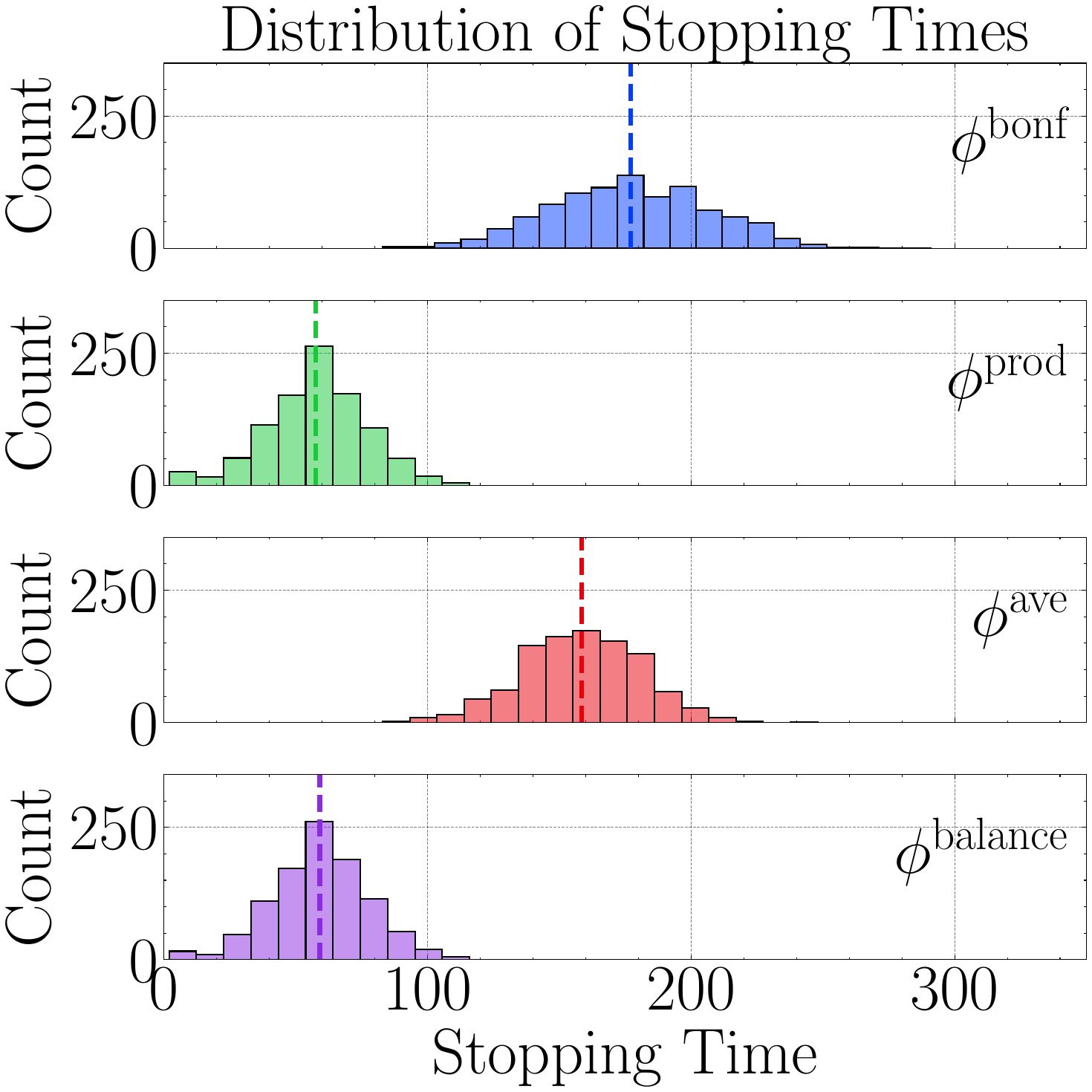}
    \end{subfigure}
    \begin{subfigure}[t]{0.19\linewidth}
        \centering
        \includegraphics[width=\linewidth]{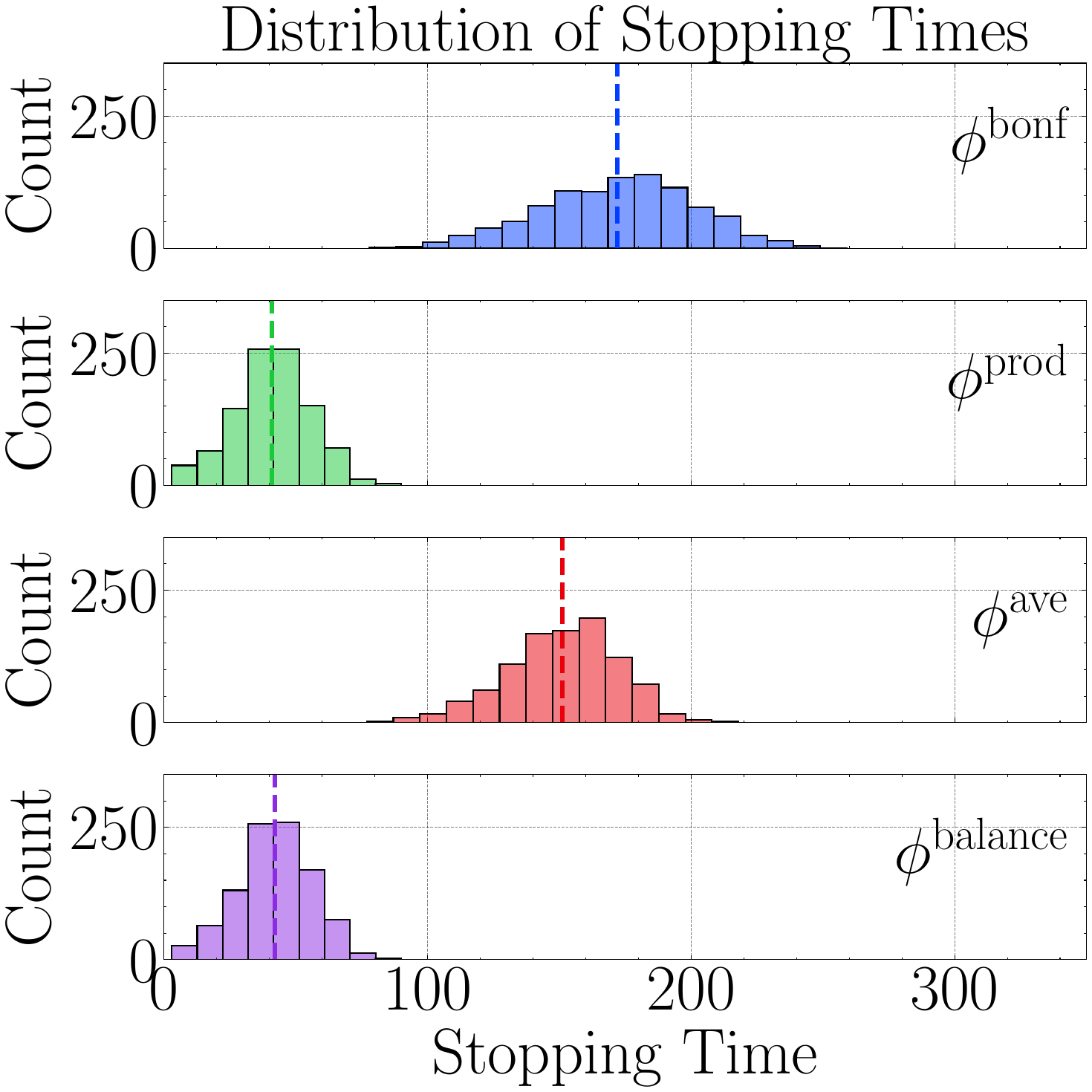}
    \end{subfigure}
    }

    \centering
    \resizebox{1\textwidth}{!}{%
    \begin{subfigure}[t]{0.19\linewidth}
        \includegraphics[width=\linewidth]{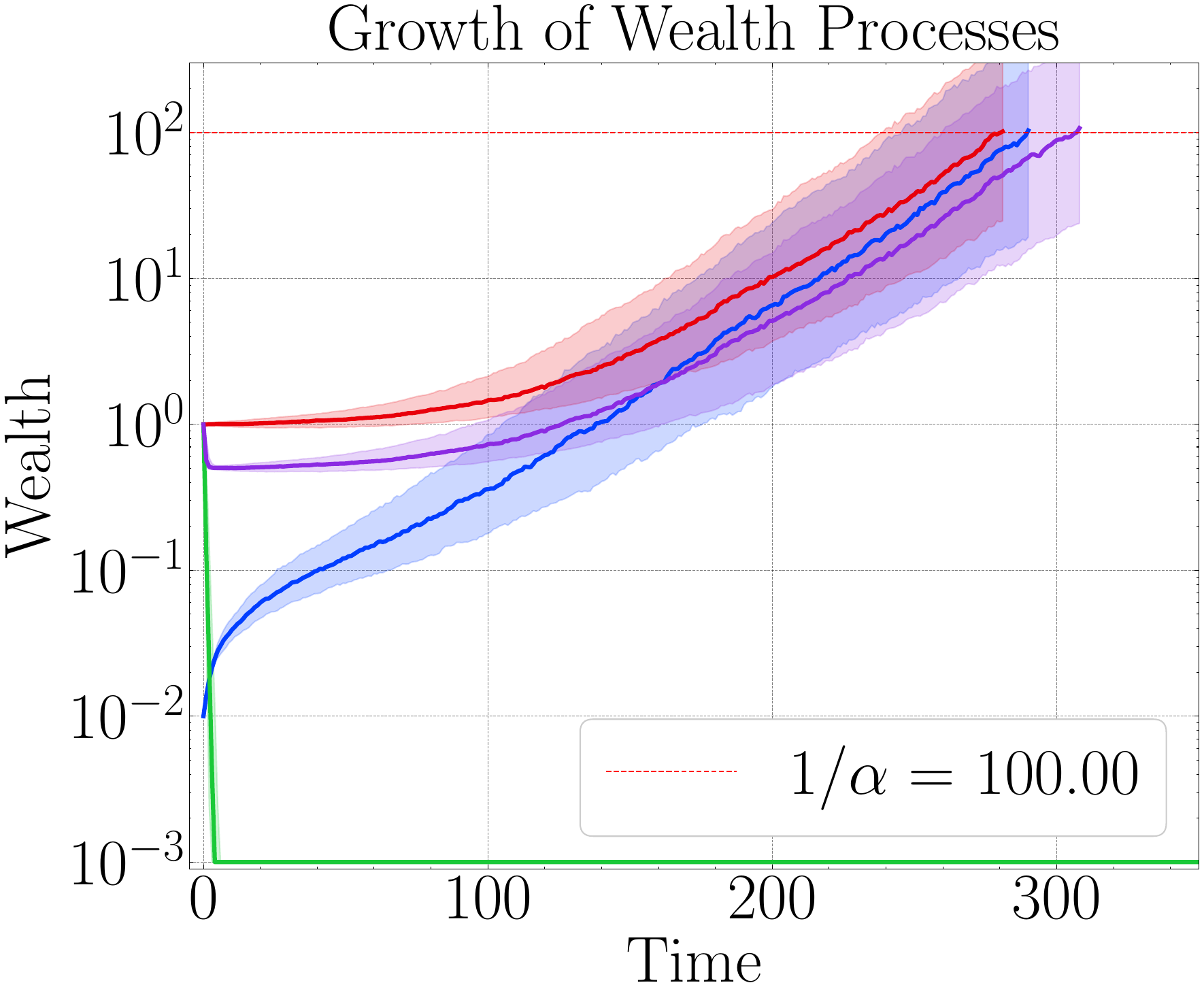}
        \caption{$\lfloor \frac{k_1}{k}\rfloor = 0.05$}
    \end{subfigure}
    \begin{subfigure}[t]{0.19\linewidth}
        \centering
        \includegraphics[width=\linewidth]{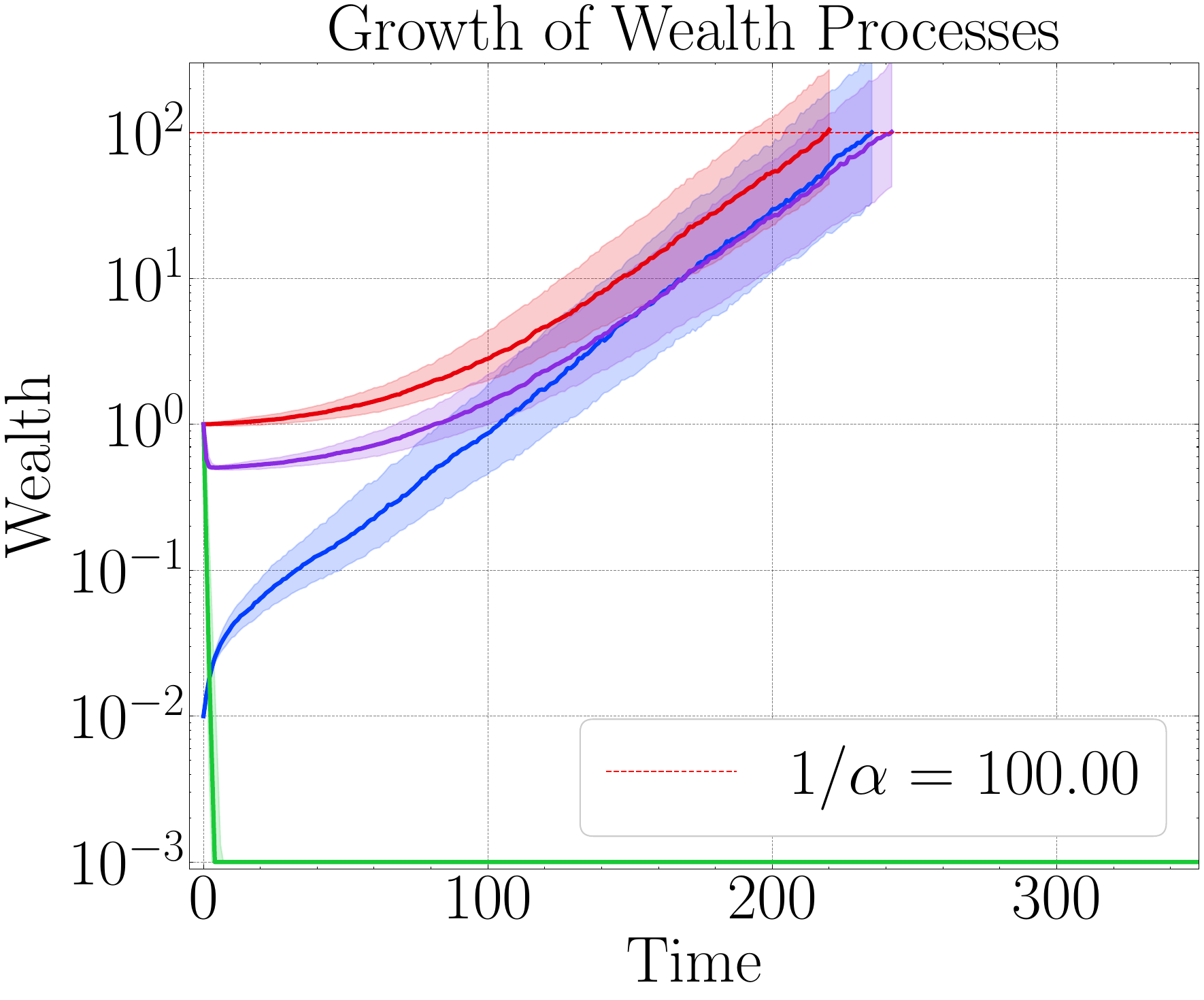}
        \caption{$\lfloor \frac{k_1}{k}\rfloor = 0.15$}
    \end{subfigure}
    \begin{subfigure}[t]{0.19\linewidth}
        \includegraphics[width=\linewidth]{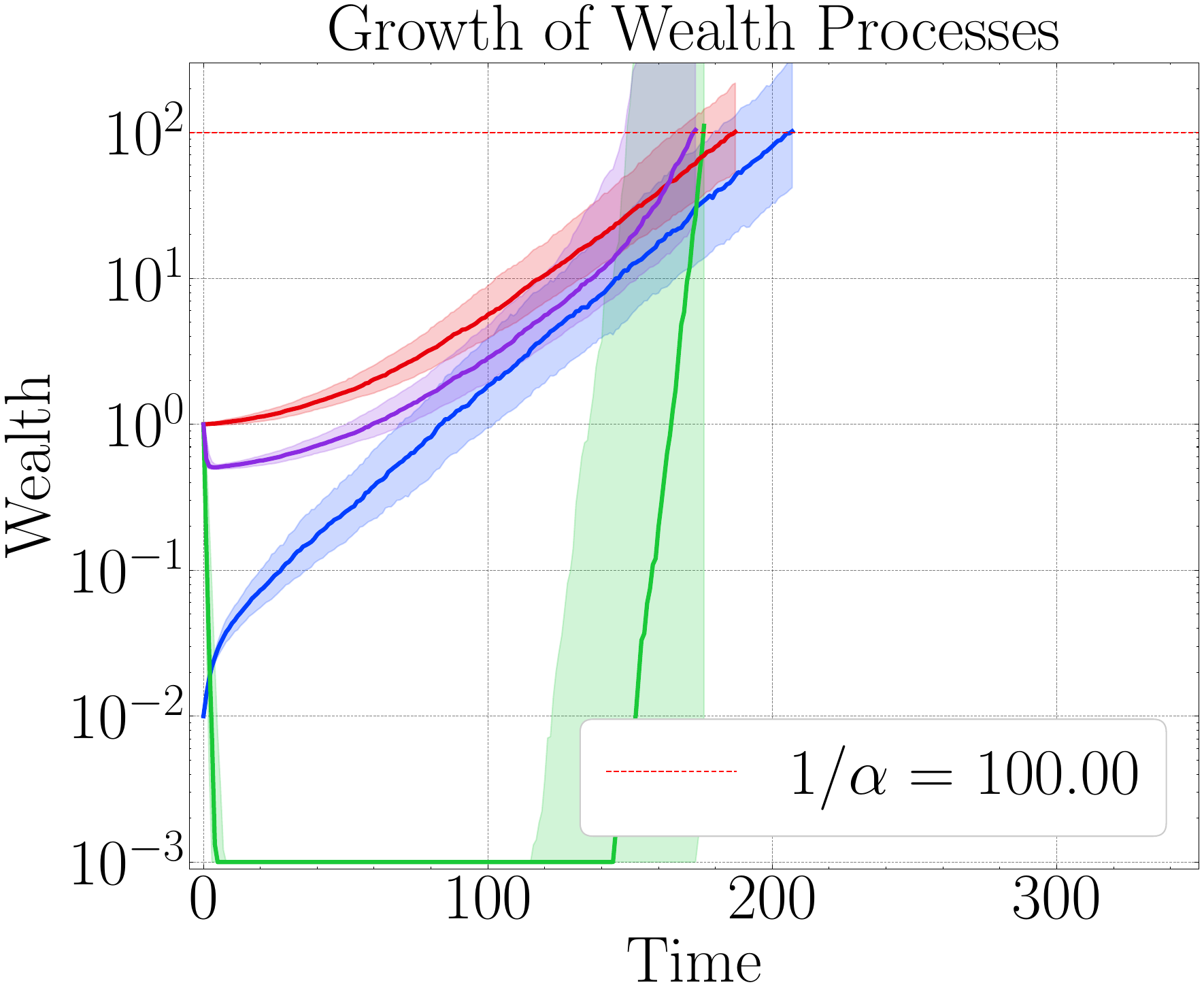}
        \caption{$\lfloor \frac{k_1}{k}\rfloor = 0.30$}
    \end{subfigure}
    \begin{subfigure}[t]{0.19\linewidth}
        \includegraphics[width=\linewidth]{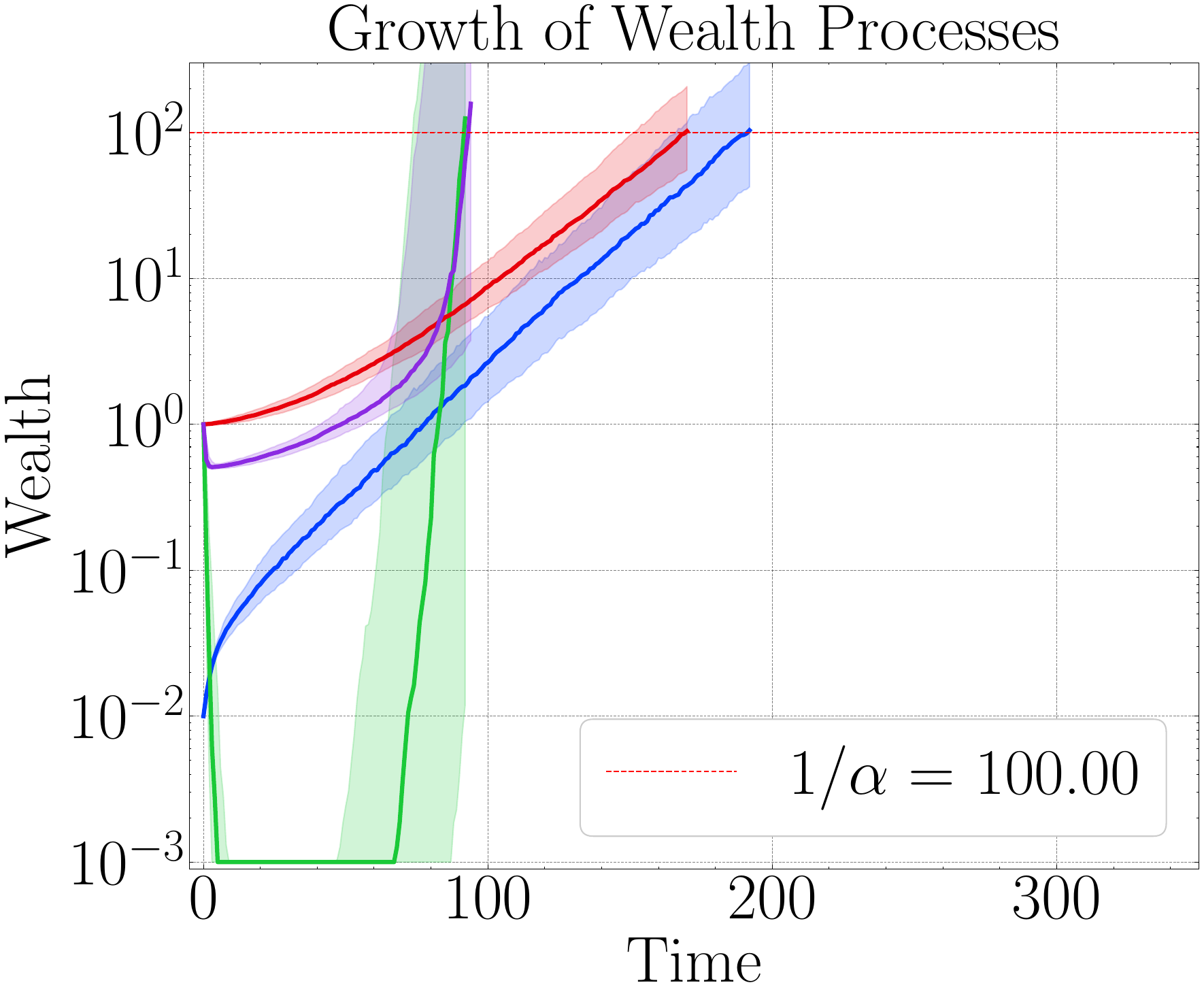}
        \caption{$\lfloor \frac{k_1}{k}\rfloor = 0.45$}
    \end{subfigure}
    \begin{subfigure}[t]{0.19\linewidth}
        \includegraphics[width=\linewidth]{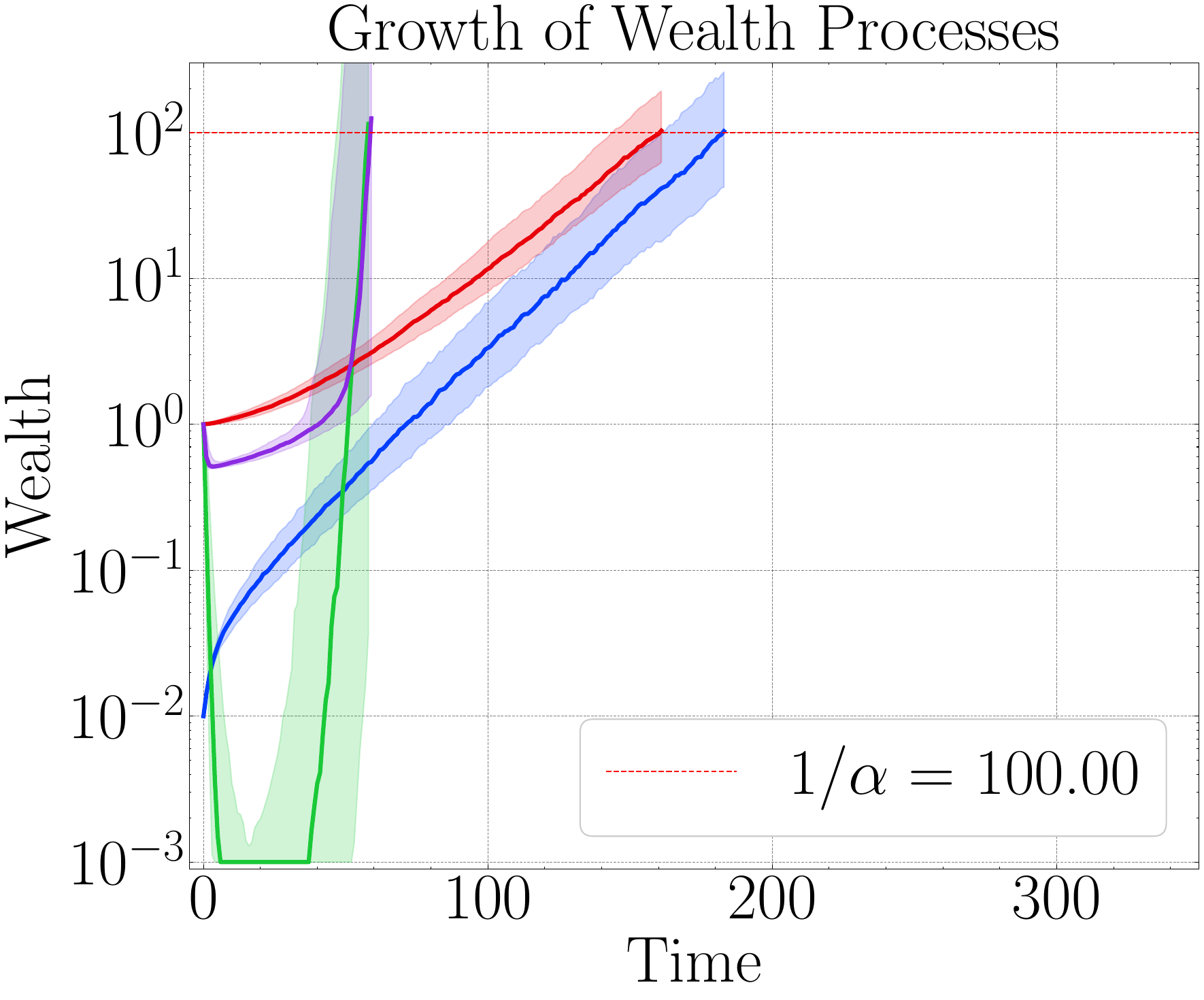}
        \caption{$\lfloor \frac{k_1}{k}\rfloor = 0.60$}
    \end{subfigure}
    \begin{subfigure}[t]{0.19\linewidth}
        \includegraphics[width=\linewidth]{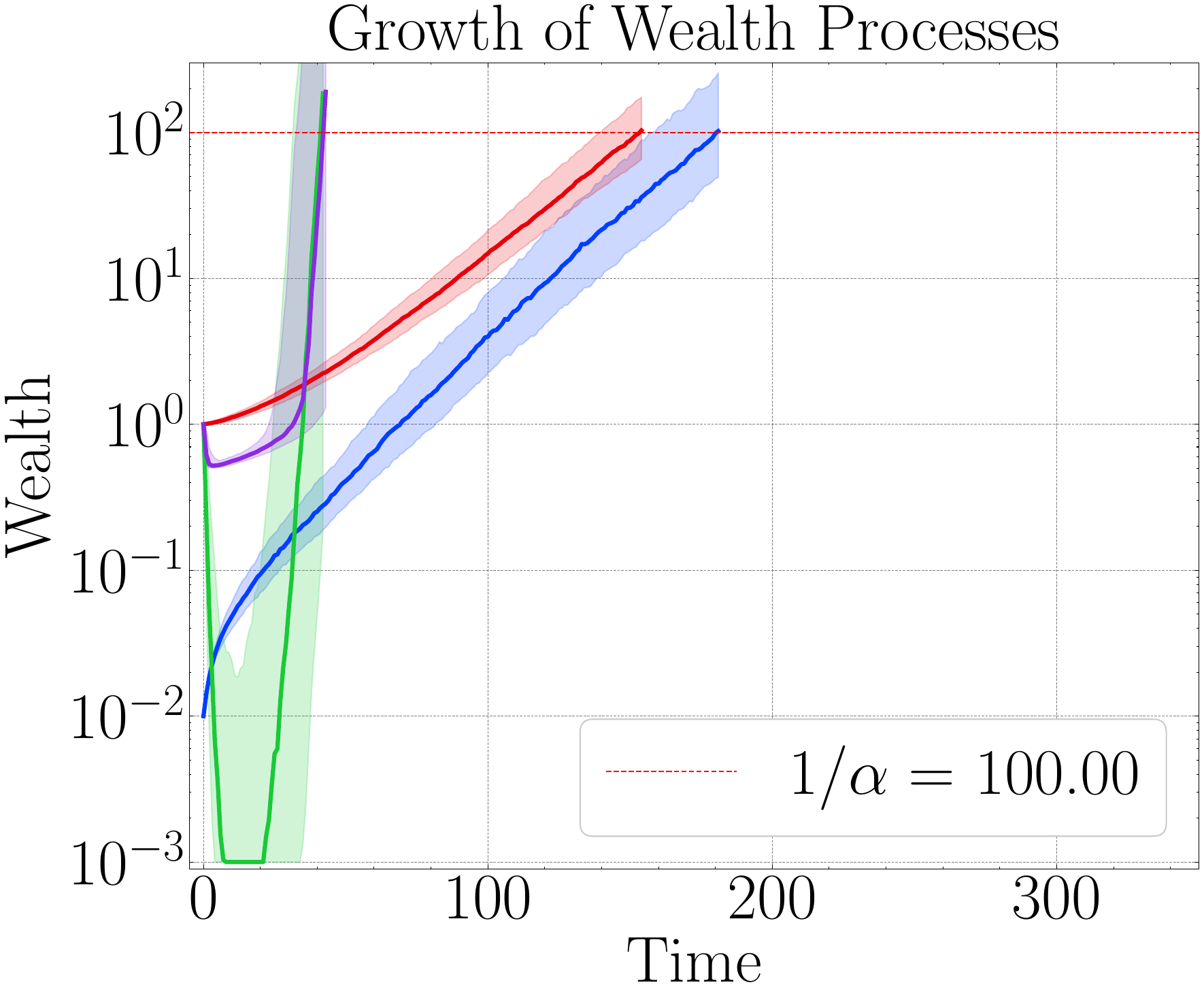}
        \caption{$\lfloor \frac{k_1}{k}\rfloor = 0.75$}
    \end{subfigure}
    }
    \caption{{\bf Top:} Distribution of stopping times, over 1,000 simulations, for various sequential tests across settings with varying proportions of streams with nonzero means. A test rejects when its corresponding wealth process exceeds $\nicefrac{1}{\alpha}$ for $\alpha = 0.01$. The dashed vertical line is the empirical mean of the stopping times. {\bf Bottom:} Trajectories of various wealth processes across settings with different amounts of nonzero means. Each line represents the median trajectory of a wealth process over 1,000 simulations, with shaded areas indicating the 25\% and 75\% quantiles. The y-axis is presented on a logarithmic scale. Wealth processes are clipped to $10^{-3}$ for visualization purposes.}
\end{figure*}

\newpage

\subsection{Zero-shot medical image classification}

\begin{table}[h!]
\centering
\caption{Additional details about the datasets used in the zero-shot medical classification experiment. The first and second classes listed in the table correspond to $Y=0$ and $Y=1$, respectively. Most of the information provided in this table is derived from \citet{nie2025conceptclip}.} 
\vspace{1em}
\small
\setlength{\tabcolsep}{4pt}
\renewcommand{\arraystretch}{0.95}

\scalebox{0.9}{
\begin{tabular}{p{2cm} p{2cm} p{6cm} p{3cm} p{0.75cm}}
\toprule
Group Name & Dataset & Description & Classes & Link \\ 
\midrule
\texttt{brain CT} & Brain Tumor CT  &  High-resolution CT scans for brain tumor detection. & Healthy \newline Tumor& \href{https://www.kaggle.com/datasets/murtozalikhon/brain-tumor-multimodal-image-ct-and-mri}{\textcolor{blue}{\faCheck}}    \\ \hline
\texttt{brain MRI}   & Brain Tumor MRI     & Similar to the Brain Tumor CT dataset but with MRIs.     & Healthy \newline Tumor  & \href{https://www.kaggle.com/datasets/murtozalikhon/brain-tumor-multimodal-image-ct-and-mri}{\textcolor{blue}{\faCheck}}    \\ \hline
\texttt{covid X-ray}   & Covid-CXR2     & 16,000 chest X-ray images including 2,300 positive COVID-19 images/  & No finding \newline Covid-19 & \href{https://www.kaggle.com/datasets/andyczhao/covidx-cxr2}{\textcolor{blue}{\faCheck}}    \\ \hline
\texttt{breast mammogram}   & Breast Cancer   & 3,383 annotated mammogram images focused on breast tumors. & Normal \newline Tumor  & \href{https://www.kaggle.com/datasets/hayder17/breast-cancer-detection/}{\textcolor{blue}{\faCheck}}  \\ \hline
\texttt{breast ultrasound}    & UBIBC    & Ultrasound images related to breast cancer. & Benign \newline Malignant &\href{https://www.kaggle.com/datasets/vuppalaadithyasairam/ultrasound-breast-images-for-breast-cancer}{\textcolor{blue}{\faCheck}} \\ \hline
\texttt{colon pathology}   & LC2500     & 10,000 pathology images from colon tissues.  & Normal \newline Adenocarcinomas &\href{https://github.com/tampapath/lung_colon_image_set}{\textcolor{blue}{\faCheck}}  \\ \hline
\texttt{retinal oct}    & Retinal OCT     & High-quality retinal OCT images.     & Normal \newline Not Normal &\href{https://www.kaggle.com/datasets/obulisainaren/retinal-oct-c8}{\textcolor{blue}{\faCheck}}    \\ \hline
 \texttt{colon endoscopy}    & WCE    & Curated colon disease images. & Normal \newline Not Normal &\href{https://www.kaggle.com/datasets/francismon/curated-colon-dataset-for-deep-learning}{\textcolor{blue}{\faCheck}}   \\ \hline
 \texttt{lung pathology}    & LC2500     & 15,000 pathology images from lung tissue.   & Normal \newline Not Normal &\href{https://github.com/tampapath/lung_colon_image_set}{\textcolor{blue}{\faCheck}}   \\ \hline
\texttt{covid ct}  & COVIDxCT     & CT scans of patients with Covid-19.    & Normal \newline Covid-19 &\href{https://www.kaggle.com/datasets/hgunraj/covidxct/data}{\textcolor{blue}{\faCheck}}   \\ 
\bottomrule
\end{tabular}
}
\label{app:conceptclip_data_table}
\end{table}

\begin{table}[ht!]
    \centering
    \caption{$\bE_{P_i}[Z]$ across different groups defined by modality and anatomical region} 
    \label{app:conceptclip_errors_table}
    \vspace{1em}
    \resizebox{0.4\linewidth}{!}{ 
    \begin{tabular}{lc}
        \toprule
        \textbf{Group} & $\bE_{P_i}[Z]$ \\ 
        \midrule
        \texttt{brain CT} & $0.31$ \\   
        \texttt{brain MRI} & $0.07$ \\   
        \texttt{covid X-ray} & $0.28$  \\
        \texttt{breast mammogram} & $-0.19$ \\   
        \texttt{breast ultrasound} & $-0.07$ \\   
        \texttt{colon pathology} & $0.09$ \\   
        \texttt{retinal oct} & $-0.04$ \\   
        \texttt{colon endoscopy} & $-0.10$ \\   
        \texttt{lung pathology} & $0.000$ \\   
        \texttt{covid ct} & $-0.30$ \\   
        \bottomrule
    \end{tabular}}
\end{table}

\end{document}

%% file: notation.tex
\def\bE{{\mathbb{E}}}

\def\bR{{\mathbb{R}}}

\def\calK{{\mathcal K}}
\def\calO{{\mathcal O}}

%% file: preamble.tex
\usepackage[margin=1in]{geometry}
\usepackage{optidef}
\usepackage{amsmath}
\usepackage{amsthm}
\usepackage{amssymb}
\usepackage{mathtools}
\usepackage{algorithm}
\usepackage{algpseudocode}
\usepackage{float}
\usepackage{bbm}
\usepackage{listings}
\usepackage{booktabs}
\usepackage{comment}
\usepackage{subcaption}
\usepackage{enumitem}
\usepackage[dvipsnames]{xcolor}
\usepackage{footmisc}
\usepackage{titlesec}

\definecolor{bleudefrance}{rgb}{0.19, 0.55, 0.91}
\definecolor{pastelblue}{rgb}{0.68, 0.78, 0.81}
\definecolor{oxfordblue}{rgb}{0.0, 0.13, 0.28}
\definecolor{lavender}{rgb}{0.75, 0.58, 0.89}

\usepackage{tcolorbox}
\usepackage{nicefrac}
\usepackage{thm-restate}
\usepackage{fontawesome}
\usepackage[export]{adjustbox}

\usepackage{hyperref} 
\hypersetup{
  colorlinks = true,
  urlcolor = bleudefrance,
  linkcolor = magenta,
  citecolor = pastelblue
}

\usepackage[capitalize]{cleveref}
\usepackage[numbers]{natbib}

\newtheorem{theorem}{Theorem}[section]
\newtheorem{lemma}{Lemma}[section]

\titleformat{\paragraph}[runin]
  {\normalfont\normalsize\bfseries} 
  {}                                
  {0pt}
  {}

\titlespacing*{\paragraph}
  {0pt}   
  {0.5ex} 
  {1em}   